\def\arxivpreprint{1}
\PassOptionsToPackage{table}{xcolor}
\documentclass{article}
\usepackage{iclr2027_conference,times}
\setcitestyle{numbers,square,comma,citesep={,},sort&compress}
\usepackage{amsmath,amssymb,mathtools}
\usepackage{booktabs}
\usepackage{multirow}
\usepackage{graphicx}
\usepackage{wrapfig}
\usepackage{tikz}
\usepackage{enumitem}
\usepackage{microtype}
\usepackage{xcolor}
\usepackage{hyperref}
\usepackage{arydshln}
\usepackage{url}
\usepackage{cleveref}
\usepackage{colortbl}
\usepackage{graphicx}
\graphicspath{{figures/}}
\usepackage{subcaption}
\usepackage{multirow}
\usepackage{makecell}
\usepackage{tcolorbox}

\usepackage{float}
\usepackage{algorithm}
\usepackage{algpseudocode}

\newcommand{\KL}{D_{\mathrm{KL}}}
\newcommand{\TV}{\mathrm{TV}}

\newcommand{\Sset}{\mathcal S}
\newcommand{\name}{\textsc{LEDFlow}}

\usepackage{amsthm}
\theoremstyle{plain}
\newtheorem{theorem}{Theorem}
\newtheorem{lemma}{Lemma}
\newtheorem{proposition}{Proposition}
\newtheorem{corollary}{Corollary}
\theoremstyle{definition}
\newtheorem{definition}{Definition}

\theoremstyle{remark}
\newtheorem{remark}{Remark}

\newcommand{\papertitle}{\name{}: Introducing Entropy-guided Generation Order into Uniform Discrete Flow}
\newcommand{\papersubtitle}{\\[0.35em]{\large\normalfont \name{}: Introducing Entropy-guided Generation Order into Uniform Discrete Flow}}

\ifdefined\arxivpreprint
\title{Entropy can flow, or it can guide. Be Entropy.\papersubtitle}
\iclrfinalcopy
\author{Tung Sum Thomas Kwok*, Yidong Ouyang*, Yingjia Wan, Ying Nian Wu\\
University of California, Los Angeles\\
\texttt{tk1018@ucla.edu}
\And
Zhijiang Guo**\\
Hong Kong University of Science and Technology\\
\texttt{zhijiangguo@hkust-gz.edu.cn}
\And
Oscar Leong**\\
Shanghai Jiao Tong University\\
\texttt{oleong@sjtu.edu.cn}
}
\else
\title{\papertitle}
\author{Anonymous authors\\ Paper under double-blind review}
\fi

\begin{document}
\maketitle
\ifdefined\arxivpreprint
% Override the conference-publication header inserted by \iclrfinalcopy.
\lhead{Preprint}
\fi

\begin{abstract}
Uniform discrete flow permits repeated updates at every generation position.
While continued revision supports correction of wrong tokens, it also exposes correct intermediate predictions to later errors.
An experiment on Sudoku puzzles shows that 9.4\% of generated cells are correct at an intermediate step but incorrect in the final output.
We introduce generation order into uniform discrete flow through selective absorption, which fixes chosen predictions while preserving the uniform-flow velocity at active positions.
To prioritize reliable predictions for absorption, we propose Low-Entropy Discrete Flow (\name{}), a training-free sampler that adaptively orders absorption by local entropy.
By decomposing absorption error into joint dependence and conditional prediction terms, we show that, under entropy-error regularity, selecting the lowest-entropy positions under a fixed absorption count minimizes an upper bound on the conditional term.
We further analyze sensitivity of global lookahead, whose worst-case decision-error bound grows with lookahead window under an imperfect denoiser.
Across reasoning benchmarks, \name{} attains 0.845 Nikoli Sudoku solve accuracy, with largest gains on strongly constrained tasks.
On a text-to-image generation benchmark it attains the best overall score among decode-time samplers, and on multimodal understanding it improves over the default sampler on all six benchmarks, at an inference cost comparable to standard flow sampling.
\end{abstract}

% ===== body sections =====
\section{Introduction}
Discrete flow matching allows multimodal generation by transporting a source distribution to a target distribution over discrete states through a continuous time probability velocity~\citep{gat2024discrete}, while uniform discrete flow progressively transforms every token in each time step under a uniform source. 
FUDOKI~\citep{wang2025fudoki} applies this framework to multimodal understanding and generation using a metric-induced probability path and a kinetic-optimal velocity. 
Conditioned on a target token, the velocity moves probability mass toward tokens closer to that estimated target, 
so that generated positions can be updated repeatedly as sampling proceeds. 

\begin{figure}[t]
\centering
\includegraphics[width=\linewidth]{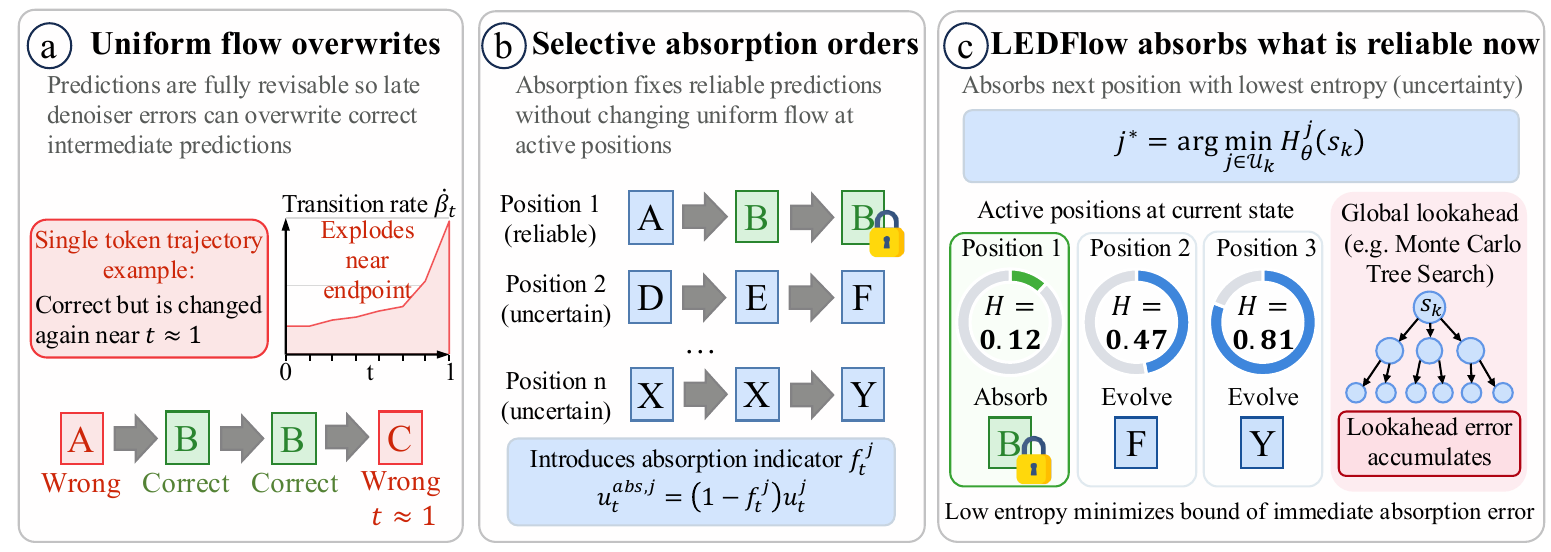}
\caption{\textbf{Entropy-guided absorption in uniform discrete flow.}
\textbf{(a)} Late revisions can overwrite correct predictions.
\textbf{(b)} Selective absorption fixes chosen predictions while other positions continue evolving.
\textbf{(c)} \name{} prioritizes low-entropy positions for absorption; global lookahead aggregates potentially noisy estimates across the remaining positions.}
\label{fig:motivation}
% \vspace{-3mm}
\end{figure}

Continued revision can replace correct intermediate tokens with incorrect predictions, especially near the end of sampling. 
Under the kinetic-optimal formulation, the schedule gradient diverges as $t$ approaches the endpoint. The outgoing transition rate then grows without bound when a token is semantically distant from the estimated target. 
This late denoiser error exposes an otherwise correct token to a large transition rate. 
A Sudoku trajectory diagnostic (Fig.~\ref{fig:overjumpy}) using the kinetic-optimal update shows on 100 puzzles with 64 sampling steps that 9.4\% of generated cells are correct at least once during sampling but incorrect in the final output.
This motivates us to introduce generation order into uniform discrete flow through selective absorption.
We equip each position with a binary absorption indicator that sets its outgoing rates to zero once absorbed, allowing reliable predictions to remain fixed while the other positions continue evolving. 
However, fixing an incorrect token prevents correction, and an arbitrary order does not distinguish reliable predictions from uncertain ones. 
Therefore, the absorption policy must account for this irreversible prediction risk.
%\textcolor{blue}{*} %\textcolor{blue}{* This random absorption may also have a bad influence. More transition to the next paragraph.}

To address this ordering problem, we propose Low-Entropy Discrete Flow (\name{}), a training-free sampler that introduces generation order into uniform discrete flow through entropy-guided
absorption (Fig.~\ref{fig:motivation}).
At each sampling step, \name{} selects the prescribed number of active positions with the lowest predictive entropy and absorbs them with the most likely predicted values. 
The remaining positions continue to evolve under the kinetic-optimal velocity, %\textcolor{blue}{conditioned on the current state,} 
with subsequent predictions conditioned on the values already fixed by absorption. %\textcolor{blue}{*} 
At each step, the sampler recomputes entropy for an adaptive order, allowing the policy to preserve the uniform flow velocity at active positions while prioritizing low-entropy predictions for absorption. 
This reduces the above correct-to-wrong reversions from 9.4\% to 2.6\% of generated cells and the fraction of token changes occurring after $t=0.75$ from 46\% to 15\%.

We then provide a conditional analysis of \name{} by analyzing the error introduced by irreversible absorption. Absorbing several positions independently can discard their dependence, while previously absorbed values affect the context of subsequent predictions. 
We characterize the stepwise absorption error through a Kullback-Leibler (KL) decomposition into joint dependence and conditional prediction terms. 
We show that selecting the lowest-entropy positions, under entropy-error regularity on oracle-reachable states and a fixed absorption count, minimizes an upper bound on the conditional error term.
We provide further analysis on global entropy lookahead~\citep{yang2026informationgain} which evaluates candidate absorptions through predicted successor states. 
With an imperfect denoiser, aggregating future estimates can compound errors, and the resulting worst-case decision-error bound scales with the lookahead window. This motivates an empirical comparison with local predictive entropy.

Experiments evaluate \name{} across tasks with increasing structural dependence.
The gains are largest on constrained reasoning, where \name{} attains 0.845 Nikoli Sudoku solve accuracy and attains the best average rank across the nine reasoning tasks.
On text-to-image generation \name{} attains the best GenEval overall score among decode-time samplers (2.8\% absolute over the native sampler), and on multimodal understanding it improves over the native sampler on all six benchmarks.  
We also examine the behavior motivating the analysis: entropy-guided absorption outperforms other absorption policies in puzzle solve accuracy. As the lookahead window increases, the mean absolute score error grows from 0.17 with a single window to 1.76 under full lookahead, while accuracy peaks at eight windows and then drops to 0.31. Lookahead therefore costs about eight times the inference runtime without bringing a significant reasoning improvement.
Our contributions are:
\begin{itemize}[leftmargin=*]
    \item We analyze rapid terminal transitions and correct-to-wrong revisions in uniform discrete flow, and introduce explicit generation order through selective absorption to preserve chosen predictions~(Sec.~\ref{sec:retention},~\ref{sec:commit}).
    \item We develop \name{}, a training-free entropy-guided absorption policy. Under entropy-error regularity on oracle-reachable states and a fixed absorption count, it minimizes an upper bound on the local conditional absorption error. We also show that a worst-case global lookahead decision-error bound grows with its window~(Sec. \ref{sec:ledflow-method},~\ref{sec:ledflow-theory}).
    \item We evaluate \name{} across reasoning, multimodal understanding, and image generation. Controlled experiments compare absorption orders and measure how lookahead window size affects score error and position selection, while runtime measurements and transfer comparisons against each model's default order assess efficiency and generality~(Sec. \ref{sec:exp}).
\end{itemize}

\section{Preliminaries and Motivation}\label{sec:prelim}
\subsection{Preliminaries}\label{sec:dfm}
We generate a target sequence
$a_1=(a_1^1,\ldots,a_1^M)\in\mathcal S^M$
over a finite vocabulary $\mathcal S$ and $M$ generated positions.
Let $x$ denote the fixed input, set $\mathcal C:=x$, and write $q(a_1\mid\mathcal C)$ for the target distribution, $\mathcal C_t$ the conditioning information available at time $t$, $q$ and $p^q$ the oracle (data) distributions and their conditionals, $p_t$ the prescribed probability path, and $\mathbb Q^\theta$ the learned denoiser posterior. Discrete flow matching transports a source distribution $p_0$ to the target
over $t\in[0,1]$ under a Continuous Time Markov Chain (CTMC)~\citep{campbell2024generative,gat2024discrete,shaul2025flow}. Let $a$ and $z$ denote the current state and a candidate destination state, respectively. For any state $x$, let $x^{-j}$ denote all coordinates except $j$. We use the discrete delta probability mass function $\delta(x,y):=\mathbf{1}\{x=y\}$. A factorized probability velocity is $u_t(z,a)=
\sum_{j=1}^{M}
\delta(z^{-j},a^{-j})\,u_t^j(z^j,a)$, where $u_t^j$ gives the instantaneous transition rate of coordinate $j$.
For a conditional path factorized on coordinates
$p_{t\mid1}(a_t\mid a_1)=\prod_j p^j_{t\mid1}(a_t^j\mid a_1^j)$,
the marginal coordinate velocity follows from posterior averaging: $u_t^j(z^j,a,\mathcal C_t)
=
\mathbb E_{a_1\sim p_{1\mid t}(\cdot\mid a,\mathcal C_t)}
\!\left[
u_t^j(z^j,a^j\mid a_1^j)
\right]$. At generation time, the reverse process $p_{1\mid t}$ is approximated by the learned denoiser
$\mathbb Q^\theta_{1\mid t}(\cdot\mid a_t,\mathcal C_t)$. FUDOKI~\citep{wang2025fudoki} implements the uniform distribution framework in \cite{gat2024discrete} with a metric-induced token path: $p^j_{t\mid1}(a^j\mid a_1^j)
=
\operatorname{softmax}_{a^j\in\mathcal S}
\!\left[-\beta_t d(a^j,a_1^j)\right]
\propto
\exp\!\left[-\beta_t d(a^j,a_1^j)\right]$, and $d$ corresponds to the token distance between embeddings, $\beta_t=c\left(t/(1-t)\right)^{\alpha}$ %with $c=3$, $\alpha=0.9$, 
so that $\beta_0=0$ and $\beta_1\rightarrow\infty$, allowing the source to be uniform over the vocabulary at $t=0$ and the path to concentrate on target token as $t\rightarrow 1$. By denoting $\dot\beta_t$ the derivative of $\beta_{t}$, %=c\alpha\,t^{\alpha-1}(1-t)^{-\alpha-1}$ the derivative of $\beta_{t}$, %it increases without bound as $t\rightarrow1$, so that 
the kinetic optimal velocity is 
\begin{equation}
u_t^j(z^j,a_t\mid a_1^j)
=
p_t(z^j\mid a_1^j)\,\dot\beta_t
\left[
d(a_t^j,a_1^j)-d(z^j,a_1^j)
\right]_+,
\label{eq:kinetic-velocity}
\end{equation}
where $[\cdot]_+=\max\{\cdot,0\}$. Probability mass moves from $a_t^j$ to $z^j$ only when
$d(z^j,a_1^j)<d(a_t^j,a_1^j)$, allowing the flow to progress monotonically toward the target value $a^{j}_{1}$ supplied to the velocity field~\citep{wang2025fudoki}. 
However, at inference time where the true target is unavailable, FUDOKI's Euler solver instead samples a predicted target $\hat a_1^j\sim\mathbb Q^\theta_{1\mid t}(\cdot\mid a_t,\mathcal C_t)$ at every step and measures the velocity (Eq.~\ref{eq:kinetic-velocity}) at that sample, so that every generated position, including those already holding a reliable prediction, is resampled at every step.
Detailed formulations from existing work are provided in App.~\ref{app:dfm-background}.

%==============================================================================
% Section 3: empirical motivation and selective absorption.
%==============================================================================
%Uniform discrete flow allows every generated position to change throughout sampling.
%While this supports correction, it also allows later denoiser errors which replace correct intermediate values.
\subsection{Motivation and Ordered Absorption}
\label{sec:retention}
The schedule gradient $\dot\beta_t$ in Eq.~\ref{eq:kinetic-velocity} diverges as $t\rightarrow 1$, which increases transition exposure at positions that disagree with the predicted target. With denoiser error, this exposure costs correctness.
\begin{proposition}[Terminal transition exposure]\label{cor:terminal-blowup}
At time $t$, consider position $j$ whose current token differs from its target $a^{j}_{t}\neq a^{j}_{1}$. The total outflow rate under Eq.~\ref{eq:kinetic-velocity} is
$\lambda_t^j=\dot\beta_t\,G_t^j$, where
$G_t^j=\sum_{z\in\mathcal S}p_t(z\mid a_1^j)\big[d(a_t^j,a_1^j)-d(z,a_1^j)\big]_+
\ge p_t(a_1^j\mid a_1^j)\,d(a_t^j,a_1^j)>0$.
Since $\dot\beta_t=c\alpha\,t^{\alpha-1}(1-t)^{-\alpha-1}\rightarrow\infty$ while $p_t(a_1^j\mid a_1^j)\rightarrow1$ as
$t\rightarrow1$, the rate diverges, $\lambda^{j}_{t}\rightarrow \infty$, and the corresponding frozen-rate jump probability for a fixed step size $h>0$ satisfies: $1-e^{-h\lambda_t^j}\rightarrow1$.
When the predicted target $\hat a_1^j$ is redrawn, disagreement with it exposes the current token to this increasing rate.
\end{proposition}
To assess net effect on accuracy between correction and incorrect revision, Fig.~\ref{fig:overjumpy} illustrates the latter by pretraining a flow model with Sudoku puzzles and evaluating on a held-out set of $100$ puzzles with 64 sampling steps:
46\% of token changes occur after $t=0.75$, with 9.4\% correctly generated cells replaced incorrectly at the end. 
This results in a puzzle solve rate of 41.0\%, exposing a tension between continued revision of wrong tokens, and the preservation of correct intermediate predictions. %, which the terminal growth of the transition rate erodes exactly when the remaining predictions should be settling.

Selective absorption addresses this tension by making chosen predictions terminal while the other positions keep evolving under the native flow, but it introduces a new problem: deciding which predictions are reliable enough to be made irreversible.
To prioritize reliable predictions at the current state and reassess remaining positions as the context changes after each absorption, the denoiser’s own uncertainty is a natural signal: a concentrated (diffuse) posterior has little (more) to gain from further revision.
Sec.~\ref{sec:sampler} formalizes selective absorption as an augmentation of uniform discrete flow and turns this principle into an ordering policy driven by local entropy.
App.~\ref{app:retention} provides the proof and scope of the proposition, together with the training and evaluation settings of this diagnostic. %\textcolor{blue}{*}
\section{\name: Low-Entropy Discrete Flow}\label{sec:sampler}
%\section{Motivation and Ordered Absorption}\label{sec:limits}
%\input{sections/3_limitations_old}

%\section{\name: Low-Entropy Discrete Flow}\label{sec:sampler}

%\input{sections/4_sampler_old}
%==============================================================================
% Section 4: LEDFlow sampler followed by theoretical justification.
%==============================================================================
\subsection{Selective absorption as an absorbing-state uniform flow}
\label{sec:uniform-absorption}\label{sec:commit}
We add an inference time indicator function into uniform discrete flow to change the sampler law.
Specifically, each generated position is augmented with an absorption indicator $f_t^j\in\{0,1\}$.
Every position starts active with $f_0^j=0$ and follows the uniform-flow velocity.
Once $f_t^j=1$, its outgoing transition rates are set to zero, hence keeping uniform velocity at active positions unchanged while fixing absorbed positions:
\begin{equation}\label{eq:absorbing-velocity}
u_t^{\theta,\mathrm{abs},j}(z^j,a_t,f_t,\mathcal C_t)
=
(1-f_t^j)\,
u_t^{\theta,j}(z^j,a_t,\mathcal C_t),
\qquad z^j\neq a_t^j .
\end{equation}
An absorbed position is a fixed point of the sampler regardless of what the denoiser later predicts for it, while every active position keeps the kinetic-optimal velocity of Eq.~\ref{eq:kinetic-velocity}, now evaluated at a state that contains the absorbed values.
The augmentation leaves two decisions to the sampler: how many positions to absorb at each step (absorption count), and which positions to absorb (absorption order).
While the absorption schedule can be fixed in advance, the order determines which predictions become irreversible and which values condition the rest of the generations (Sec.~\ref{sec:retention}).
Under this formulation, we introduce \name{}, an absorption-ordering policy guided by local entropy, and analyze its local absorption risk.
App.~\ref{app:absorbing-background} details the augmented process and its training-time counterpart for the puzzle denoisers.

\subsection{The \name{} sampler absorbs by local entropy}
\label{sec:ledflow-method}
\begin{wrapfigure}{r}{0.5\textwidth}
\vspace{-0.9\baselineskip}
\begin{minipage}{\linewidth}
\footnotesize
\refstepcounter{algorithm}\label{alg:ledflow}
\hrule height 0.8pt\vspace{2pt}
\noindent\textbf{Algorithm~\thealgorithm:} \name{} selective absorption over uniform discrete flow
\vspace{2pt}\hrule\vspace{2pt}
\begin{algorithmic}[1]
\State $a_{t_0}\sim p_0,\quad f_{t_0}^j\gets0\ \forall j$
\For{$k=0,\ldots,K-1$}
    \State $(a_{t_k}^{+},f_{t_k}^{+})\gets(a_{t_k},f_{t_k})$
    \State $\mathcal U_k\gets\{j:f_{t_k}^j=0\}$,\ \ $s_k\gets(a_{t_k},f_{t_k},\mathcal C_{t_k})$
    \State Compute $\mathbb Q^\theta_{1\mid t_k}(\cdot\mid s_k)$,\ $\{H_\theta^j(s_k)\}_{j\in\mathcal U_k}$
    \State $b_k\gets\max\{0,|\mathcal U_k|-m_k\}$
    \State $B_k^{\mathrm{LED}}\gets\text{Eq.~\eqref{eq:led-block}}$
    \For{$j\in B_k^{\mathrm{LED}}$}
        \State $a_{t_k}^{j,+}\gets\arg\max_{z\in\mathcal S}
        \mathbb Q^\theta_{1\mid t_k}(a_1^j=z\mid s_k)$
        \State $f_{t_k}^{j,+}\gets1$
    \EndFor
    \State Evolve active $j\notin B_k^{\mathrm{LED}}$ to $t_{k+1}$ by
    Eq.~\eqref{eq:absorbing-velocity}; absorbed $j$ stay fixed
    \State $f_{t_{k+1}}\gets f_{t_k}^{+}$
\EndFor
\State \Return $a_{t_K}$
\end{algorithmic}
\vspace{2pt}\hrule height 0.8pt
\end{minipage}
\vspace{-\baselineskip}
\end{wrapfigure}

\name{} selects active positions with low predictive entropy, absorbs their most likely values, and continues evolving the remaining positions under the uniform-flow velocity (illustrated in Algorithm~\ref{alg:ledflow}).
The absorption indicator in Eq.~\ref{eq:absorbing-velocity} addresses late revisions in Sec.~\ref{sec:retention} by retaining selected predictions and preventing further transitions. Let $s_k=(a_{t_k},f_{t_k},\mathcal C_{t_k})$ be the sampler state before absorption and
$\mathcal U_k:=\{j:f_{t_k}^j=0\}$ the active positions. Write
$\mathbb Q^\theta_{1\mid t_k}(a_1^j\mid s_k)$ for the denoiser marginal at position $j$, with predictive entropy: $H_\theta^j(s_k)
:=
-\sum_{z^j\in\mathcal S}
\mathbb Q^\theta_{1\mid t_k}
(a_1^j=z^j\mid s_k)
\log
\mathbb Q^\theta_{1\mid t_k}
(a_1^j=z^j\mid s_k)$. 
At grid point $t_k$, let
$m_k$ be the target number of active positions and set
$b_k:=\max\{0,|\mathcal U_k|-m_k\}$. The entropy-guided ordering policy selects the $b_k$ active positions with lowest predictive entropy and updates the indicator for every selected position,
\begin{equation}
B_k^{\mathrm{LED}}
\in
\arg\min_{\substack{B\subseteq\mathcal U_k\\|B|=b_k}}
\sum_{j\in B}H_\theta^j(s_k)\Longrightarrow f_{t_k}^{j,+}
=
\begin{cases}
1, & j\in B_k^{\mathrm{LED}},\\
f_{t_k}^j, & j\notin B_k^{\mathrm{LED}},
\end{cases}
\label{eq:led-block}
\end{equation}
\begin{equation}\label{eq:absorption-update}
a_{t_k}^j
\leftarrow
\arg\max_{z^j\in\mathcal S}
\mathbb Q^\theta_{1\mid t_k}(a_1^j=z^j\mid s_k),
\qquad
f_{t_k}^j\leftarrow1,
\quad j\in B_k^{\mathrm{LED}}.
\end{equation}
This policy turns each selected posterior into a fixed token prediction, changing the sampler law while retaining the uniform flow update in Eq.~\ref{eq:absorbing-velocity} for active positions outside $B_k^{\mathrm{LED}}$.
At the next time grid, the posterior and entropy are recomputed from the updated state to reassess subsequent decisions based on absorbed values. %so subsequent decisions condition on the absorbed values.
In the trajectory run (Fig.~\ref{fig:overjumpy}(a--c)), \name{} reduces the fraction of token changes after $t=0.75$ from 46\% to 15\% and raises the puzzle solve rate from 41.0\% to 84.5\%.
%With ordered absorption, the fraction of token changes occurring after $t=0.75$ drops from $46\%$ to $15\%$, \textcolor{red}{while correct-to-wrong rates decrease from $9.4\%$ to $2.6\%$ of generated cells.}
%Puzzle solve rate increases from $34.0\%$ to $82.0\%$, and final cell accuracy rises from $80.6\%$ to $89.9\%$ (Fig.~\ref{fig:overjumpy}).
This ordering policy is training free, as the absorption indicators can be directly plugged into the denoiser of the same checkpoint. Puzzle checkpoints take absorption indicators as an input feature and serve every absorption policy, including no absorption with indicators fixed at zero, which recovers the base uniform discrete flow sampler.
%The ordering policy requires no training. On frozen FUDOKI, $f_t$ is maintained by the sampler and is not a denoiser input; the puzzle denoisers instead receive absorption indicators during training. 
%Setting $f_t^j=0$ throughout recovers the base uniform discrete flow sampler for the same checkpoint.
%This most-likely-value rule turns each selected posterior into a fixed token prediction; Corollary~\ref{cor:risk-bridge-main} in Sec.~\ref{sec:ledflow-theory} bounds its excess absorption risk through the conditional prediction error.
%Using current-state entropy also avoids evaluating hypothetical future configurations; Lemma~\ref{lem:gain-amplification} in Sec.~\ref{sec:ledflow-theory} analyzes how such lookahead can amplify estimation error under an imperfect denoiser.
We provide a conditional analysis of the selection of low local entropy for absorption order in Sec.~\ref{sec:ledflow-theory}.
App.~\ref{app:ledflow-details} provides the schedule, update order, and matched-policy details.
\begin{figure}[t]\centering
\begin{minipage}[c]{0.46\textwidth}\centering
  \includegraphics[width=\linewidth]{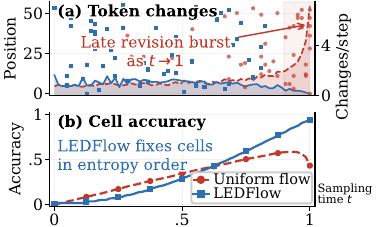}
\end{minipage}\hspace{-0.001\textwidth}%
\begin{minipage}[c]{0.52\textwidth}\centering
  \resizebox{\textwidth}{!}{%
  \begin{tabular}{@{}lccc@{}}
  \toprule
  \multicolumn{4}{c}{\textbf{(c) Trajectory-run statistics}}\\
  \midrule
  Measurement & \makecell{Uniform flow\\(no absorption)} & \makecell{LEDFlow\\(ours)} & \makecell{LEDFlow $-$\\Uniform flow}\\
  \midrule
  \makecell[l]{Token changes per position} & $1.49$ & $1.06$ & $-0.43$\\
  \cdashline{1-4}[2pt/1pt]
  \makecell[l]{Fraction of positions\\changed more than once} & $0.30$ & $0.13$ & $-0.17$\\
  \cdashline{1-4}[2pt/1pt]
  \makecell[l]{Fraction of token changes\\occurring after $t=0.75$} & $0.46$ & $0.15$ & $-0.31$\\
  \cdashline{1-4}[2pt/1pt]
  \makecell[l]{Mean sampling time of\\a token's last change} & $0.67$ & $0.48$ & $-0.19$\\
  \midrule
  %\makecell[l]{Cell-fill accuracy (final)} & $0.806$ & $\mathbf{0.899}$ & $+0.093$\\
  \makecell[l]{Puzzle solve rate} & $0.410$ & $\mathbf{0.845}$ & $+0.435$\\
  \midrule
  \multicolumn{4}{@{}l}{\emph{Trajectory correctness (fraction of cells)}}\\
  \makecell[l]{Correct to wrong} & $0.094$ & $\mathbf{0.026}$ & $-0.068$\\
  \bottomrule
  \end{tabular}}
\end{minipage}
\caption{\textbf{Ordered absorption reduces late revisions and improves puzzle solving.}
\textbf{(a,b)} Token-change timing and cell accuracy on 100 Nikoli puzzles with $K=64$. Without absorption, cell accuracy falls rapidly from 0.58 to 0.43 at $t\approx 0.94$, whereas \name{} rises monotonically.
\textbf{(c)} Statistics from the trajectory run: correct-to-wrong reversions decrease from 0.094 to 0.026 of all generated cells (cells correct at an intermediate step but incorrect at the endpoint).}
\label{fig:overjumpy}
\end{figure}

%==============================================================================

\subsection{Theoretical analysis}
\label{sec:ledflow-theory}
\label{sec:limit}\label{sec:limit-order}\label{sec:relorder}
This section analyzes local entropy selection under an entropy-error assumption.
We first characterize the error of irreversible absorption, then establish the conditional-error upper bound minimized by low-entropy selection, and compare the sensitivity of local and lookahead scores to estimation error.

\paragraph{Error induced by absorption.}
Noting that absorption is irreversible, its cost depends on the denoiser during each absorption process.
At an oracle-reachable sampler state $s_k$, where the absorbed partial assignment has positive probability under the target distribution, let $B_k\subseteq\mathcal U_k$ be the positions absorbed at that step.
Given the target distribution $q$, we write $p^q_{1\mid s_k}(a_1^{B_k})$ for the oracle posterior, with the denoiser marginal defined in Sec.~\ref{sec:ledflow-method}.
We apply KL decomposition~\citep{xu2026scheduling,benhamu2025ebsampler}, to a stochastic reference kernel that samples selected positions independently from denoiser marginals.
%Algorithm~\ref{alg:ledflow} instead absorbs their modes; Corollary~\ref{cor:risk-bridge-main} provides the link to that deterministic rule.

\begin{proposition}[Decomposition of local absorption error]\label{prop:kl-decomp}
Conditioned on a reachable state $s_k$ and an absorption block $B_k$, we write the error $\mathcal L_{\mathrm{abs}}(B_k;s_k)
:=
\KL\!\left(
p^q_{1\mid s_k}(a_1^{B_k})
\,\middle\|\,
\prod_{j\in B_k}
\mathbb Q^\theta_{1\mid t_k}(a_1^j\mid s_k)
\right)$, such that 
\begin{equation}\label{eq:kl-decomp}
\mathcal L_{\mathrm{abs}}(B_k;s_k)
=
\underbrace{\mathcal{TC}_q(a_1^{B_k}\mid s_k)}
_{\mathcal E_{\mathrm{joint}}(B_k;s_k)}
+
\underbrace{
\sum_{j\in B_k}
\KL\!\left(
p^q_{1\mid s_k}(a_1^j)
\,\middle\|\,
\mathbb Q^\theta_{1\mid t_k}(a_1^j\mid s_k)
\right)}
_{\mathcal E_{\mathrm{cond}}(B_k;s_k)},
\end{equation}
where
$\mathcal{TC}_q(a_1^{B_k}\mid s_k)
=\sum_{j\in B_k}H_q(a_1^j\mid s_k)-H_q(a_1^{B_k}\mid s_k)$
is the conditional total correlation of the positions absorbed together.
\end{proposition}
\begin{corollary}[Conditional error bounds the excess absorption risk]\label{cor:risk-bridge-main}
For a reachable state $s_k$ and a block $B_k$ with $|B_k|=b$, let $\hat z^j$ be the most likely value under $\mathbb Q^\theta_{1\mid t_k}(a_1^j\mid s_k)$, $c_j=p^q_{1\mid s_k}(a_1^j=\hat z^j)$ its oracle probability, and $c_j^\star=\max_z p^q_{1\mid s_k}(a_1^j=z)$ the probability of the oracle's best value. Then
\begin{equation}\label{eq:risk-bridge-main}
\sum_{j\in B_k}(1-c_j)
\;\leq\;
\sum_{j\in B_k}(1-c_j^\star)
+
\sqrt{2b\,\mathcal E_{\mathrm{cond}}(B_k;s_k)}.
\end{equation}
\end{corollary}
The left side of Eq.~\ref{eq:risk-bridge-main} is the expected number of incorrect deterministic absorptions in the selected block, while the first term on the right is the unremovable marginal classification risk of the oracle.
The square-root term bounds the excess risk due to denoiser error by Pinsker's inequality and Cauchy-Schwarz.
While larger blocks reduce sequential evaluations at the cost of unoptimized dependence error, the joint term vanishes for $b=1$.
App.~\ref{app:order} gives the derivation of the decomposition and the proof of the corollary.
%==============================================================================

\paragraph{Minimizing the conditional error bound.}

For the active set $\mathcal U_k$ at a reachable sampler state $s_k$, fix an absorption count $b\geq1$ and let
$\mathfrak B_k(b)=\{B\subseteq\mathcal U_k:|B|=b\}$.
The oracle posterior serves as an analytical reference for selecting positions with small model error. If it were available, the conditional-error objective would be:
\begin{equation}
B_k^{\mathrm{oracle}}
\in
\arg\min_{B\in\mathfrak B_k(b)}
\underbrace{
\sum_{j\in B}
\KL\!\left(
p^q_{1\mid s_k}(a_1^j)
\,\middle\|\,
\mathbb Q^\theta_{1\mid t_k}(a_1^j\mid s_k)
\right)
}_{\mathcal E_{\mathrm{cond}}(B;s_k)}.
\label{eq:objective}
\end{equation}

%The oracle posterior is unavailable at inference. To prioritize predictions that are currently reliable, we formulate a reliability condition based on denoiser uncertainty, quantified by predictive entropy. This is motivated by the negative relationship between prediction entropy and reasoning accuracy~\cite{yang2026informationgain}. 

The oracle posterior is unavailable at inference. To prioritize predictions that are currently reliable, we formulate a reliability condition based on denoiser uncertainty, quantified by predictive entropy. This is motivated by the negative relationship between prediction entropy and reasoning accuracy~\cite{yang2026informationgain}. 

%Motivated by the negative relationship between prediction entropy and reasoning accuracy in~\citet{yang2026informationgain} motivates this condition. 
%We therefore assume a \textcolor{red}{reliability envelope} relating its mismatch with the denoiser to predictive entropy.

\begin{definition}[Entropy-error regularity]\label{def:monotone}
Let $H_{\max}=\log|\mathcal S|$. A denoiser is \emph{$\phi$-regular} on the oracle-reachable
sampler states $s_{k}$ if there exists a nondecreasing function
$\phi:[0,H_{\max}]\to\mathbb R_{\geq0}$, with $\phi(0)=0$, such that $\KL\!\left(
p^q_{1\mid s_k}(a_1^j)
\,\middle\|\,
\mathbb Q^\theta_{1\mid t_k}(a_1^j\mid s_k)
\right)
\leq
\phi\!\left(H_\theta^j(s_k)\right)$ for any $s_{k}$ and $j\in\mathcal U_k$.
\end{definition}

Applying Def.~\ref{def:monotone} to the conditional term in
Eq.~\ref{eq:kl-decomp} gives, for any absorption block
$B_k\subseteq\mathcal U_k$, $\mathcal E_{\mathrm{cond}}(B_k;s_k)
\leq
\Phi(B_k;s_k)
:=
\sum_{j\in B_k}\phi\!\left(H_\theta^j(s_k)\right)$.
Because $\phi$ is nondecreasing, the block that minimizes $\Phi$ under a fixed absorption count contains the active positions with lowest predictive entropy, so the square root term in Eq.~\ref{eq:risk-bridge-main} is at most $\sqrt{2b\,\Phi(B_k;s_k)}$, so entropy selection minimizes this surrogate excess-risk bound for a fixed absorption count.
%When one position is absorbed, $B_k=\{j_k\}$, the joint term in Eq.~\ref{eq:kl-decomp} vanishes, so that an entropy-guided ordering policy controls the upper bound of entire local absorption error.

\begin{theorem}[Local surrogate minimization and the oracle limit]\label{thm:order}
Given a $\phi$-regular denoiser:
\emph{(i) Local adaptive absorption.}
For a fixed absorption count $b\geq1$, a block containing $b$ active positions with lowest entropy minimizes
$\Phi(B;s_k)$ over $B\in\mathfrak B_k(b)$. In particular, for single position absorption,
\begin{equation}
j^\star
\in
\arg\min_{j\in\mathcal U_k}H_\theta^j(s_k)
\label{eq:entropy-policy}
\end{equation}
minimizes the upper bound on the conditional error of the next irreversible prediction.
\emph{(ii) Oracle limit.}
As the learned posterior converges to the oracle posterior on reachable states,
$\mathcal E_{\mathrm{cond}}(B_k;s_k)\to0$ for every admissible absorption block and hence
for every absorption policy.
\end{theorem}

Part (i) follows from monotone sorting under Def.~\ref{def:monotone}, while part (ii) gives the oracle limit common to all policies.
For $b>1$, the theorem minimizes the conditional-term bound in Eq.~\ref{eq:kl-decomp}.
The guarantee is local, since each absorption changes the sampler state and subsequent errors.
The condition $\phi(0)=0$ excludes confidently wrong deterministic predictions.
Although Def.~\ref{def:monotone} uses predictive entropy, the analysis extends to any local uncertainty score whose conditional error is bounded by a nondecreasing function of that score, including maximum probability and probability margin under suitable regularity conditions.
While it holds for 95.5\% of sampled Sudoku states in Fig.~\ref{fig:overjumpy}, where the oracle posterior is computable, we treat it as idealized.
App.~\ref{app:order} provides the proof and the relaxed offset.
We next compare local greedy decoding with estimation over future states.

\paragraph{Global lookahead can amplify estimation error.}
Global lookahead decoding (represented by Info-Gain~\citep{yang2026informationgain}) evaluates information gain by estimating predicted uncertainty over all remaining active positions.
Let $\mathcal U(s)$ contain $m$ active positions at state $s$, and let
$h_j^q(s)$ and $h_j^\theta(s)$ denote the oracle and model entropies.
Define $R_r(s):=\sum_{j\in\mathcal U(s)}h_j^r(s)$, $r\in\{q,\theta\}$.
For an active candidate position $j\in \mathcal{U}_{k}$ and absorbed value $z^j$, let $s^{j,z^j}$ denote the
hypothetical state after absorption obtained by setting coordinate $j$ to $z^j$. Writing the model and oracle posterior as $P_\theta^j(\cdot\mid s)=\mathbb Q^\theta_{1\mid t}(a_1^j=\cdot\mid s)$, $P_q^j(\cdot\mid s)=p^q_{1\mid s}(a_1^j=\cdot)$, 
the corresponding entropy reductions are $G_j^r(s)
:=
R_r(s)
-
\mathbb E_{z^j\sim P_r^j(\cdot\mid s)}
\left[
R_r(s^{j,z^j})
\right]$, $r\in\{q,\theta\}$.
Assume that at current state $s$ and every hypothetical one-step lookahead state $s'$ with positive model probability,
$|h_j^\theta(s')-h_j^q(s')|\leq\delta$ for every active $j$, and
the total variation $\frac{1}{2}\sum_{z^j\in\mathcal S}\left|P_\theta^j(z^j\mid s)-P_q^j(z^j\mid s)\right|\leq \epsilon$ for every candidate $j$.

\begin{lemma}[Error amplification in global entropy reduction]
\label{lem:gain-amplification}
Let
$\hat j\in\arg\max_jG_j^\theta(s)$,
$j^\star\in\arg\max_jG_j^q(s)$, and
$\hat j_{\mathrm{ent}}\in\arg\min_jh_j^\theta(s)$.
Because $R_r(s)$ is common to every candidate at $s$, it cancels from the comparison between
candidates, and for every pair $j,j'$ the error of the decision-relevant score difference obeys
\[
\begin{aligned}
\left|\bigl(G_j^\theta(s)-G_{j'}^\theta(s)\bigr)-\bigl(G_j^q(s)-G_{j'}^q(s)\bigr)\right|
&\leq
2\eta_m,
\qquad
\eta_m:=(m-1)\bigl(\delta+\epsilon\log|\mathcal S|\bigr),\\
G_{j^\star}^q(s)-G_{\hat j}^q(s)
&\leq2\eta_m,
\qquad
h_{\hat j_{\mathrm{ent}}}^q(s)-\min_j h_j^q(s)
\leq2\delta.
\end{aligned}
\]
\end{lemma}
While the bounds do not imply a direct performance ordering, they clarify how estimation error affects the two objectives.
The worst-case lookahead bound scales with window size $m$ through accumulated error over the remaining $m-1$ positions, and preserves the oracle ranking whenever the candidate margin exceeds $2\eta_m$.
%the deficit of lookahead relative to local entropy shrinks as the denoiser improves, without closing at the capacities we test (\S\ref{sec:exp-sudoku}).
App.~\ref{app:order} details derivation and analysis outside oracle support. %the full derivation, the treatment of lookahead states outside the oracle support, and the breakeven characterization.

\section{Experiments}\label{sec:exp}

% \paragraph{Setup.} 
\begin{figure}[t]
\centering
\includegraphics[width=1\textwidth]{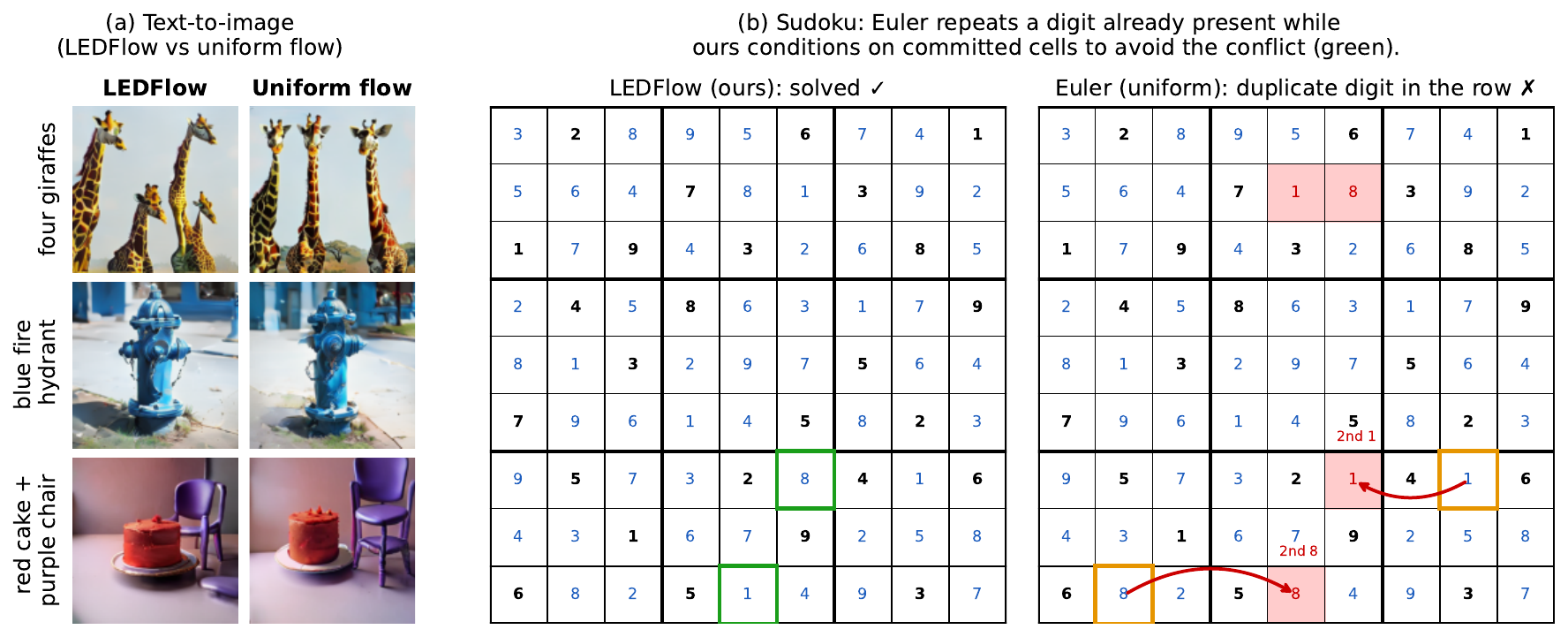}
\vspace{-4mm}
\caption{\textbf{Qualitative comparison across the structure spectrum.} \textbf{(a)} Text-to-image
on FUDOKI shows \name{} produces images with improved alignment with the required descriptions. \textbf{(b)} Sudoku: unordered Euler sampler repeats digits already present (two `1's and `8's in row~7 and~9), violating the row constraint, while \name{} conditions on absorbed cells and avoids the conflict. }
\label{fig:qualitative} \vspace{-4mm}
\end{figure}
\noindent\textbf{Setup.} We evaluate \name{} across tasks with increasing structural dependence: six multimodal-understanding benchmarks via VLMEvalKit~\citep{duan2024vlmevalkit}, GenEval text-to-image generation~\citep{ghosh2023geneval}, and reasoning on MathVista, MathVerse, GSM8K, three Sudoku sets, Latin-square, graph coloring, and molecular infilling~\citep{lu2024mathvista,zhang2024mathverse,gsm8k,seely2025sudoku,wang2025hrm,polykovskiy2020moses}. 
We use FUDOKI~\citep{wang2025fudoki} with matched configurations for understanding, generation, and mathematics, and task-specific uniform-flow denoisers for puzzles, evaluating \name{} in a training-free setting. 
All tasks use $K=64$ sampling steps. 
Evaluation uses full standard splits, 553 GenEval prompts with four images each, 500/100/700 generated/Nikoli/Extreme Sudoku instances, and five seeds.
We compare two baseline groups. The first contains non-absorbing decoders: Euler~\citep{wang2025fudoki}, Time- and Location-corrected samplers~\citep{wan2026corrected}, higher-order solvers (RK, RK2-Trapezoid)~\citep{zhang2026error}, exact guidance~\citep{discreteguidance2026}, and FreeCorrection~\citep{ouyang2026training}.
The second contains uniform flow with different absorption policies: probability margin (Top K-margin) and global lookahead (Info-Gain)~\citep{yang2026informationgain}.
App.~\ref{app:setup} provides experimental details.

\subsection{Empirical Performance}\label{sec:empirical-performance}
\begin{table}[!t]
\centering
\begingroup
\setlength{\abovecaptionskip}{2pt}
\setlength{\belowcaptionskip}{2pt}
\renewcommand{\arraystretch}{0.90}
\vspace{1pt}
\begingroup\centering\small
\caption{\textbf{Multimodal understanding scores.} White/grey/blue rows denote non-absorbing baselines/absorbing baselines/\name{}. \textbf{Bold}, \textit{italic}, and \underline{underlining} mark first, second, and third place. Avg.\ rank is the mean benchmark rank ($\downarrow$). Higher scores are better ($\uparrow$). Entropy-guided selective absorption gives consistently higher but
modest gains. }
\label{tab:mm-understanding}
\resizebox{\columnwidth}{!}{%
\begin{tabular}{lcccccc|c}
\toprule
sampler & POPE$\uparrow$ & MME$\uparrow$ & MMBench$\uparrow$ & GQA$\uparrow$ & MMMU$\uparrow$ & MM-Vet$\uparrow$ & Avg. rank$\downarrow$\\
\midrule
Euler & 86.1 & 1485.4 & \underline{73.9} & \underline{57.6} & 34.3 & 38.0 & 4.33\\
Euler (our replication) & \underline{87.4} & 1494.5 & 73.7 & 57.0 & \textit{36.2} & 38.2 & \underline{3.50}\\
Exact guidance & 86.8 & 1492.7& \textit{74.2} & \textbf{58.2} & \underline{35.4} & \underline{38.6} & \textit{3.00}\\
\midrule
\rowcolor{gray!14} Top K-margin & 83.5 & \underline{1497.4} & 65.0 & 56.9 & 30.0 & \textit{40.0} & 4.50\\
\rowcolor{gray!14} Info-Gain & \textbf{89.3} & \textit{1503.0} & 65.1 & 56.4 & 29.1 & 37.5 & 4.33\\
\cdashline{1-8}[2pt/1pt]
\rowcolor{blue!8}
\name{} (ours)          & \textit{89.0} & \textbf{1504.6} & \textbf{74.6} & \textit{58.0} & \textbf{37.4} & \textbf{40.5} & \textbf{1.33}\\
\bottomrule
\end{tabular}}
\endgroup
\par
\vspace{1pt}
\begingroup\centering\small
\setlength{\tabcolsep}{3.5pt}
\caption{\textbf{Text-to-image generation scores.} Wider spread from \name{} on absorption policies than Table~\ref{tab:mm-understanding} is consistent with text-to-image prompts imposing more structure than text generation.}
\label{tab:mm-geneval}
\resizebox{\columnwidth}{!}{%
\begin{tabular}{lccccccc|c}
\toprule
Sampler & Single$\uparrow$ & Two$\uparrow$ & Count$\uparrow$ & Colors$\uparrow$ & Pos.$\uparrow$ & Attr.$\uparrow$ & Overall$\uparrow$ & Avg. rank$\downarrow$\\
\midrule
Euler & \textit{0.9625} & 0.8384 & 0.4875 & 0.8833 & \textbf{0.7200} & 0.6300 & 0.7536 & 5.42\\
Euler~{(our replication)} & 0.9375 & \underline{0.8636} & 0.5031 & 0.8723 & 0.6800 & 0.6625 & 0.7532 & 5.92\\
Time-corrected & 0.9125 & \textit{0.8687} & 0.5125 & \textbf{0.9149} & 0.6600 & 0.7000 & 0.7614 & 4.83\\
Location-corrected & \textbf{0.9750} & \textit{0.8687} & 0.5000 & \textit{0.9043} & 0.6700 & 0.6400 & 0.7597 & \underline{4.08}\\
RK &  0.9250 & 0.8384 & \textbf{0.5625} & 0.8723 & 0.6500 & \textit{0.7400} & \underline{0.7647} & 6.17\\
RK2-Trapezoid & 0.9500 & 0.7980 & 0.4750 & \textit{0.9043} & 0.6700 & 0.6800 & 0.7462 & 5.67\\
Exact guidance & 0.9400 & 0.8600 & \textit{0.5300} & 0.8900 & \underline{0.7000} & \textbf{0.7700} & \textit{0.7800} & \textit{3.67}\\
\midrule
\rowcolor{gray!14} Top K-margin & 0.9469 & 0.7854 & 0.4188 & 0.8856 & 0.6275 & 0.5800 & 0.7071 & 8.17\\
\rowcolor{gray!14} Info-Gain &  0.9312 & 0.7753 & 0.4906 & 0.8777 & 0.6600 & 0.5550 & 0.7150 & 8.42\\
\cdashline{1-9}[2pt/1pt]
\rowcolor{blue!8} \textbf{\name{}} (ours) & \underline{0.9563} & \textbf{0.8712} & \underline{0.5219} & \underline{0.9016} & \textit{0.7150} & \underline{0.7225} & \textbf{0.7814} & \textbf{2.67}\\
\bottomrule
\end{tabular}}
\endgroup
\par
\vspace{1pt}
\begingroup\centering\small
\caption{\textbf{Constrained reasoning accuracy.} \name{} outperforms in constrained reasoning task. }
\label{tab:sudoku}
\setlength{\tabcolsep}{3.2pt}
\resizebox{\columnwidth}{!}{%
\begin{tabular}{lccccccccc|c}
\toprule
\multirow{2}{*}{Sampler} & \multicolumn{3}{c}{Mathematics$\uparrow$} & \multicolumn{3}{c}{Sudoku$\uparrow$} & \multicolumn{3}{c}{Other puzzles$\uparrow$} & \multirow{2}{*}{Avg. rank$\downarrow$}\\
\cmidrule(lr){2-4}\cmidrule(lr){5-7}\cmidrule(lr){8-10}
& MathVista & MathVerse & GSM8K & Generated & Nikoli & Extreme & Latin & Graph & Molec. & \\
\midrule
Euler & \underline{0.254} & 0.110 & \underline{0.026} & 0.610 & 0.410 & 0.078 & 0.652 & 0.936 & 0.828 & 5.22\\
FreeCorrection & \textbf{0.261} & \underline{0.119} & 0.019 & 0.630 & 0.400 & 0.126 & 0.628 & \textit{0.988} & 0.797 & 4.72\\
Time-corrected & 0.241 & 0.111 & 0.022 & 0.610 & 0.410 & 0.078 & 0.652 & 0.936 & 0.828 & 5.56\\
Location-corrected & 0.246 & \textit{0.120} & \textbf{0.029} & 0.588 & 0.360 & 0.096 & 0.628 & 0.938 & 0.830 & 4.78\\
\midrule
\rowcolor{gray!14} Top K-margin & \textit{0.255} & \textit{0.120} & 0.024 & \textit{0.820} & \textit{0.796} & \textbf{0.271} & \underline{0.948} & \underline{0.951} & \underline{0.887} & \textit{2.67}\\
\rowcolor{gray!14} Info-Gain & 0.227 & \underline{0.119} & 0.019 & \underline{0.804} & \underline{0.780} & \underline{0.231} & \textit{0.990} & \textbf{1.000} & \textbf{0.914} & \underline{3.50}\\
\cdashline{1-11}[2pt/1pt]
\rowcolor{blue!8}
\textbf{\name{}} (ours) & \textit{0.255} & \textbf{0.123} & \textit{0.027} & \textbf{0.865} & \textbf{0.845} & \textit{0.269} & \textbf{0.998} & \textbf{1.000} & \textit{0.906} & \textbf{1.56}\\
\bottomrule
\end{tabular}}
\endgroup
\par
\endgroup \vspace{-5mm}
\end{table}

\noindent\textbf{Multimodal understanding.}\label{sec:exp-mv}
Adding entropy-guided selective absorption gives consistently higher but modest gains over
the replicated Euler baseline across all six benchmarks (Table~\ref{tab:mm-understanding}), although
Info-Gain leads on POPE and exact guidance on GQA. These modest gains
place understanding at the weakly constrained end of the task spectrum, where absorption order matters less. Further analyses, full results, and the uncertainty of these gains are provided in App.~\ref{app:results-understanding}.

\noindent\textbf{Text-to-image generation.}\label{sec:exp-generation}
Table~\ref{tab:mm-geneval} shows that \name{} attains the highest overall score (0.7814) and the best average rank (2.67), ahead of decode-time exact guidance (0.7800). The resulting spread across absorption policies is wider than in multimodal understanding, consistent with text-to-image prompts imposing more structure than the understanding benchmarks above. Fig.~\ref{fig:qualitative} provides qualitative examples. We provide further generation experiments and comparisons in App.~\ref{app:results-generation}.

% -----------------------------------------------------------------------------
\noindent\textbf{Reasoning.}\label{sec:exp-sudoku}
Table~\ref{tab:sudoku} reports constrained reasoning results. \name{} reaches 0.845 on Nikoli against 0.796 for margin and 0.780 for lookahead ordering, and also leads these policies on the generated set while staying competitive in the Extreme set.
Across the nine tasks, \name{} is numerically higher than Info-Gain on seven and remains competitive on molecular infilling (0.906 against 0.914).
%While the bound in Thm.~\ref{thm:order} is exact at $b=1$ with a vanishing $\mathcal E_{\mathrm{joint}}$, 
Under the standard protocol ($K<M$ and $b_k>1$), \name{} improves empirically, while Thm.~\ref{thm:order} provides a local conditional-error guarantee.
App.~\ref{app:results-reasoning} reports detailed comparisons. %these comparisons in full, including robustness checks at larger scale with paired tests, against exact integration of the non-absorbing flow, and with alternative uncertainty scores.
% -----------------------------------------------------------------------------
\begin{figure}[t]
\centering
\includegraphics[width=\textwidth]{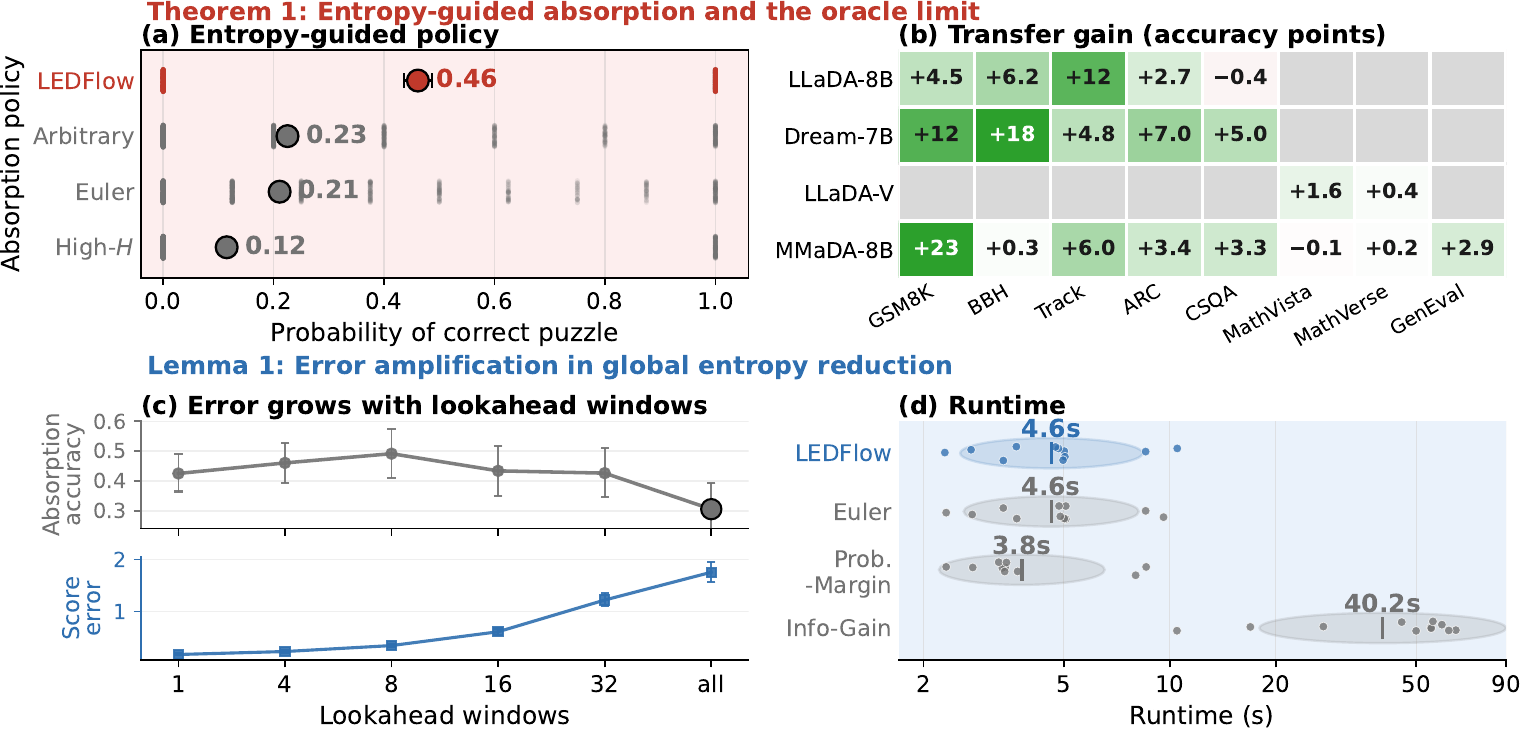}
\vspace{-2mm}
\caption{\textbf{Empirical analysis of entropy-guided absorption and the
score sensitivity of global lookahead.} \emph{Red (Thm.~\ref{thm:order}):} \name{}'s
entropy-guided absorption policy gives the highest puzzle accuracy (a), and its gain over each model's default generation order transfers across
pretrained models and tasks (b).
\emph{Blue
(Lemma~\ref{lem:gain-amplification}):} aggregating more lookahead terms increases score
error and eventually reduces absorption accuracy (c), with increased inference runtime (d).}
\label{fig:theory-aligned-analysis}
\vspace{-3mm}
\end{figure}
\subsection{Empirical Analysis of Absorption and Lookahead}\label{sec:theory-analysis}

\noindent\textbf{Entropy-guided absorption (Thm.~\ref{thm:order}).}
Fig.~\ref{fig:theory-aligned-analysis}(a) examines the ordering policy in Thm.~\ref{thm:order} pooled over Sudoku puzzles in Table~\ref{tab:sudoku}. Accuracy decreases monotonically from \name{} (0.461) to arbitrary (0.225), the non-absorbing Euler flow (0.211), and a controlled absorption of high-entropy tokens (0.115).
Panel~(b) shows that, relative to the default absorption order of four models' samplers, \name{} provides generalizable improvement, with largest gains on constrained reasoning benchmarks. Grey cells mark inapplicable settings.
Detailed results with statistical tests are provided in App.~\ref{app:results-generality}.

\noindent\textbf{Error amplification in global entropy reduction (Lemma~\ref{lem:gain-amplification}).}
Fig.~\ref{fig:theory-aligned-analysis}(c) holds the configuration of \name{} fixed and varies only the lookahead window aggregated by the Info-Gain policy~\citep{yang2026informationgain}.
We observe score sensitivity consistent with the analysis in Lemma~\ref{lem:gain-amplification}: the mean absolute score error grows from 0.17 with a single window to 1.76 under full lookahead, while accuracy peaks at eight aggregated windows and then falls to 0.31, consistent with accumulated estimation error affecting position selection under an imperfect denoiser.
Panel~(d) shows that across the eleven benchmarks, \name{} has an inference time similar to that of the uniform flow sampler, while Info-Gain costs about eight times as much.
App.~\ref{app:results-efficiency} provides a detailed analysis of efficiency and block size.
%==============================================================================

\section{Related work}\label{sec:related}
%==============================================================================
\noindent\textbf{Discrete flow sampling, correction, and absorption.}
Uniform discrete flows iteratively sample each position for repeated revisions~\citep{gat2024discrete,wang2025fudoki}, using Euler or $\tau$-leaping sampler on a time
grid~\citep{campbell2024generative,gat2024discrete,shaul2025flow}.
Corrected samplers reduce errors caused by approximating the CTMC within a sampling interval, while
uniformization removes this discretization error by simulating the process exactly
~\citep{wan2026corrected,zhang2026error,zhao2025informed}. 
Separately, while absorbing states are standard in masked diffusion~\citep{austin2021structured,he2022diffusionbert,shi2024simplified,lou2024discrete}, uniform diffusion admits conditional-absorbing representations~\citep{gourevitch2026uniformdiffusionmodelsrevisited,rutte2025generalized}, motivating work on persistent context~\citep{hayakawa2026truncationcommitmentpersistentcontext,deschenaux2026the}.
Building on these, we augment a
native uniform discrete flow with selective absorption and analyze its local KL error without
changing the uniform flow velocity on positions that remain active.

\noindent\textbf{Reveal order in masked generation.}
Masked generation has long used confidence to select predictions for irreversible decoding
~\citep{chang2022maskgit,kim2025train,ye2024beyond}, while planning and remasking methods learn which
tokens to reveal or revisit~\citep{peng2025path,hong2026improving,xu2026scheduling,wang2026remaskingdiscretediffusionmodels}.
Prioritizing confidence alone can postpone high entropy reasoning forks and reduce solution coverage~\citep{ni2026flexibility,webb2026limitsconfidencediffusion}.
Info-Gain estimates lookahead quality over the remaining positions~\citep{yang2026informationgain}, which depends on denoiser quality.
EB-Sampler~\citep{benhamu2025ebsampler} controls parallel unmasking block error with an entropy bound, confidence thresholds decode positions in parallel~\citep{wu2025fastdllmtrainingfreeaccelerationdiffusion}, while BoE steering~\citep{saini2026tabes} guides decoding with gradients of successor state entropy.
\name{} applies entropy-guided absorption to a revisable uniform flow with analysis of local prediction and dependence costs. 
App.~\ref{app:extended-related} provides broader scope of related work.

\section{Conclusion}\label{sec:concl}
We introduce generation order into uniform discrete flow through selective absorption, which fixes chosen predictions while allowing other positions to continue evolving.
Our training-free sampler, \name{}, absorbs the lowest-entropy positions first to protect reliable predictions from later errors.
We decompose local absorption error into joint dependence and conditional prediction terms, and show that under entropy-error regularity this policy minimizes an upper bound on the conditional term for a fixed absorption count.
We also show how a worst-case decision-error bound for global lookahead grows as it combines estimates across more positions.
Experiments on reasoning, text-to-image generation, and multimodal understanding show the largest gains on strongly constrained tasks, while trajectory diagnostics show fewer correct-to-wrong revisions.
Our analysis characterizes a local error bound under entropy-error regularity, with joint position selection, correction of absorbed predictions, and lookahead with more accurate denoisers offering directions for further study (App.~\ref{app:claim-scope}).

% Required by ICLR 2027; these statements do not count toward the page limit.
\section*{Ethics Statement}
This work introduces an inference-time sampler and involves no human subjects, no user study, and no collection of new data.
All experiments use publicly released pretrained models (FUDOKI, LLaDA, LLaDA-V, MMaDA, Dream) and public benchmarks, used within their stated licenses and intended research use.
The task-specific Sudoku, Latin-square, graph-coloring, and molecular-infilling denoisers are trained from scratch on programmatically generated or public puzzle sets that contain no personal or identifying information.
\name{} changes the decoding order of an existing model and adds no safety filtering, so a model decoded with \name{} inherits the biases, factual errors, and misuse potential of the underlying checkpoint, and the improvements we report on text-to-image and multimodal generation carry the dual-use risks common to generative models.
The molecular-infilling results are benchmark measurements on MOSES and are not a validation of any candidate for synthesis or downstream chemical use.
We report the settings in which \name{} gives little or no benefit, including weakly constrained multimodal understanding and the Sudoku-Extreme set, in Sec.~\ref{sec:exp} and App.~\ref{app:claim-scope}.

\section*{Reproducibility Statement}
The sampler is specified completely in the main text: the absorbing-state augmentation in Eq.~\ref{eq:absorbing-velocity}, the entropy-guided selection rule in Eq.~\ref{eq:led-block} and Eq.~\ref{eq:absorption-update}, and the full procedure in Algorithm~\ref{alg:ledflow}.
\name{} is training free and requires no training or fine-tuning of the denoiser: on a released backbone such as FUDOKI the indicator only gates the velocity, while the puzzle denoisers receive the indicator as an input feature during their single pretraining run, shared by every sampler (App.~\ref{app:absorbing-background}). Any uniform discrete flow checkpoint can therefore reproduce the sampler from these equations alone.
The theoretical claims state their assumptions explicitly in Def.~\ref{def:monotone} and in the reachability conditions of Sec.~\ref{sec:ledflow-theory}. App.~\ref{app:order} gives the proofs of Proposition~\ref{prop:kl-decomp}, Corollary~\ref{cor:risk-bridge-main}, Thm.~\ref{thm:order}, and Lemma~\ref{lem:gain-amplification}, together with the treatment of lookahead states outside the oracle support.
App.~\ref{app:setup} documents the experimental settings: checkpoints and capacities, the absorption schedule and count, the sampler step budget $K=64$, the denoiser training configuration for the puzzle tasks, and the evaluation split and sample count for every benchmark.
App.~\ref{app:sudoku_protocol} defines the two Sudoku protocols used in this paper and states which tables and figures belong to each, so that measurements from different protocols are not compared directly.
App.~\ref{app:ledflow-details} gives the schedule, update order, and matched-policy details needed to reproduce the baseline absorption policies under identical conditions.
Reported comparisons state their uncertainty: App.~\ref{app:results-understanding} gives interval estimates and their approximation limits, and App.~\ref{app:results-reasoning} reports puzzle-level bootstrap confidence intervals and paired McNemar tests at $n{=}1000$.
The supplementary material accompanying this submission releases the code and environment for the self-contained Sudoku and constrained-reasoning experiments: the task-specific denoiser and its training, the reveal-order samplers (\name{}, margin, and Info-Gain), and the trajectory, calibration, and efficiency diagnostics, which reproduce the results in Table~\ref{tab:sudoku}, Fig.~\ref{fig:overjumpy}, and Fig.~\ref{fig:theory-aligned-analysis}(a) without any external checkpoint.
These experiments require no external base model, so \name{} itself is fully specified by the equations above and needs no released code to reproduce.
The complete code for the text-to-image, multimodal-understanding, and 8B-transfer experiments, which build on third-party frameworks and their released checkpoints, will be released with the camera-ready version.

\section*{AI Use Statement}
In this work, we used generative AI tools to implement methods, specifically to write and execute the experimental code, polish the writing, structure the code repository, and modify the scientific figures.
We have not used generative AI tools to develop theoretical models or conceptual frameworks, to design research methodology or experiments, to interpret results, to generate synthetic data, or to perform qualitative data analysis.
We have reviewed all AI-assisted work: the authors checked the experimental code against the method descriptions, verified that the reported numbers are reproduced from the logged runs, inspected every figure against its underlying data, and read all AI-edited text to confirm that it preserves the intended technical meaning.
We take responsibility for the final content of this work, including text, claims or artifacts produced with the aid of generative AI.

\bibliographystyle{iclr2027_conference}
\bibliography{references}

\appendix
% ===== appendix sections =====
\section{Claim scope and limitations}\label{app:claim-scope}
Table~\ref{tab:claims} summarizes the evidence and boundary associated with each main claim. The
core paper studies an entropy-guided selective-absorption policy added at inference time to a
uniform discrete flow.

\begin{table}[H]
\centering
\small
\caption{Evidence and scope of the main claims.}
\label{tab:claims}
\begin{tabular}{@{}p{0.255\textwidth}p{0.29\textwidth}p{0.345\textwidth}@{}}
\toprule
Claim & Supported by & Boundary\\
\midrule
Selective absorption separates the priority decision from the native flow update
& Section~\ref{sec:commit} and Appendix~\ref{app:absorbing-background}
& Absorption deliberately changes the sampler law; it is not claimed to preserve the path law of the unmodified uniform flow.\\
The local absorption error separates into joint factorization and conditional prediction terms
& Proposition~\ref{prop:kl-decomp} and Appendix~\ref{app:tc}
& The equality is local to a fixed state and selected block; errors across the complete sampling trajectory also depend on subsequent states.\\
Lowest predictive entropy minimizes the conditional-error upper bound
& Theorem~\ref{thm:order} and Appendix~\ref{app:order}
& The result assumes entropy-error regularity on oracle-reachable states; it optimizes a local surrogate and does not cover trajectories after support is lost.\\
Global entropy reduction accumulates estimation error with lookahead window
& Lemma~\ref{lem:gain-amplification} and Appendix~\ref{app:results-efficiency}
& This is an upper-bound and controlled fixed-state result; it does not imply that lookahead fails for every denoiser or state.\\
Entropy-guided absorption transfers across tasks and pretrained models
& Section~\ref{sec:theory-analysis} and Appendix~\ref{app:results-generality}
& Transfer replaces each model's default generation order through its own sampler; several differences are not significant. Sudoku runs share a checkpoint within each capacity; all rows of Table~\ref{tab:sudoku} share one matched cohort, while separate diagnostics retain their stated protocols.\\
\bottomrule
\end{tabular}
\end{table}

\subsection{Limitations and intended scope}
\paragraph{Local theoretical result.}
Theorem~\ref{thm:order} bounds the error of the current absorption decision under
entropy-error regularity. Absorbing a position changes the state encountered later, so the theorem
does not assert that a locally optimal block minimizes the final sequence loss for every model or
task. The oracle limit concerns conditional KL at a fixed oracle-reachable state. It does not guarantee recovery from earlier incorrect absorptions or eliminate joint dependence and oracle classification risk.

\paragraph{Irreversible absorption.}
Once selected, a position cannot be revised. This makes the contribution of the absorption policy
measurable and preserves the native uniform-flow update on active positions, but it can propagate an
early incorrect prediction. The method is consequently best suited to settings in which predictive
entropy is informative. Combining selective absorption with principled revision is an open direction.

\paragraph{Lookahead comparison.}
Lemma~\ref{lem:gain-amplification} identifies how errors in counterfactual predictions accumulate
when more entropy terms are aggregated. The bound does not state that Info-Gain must underperform:
lookahead can remain useful when its estimates are accurate or its action margins exceed their
error. Our fixed-state window experiment tests when the accumulated error becomes large enough to
change the selected position.

\paragraph{Baseline coverage.}
Our decode-time baselines are those with a published implementation or a fully specified algorithm
that runs on a uniform discrete flow without modification. Two related samplers fall outside that
set. EB-Sampler~\citep{benhamu2025ebsampler} controls the block size rather than the ordering, so it is complementary to \name{} rather than an alternative to it. We report a reimplementation on the puzzle tasks in Appendix~\ref{app:eb-sampler}, labelled as such, rather than a row in the cross-task tables. BoE steering~\citep{saini2026tabes}
requires a mask embedding to anchor its first-order expansion, which a uniform flow does not provide.
We therefore do not report a number for it rather than substitute a surrogate of our own
construction. Any future comparison against either method should be labelled as a reimplementation,
compute matched by function evaluations, and reported at the best setting found in a disclosed
hyperparameter sweep.

\paragraph{Empirical coverage.}
The experiments span multimodal understanding, generation, and structured reasoning, but the
largest improvements occur when outputs contain strong cross-position constraints. Broader
nonuniform flow backbones and tasks permitting revision would further test the scope of selective
absorption.

\paragraph{Training and decoding regimes.}
\name{} is a training-free sampler in both settings: no sampler, including \name{}, trains or tunes any model. The absorption indicator is maintained by the sampler and gates the velocity in Eq.~\ref{eq:absorbing-velocity}. On FUDOKI the released weights remain frozen and the indicator is not a denoiser input, so the model sees only the current token configuration. For the puzzle tasks, the task-specific denoiser is pretrained once with the standard cross-entropy objective and additionally receives the indicator as an input feature, so fixed positions are visible to it; this single checkpoint is shared by all absorbing and non-absorbing samplers (App.~\ref{app:absorbing-background}). That pretraining does not guarantee calibration after model-generated absorptions. Both non-absorbing and absorbing Sudoku baselines use coordinatewise kinetic-optimal updates, with a cosine absorption schedule for absorbing baselines. The terminal-rate proposition concerns the flow velocity, whereas the local absorption analysis concerns the selected absorption block. The ordering diagnostic and the singleton-absorption studies are separate ablations.

\paragraph{Remaining empirical attribution.}
The current experiments do not isolate selective absorption against terminal-rate clipping, early stopping, or suppressed target resampling. The per-step profile in Appendix~\ref{app:results-efficiency} attributes 3.6\% of a step to the ordering score, so reported timing differences among confidence scores reflect implementation rather than the score itself. That profile is relative and measured on CPU. The block-size ablation varies the number of positions absorbed per step and does not cover every absorption schedule.
          % A: claim scope and limitations
\section{Details for Section~\ref{sec:prelim}: Problem Formulation}\label{app:section2}
%==============================================================================
This appendix expands the two subsections of Section~\ref{sec:prelim}: Appendix~\ref{app:dfm-background} gives the discrete-flow background behind Section~\ref{sec:dfm}, and Appendix~\ref{app:retention} proves the terminal-exposure proposition and documents the Sudoku diagnostic of Section~\ref{sec:retention}.

\subsection{Preliminaries: discrete flow matching background}\label{app:dfm-background}
\begin{definition}[Conditional CTMC]\label{def:ctmc}
Given a rate matrix $R_{t}(y,z)_{y,z\in\Sset^{M}}$ where $R_{t}(y,z)\geq 0$ for $z\neq y$ and $\sum_{z}R_{t}(y,z)=0$, the process $\{a_{t}\}$ is a CTMC with rate $R_{t}$, if for $h>0$, $\mathbb{P}(a_{t+h}=z|a_{t}=y)=\delta_{y}(z)+R_{t}(y,z)h+o(h)$ and its marginal density $p_{t}$ at time $t$ obeys the Kolmogorov forward equation $\dot{p}_{t}(y)=\sum_{z\neq y}(p_{t}(z)R_{t}(z,y)-p_{t}(y)R_{t}(y,z))$.
\end{definition}

We expand the probability-path and velocity formulation summarized in
\Cref{sec:dfm}, following \citet{gat2024discrete}, and then give the uniform metric-induced
specialization used by FUDOKI~\citep{wang2025fudoki}.

\paragraph{Conditional and marginal probability paths.}
Let $\pi(a_0,a_1\mid\mathcal C)$ be a coupling whose marginals are the source
$p_0(a_0\mid\mathcal C)$ and target $q(a_1\mid\mathcal C)$. Discrete flow matching specifies a
conditional path between each coupled pair and marginalizes the pair:
\begin{align}
p_t(a\mid\mathcal C)
&=\sum_{a_0,a_1}p_t(a\mid a_0,a_1,\mathcal C)\,
  \pi(a_0,a_1\mid\mathcal C),\label{eq:dfm-marginal-path}\\
p_t(a\mid a_0,a_1,\mathcal C)
&=\prod_{j=1}^{M}p_t^j(a^j\mid a_0,a_1,\mathcal C).\label{eq:dfm-conditional-path}
\end{align}
The boundary conditions
$p_0^j(\cdot\mid a_0,a_1)=\delta_{a_0^j}$ and
$p_1^j(\cdot\mid a_0,a_1)=\delta_{a_1^j}$ ensure that the marginal path begins at $p_0$ and ends
at $q$. When the source is fixed or integrated into the coordinate path, this reduces to the
target-conditioned notation $p_{t\mid1}(a\mid a_1)$ used in the main text.

\paragraph{Probability velocities and continuity.}
For a current state $a$, the coordinate velocity $u_t^j(v,a)$ defines the infinitesimal update
\begin{equation}\label{eq:dfm-local-update}
\mathbb P(A_{t+h}^j=v\mid A_t=a)
=\delta_{a^j}(v)+h\,u_t^j(v,a)+o(h).
\end{equation}
It is a valid CTMC rate when
\begin{equation}\label{eq:dfm-rate-condition}
\sum_{v\in\mathcal S}u_t^j(v,a)=0,
\qquad
u_t^j(v,a)\geq0\quad\text{for }v\neq a^j.
\end{equation}
The full velocity permits one-coordinate transitions,
$u_t(z,a)=\sum_j\delta(z^{-j},a^{-j})u_t^j(z^j,a)$, with the first argument denoting the
destination. Hence the path evolves according to the Kolmogorov forward equation
\begin{equation}\label{eq:dfm-forward}
\dot p_t(a\mid\mathcal C)
=\sum_{z\in\mathcal S^M}p_t(z\mid\mathcal C)u_t(a,z).
\end{equation}
Equivalently, define the probability flux from $z$ to $a$ by
$J_t(a,z)=p_t(z)u_t(a,z)$. Then
$\dot p_t(a)+\operatorname{div}_a J_t=0$, where
\begin{equation}\label{eq:dfm-divergence}
\operatorname{div}_a J_t
=\sum_{z\neq a}\bigl[J_t(z,a)-J_t(a,z)\bigr].
\end{equation}
Thus a velocity generates $p_t$ precisely when its flux satisfies this continuity equation and the
rate conditions above.

\paragraph{Conditional velocities and posterior marginalization.}
Suppose $u_t^j(v,a^j\mid a_0,a_1,\mathcal C)$ generates the coordinate path in
\eqref{eq:dfm-conditional-path}. The corresponding marginal velocity is obtained by averaging
over the target pair conditioned on the current state:
\begin{align}
u_t^j(v,a,\mathcal C)
&=\mathbb E_{(a_0,a_1)\sim p_t(\cdot,\cdot\mid a,\mathcal C)}
  \!\left[u_t^j(v,a^j\mid a_0,a_1,\mathcal C)\right],\label{eq:dfm-posterior-average}\\
p_t(a_0,a_1\mid a,\mathcal C)
&=\frac{p_t(a\mid a_0,a_1,\mathcal C)\pi(a_0,a_1\mid\mathcal C)}
        {p_t(a\mid\mathcal C)}.\label{eq:dfm-pair-posterior}
\end{align}
For a target-conditioned path, only $p_{1\mid t}(a_1\mid a,\mathcal C)$ remains unknown.
It is approximated by the denoiser $\mathbb Q_{1\mid t}^{\theta}$, trained with the coordinate
cross-entropy objective
\begin{equation}\label{eq:dfm-denoiser-objective}
\mathcal L_{\mathrm{den}}(\theta)
=-\mathbb E_{t,a_0,a_1,a_t}
\left[\sum_{j=1}^{M}\log
\mathbb Q_{1\mid t}^{\theta}(a_1^j\mid a_t,\mathcal C_t)\right].
\end{equation}
Substituting this posterior into~\eqref{eq:dfm-posterior-average} yields the learned velocity used
at generation time.

\paragraph{Uniform metric-induced flow in FUDOKI.}
FUDOKI uses a distance $d:\mathcal S\times\mathcal S\to\mathbb R_{\geq0}$ satisfying
$d(v,y)=0$ if and only if $v=y$, and defines
\begin{equation}\label{eq:fudoki-metric-path}
p_{t\mid1}^j(v\mid a_1^j)
=\frac{\exp[-\beta_t d(v,a_1^j)]}
       {\sum_{w\in\mathcal S}\exp[-\beta_t d(w,a_1^j)]},
\qquad \beta_0=0,\quad \beta_t\xrightarrow[t\to1]{}\infty.
\end{equation}
The initial coordinate distribution is therefore uniform over $\mathcal S$, while the target
concentrates on $a_1^j$. To select a velocity for this prescribed path, write the conditional flux
from $r$ to $v$ as $J_t^j(v,r\mid a_1^j)$ and minimize its kinetic energy subject to continuity,
nonnegative off-diagonal flux, and the path boundaries:
\begin{equation}\label{eq:fudoki-kinetic-objective}
\min_{J_t^j}\int_0^1\sum_{v\neq r}
\omega_t(v,r)\frac{J_t^j(v,r\mid a_1^j)^2}
{p_{t\mid1}^j(r\mid a_1^j)}\,\mathrm dt,
\qquad \omega_t(v,r)>0.
\end{equation}
For the weighting $\omega_t(v,r)=1/p_{t\mid1}^j(v\mid a_1^j)$, the optimal flux is
\begin{equation}\label{eq:fudoki-optimal-flux}
J_t^{j,\star}(v,r\mid a_1^j)
=\left[p_{t\mid1}^j(r\mid a_1^j)\,\dot p_{t\mid1}^j(v\mid a_1^j)
-\dot p_{t\mid1}^j(r\mid a_1^j)\,p_{t\mid1}^j(v\mid a_1^j)\right]_+.
\end{equation}
Dividing by the probability at the current token $r=a^j$ gives the off-diagonal conditional
velocity
\begin{equation}\label{eq:fudoki-conditional-velocity}
u_t^j(v,a^j\mid a_1^j)
=p_{t\mid1}^j(v\mid a_1^j)\dot\beta_t
\left[d(a^j,a_1^j)-d(v,a_1^j)\right]_+,
\qquad v\neq a^j,
\end{equation}
with the diagonal fixed by~\eqref{eq:dfm-rate-condition}. Probability mass therefore moves only
to tokens closer to the target under $d$. Finally, FUDOKI replaces the target posterior in
\eqref{eq:dfm-posterior-average} by $\mathbb Q_{1\mid t}^{\theta}$:
\begin{equation}\label{eq:fudoki-learned-velocity}
u_t^{\theta,j}(v,a,\mathcal C_t)
=\mathbb E_{a_1\sim\mathbb Q_{1\mid t}^{\theta}(\cdot\mid a,\mathcal C_t)}
\!\left[u_t^j(v,a^j\mid a_1^j)\right].
\end{equation}
Applying~\eqref{eq:dfm-local-update} with this velocity gives the native FUDOKI sampler. Because
the uniform path has full vocabulary support and introduces no terminal state, a coordinate may
change repeatedly before $t=1$. The selective absorbing state in \Cref{sec:commit} is an additional
inference mechanism rather than part of this standard flow construction.

\subsection{Motivation and ordered absorption: repeated transitions and the absorption decision}\label{app:retention}
This appendix formalizes the transition behavior summarized in Section~\ref{sec:retention} and
derives a decision threshold for absorption.

\paragraph{Setting.}
We work directly with the kinetic-optimal velocity of Eq.~\ref{eq:kinetic-velocity}. Conditioned on
the true target $a_1^j$ it reads
\begin{equation}
u_t^j(z^j,a_t\mid a_1^j)
=
p_t(z^j\mid a_1^j)\,\dot\beta_t
\left[
d(a_t^j,a_1^j)-d(z^j,a_1^j)
\right]_+,
\qquad
p_t(z^j\mid a_1^j)=\frac{e^{-\beta_t d(z^j,a_1^j)}}{Z_t(a_1^j)},
\label{eq:ret-kin}
\end{equation}
with $Z_t(a_1^j)=\sum_{z\in\mathcal S}e^{-\beta_t d(z,a_1^j)}$. At inference the solver substitutes
the predicted target $\hat a_1^j\sim\mathbb Q^\theta_{1\mid t}(\cdot\mid a_t,\mathcal C_t)$ for
$a_1^j$ throughout. We assume only that $d$ is a metric on $\mathcal S$, so that $d(z,z')>0$ whenever
$z\neq z'$. No further structure is needed, and in particular the argument does not require $d$ to be
the discrete metric. This matters because FUDOKI computes $d$ from pretrained token embeddings rather
than from token identity. We also use $\beta_t=c(t/(1-t))^{\alpha}$ with $c=3,\alpha=0.9$, whence
\begin{equation}
\dot\beta_t
=
c\,\alpha\,t^{\alpha-1}(1-t)^{-\alpha-1}
\xrightarrow[t\to1]{}\infty .
\label{eq:ret-betadot}
\end{equation}
Throughout, $a_t^j$ is the token currently held at position $j$ and $a_1^j$ is the true target there.

\begin{proposition}[Target-conditioned monotonicity]\label{prop:retention}
Let
\begin{equation}
\lambda_t^j
:=
\sum_{z^j\neq a_t^j}u_t^j(z^j,a_t\mid a_1^j),
\qquad
\hat\lambda_t^j
:=
\sum_{z^j\neq a_t^j}u_t^j(z^j,a_t\mid\hat a_1^j)
\label{eq:total-rate}
\end{equation}
be the total transition rates obtained by conditioning the kinetic-optimal velocity on the oracle and
predicted targets, respectively. If $a_t^j=a_1^j$, then $\lambda_t^j=0$. For
$\hat a_1^j\sim\mathbb Q^\theta_{1\mid t}(\cdot\mid a_t,\mathcal C_t)$, the expected rate at which a
correct token is replaced satisfies
\begin{equation}
\mathbb E[\hat\lambda_t^j]
\geq
\dot\beta_t\,
\mathbb E\!\left[
\mathbf 1\{\hat a_1^j\neq a_1^j\}
p_t(\hat a_1^j\mid\hat a_1^j)
d(a_1^j,\hat a_1^j)
\right]
\geq
\dot\beta_t\,\varepsilon_t^j p_{\min,t}d_{\min},
\label{eq:replacement-rate}
\end{equation}
where
$\varepsilon_t^j=1-\mathbb Q^\theta_{1\mid t}(a_1^j\mid a_t,\mathcal C_t)$,
$p_{\min,t}=\min_{z\in\mathcal S}p_t(z\mid z)$, and
$d_{\min}=\min_{z\neq a_1^j}d(z,a_1^j)$. Hence the lower bound is positive whenever all three
quantities are positive.
\end{proposition}

\paragraph{Part (i): the true target leaves a correct token fixed.}
Suppose the velocity is conditioned on the true target and the position already holds it, so that
$a_t^j=a_1^j$. For any $z^j\neq a_t^j$,
\[
u_t^j(z^j,a_t\mid a_1^j)
=
p_t(z^j\mid a_1^j)\,\dot\beta_t
\left[d(a_1^j,a_1^j)-d(z^j,a_1^j)\right]_+
=
p_t(z^j\mid a_1^j)\,\dot\beta_t
\left[-d(z^j,a_1^j)\right]_+
=0,
\]
because $d(a_1^j,a_1^j)=0$ and $d(z^j,a_1^j)>0$, so the bracket is negative and the ReLU annihilates
it. Summing over destinations gives $\lambda_t^j=0$: conditioned on the true target, a position that
has reached it never leaves. Retention at the true target is a property of the exact velocity itself,
not an addition to it.

\paragraph{Toward part (ii): a single transition under a wrong predicted target.}
At inference the velocity is conditioned on $\hat a_1^j$ rather than $a_1^j$. Suppose the position is
correct, $a_t^j=a_1^j$, but the prediction misses, $\hat a_1^j\neq a_1^j$. Taking the destination
$z^j=\hat a_1^j$,
\[
u_t^j(\hat a_1^j,a_t\mid\hat a_1^j)
=
p_t(\hat a_1^j\mid\hat a_1^j)\,\dot\beta_t
\left[d(a_1^j,\hat a_1^j)-d(\hat a_1^j,\hat a_1^j)\right]_+
=
p_t(\hat a_1^j\mid\hat a_1^j)\,\dot\beta_t\,d(a_1^j,\hat a_1^j)
>0 ,
\]
since $p_t(\hat a_1^j\mid\hat a_1^j)=1/Z_t(\hat a_1^j)>0$ and $d(a_1^j,\hat a_1^j)>0$. The transition
is not merely permitted but is the one the velocity most favours, because $\hat a_1^j$ is the closest
state to the value the velocity is conditioned on. The monotonicity of Eq.~\ref{eq:ret-kin} is thus
relative to $\hat a_1^j$ and carries no implication for $d(\cdot,a_1^j)$: a step that moves closer to
the prediction can move away from the truth.

\paragraph{Part (ii): the expected total conditional transition rate.}
Every destination contributes a nonnegative rate, so retaining only the summand $z^j=\hat a_1^j$
computed above lower bounds $\hat\lambda_t^j$. The term vanishes when $\hat a_1^j=a_1^j$ by part
(i), so taking expectations over the predicted target gives the first inequality of
Eq.~\ref{eq:replacement-rate},
\[
\mathbb E\!\left[\hat\lambda_t^j\right]
\geq
\dot\beta_t\,
\mathbb E\!\left[
\mathbf 1\{\hat a_1^j\neq a_1^j\}\,
p_t(\hat a_1^j\mid\hat a_1^j)\,
d(a_1^j,\hat a_1^j)
\right].
\]
Bounding $p_t(\hat a_1^j\mid\hat a_1^j)\geq p_{\min,t}:=\min_{z\in\mathcal S}1/Z_t(z)>0$ and
$d(a_1^j,\hat a_1^j)\geq d_{\min}:=\min_{z\neq a_1^j}d(z,a_1^j)>0$ on the event
$\{\hat a_1^j\neq a_1^j\}$, whose probability is
$\varepsilon_t^j=1-\mathbb Q^\theta_{1\mid t}(a_1^j\mid a_t,\mathcal C_t)$, yields the second
inequality,
$\mathbb E[\hat\lambda_t^j]\geq\dot\beta_t\,\varepsilon_t^j\,p_{\min,t}\,d_{\min}$.
A denoiser placing all of its mass on the true target has $\varepsilon_t^j=0$, recovering part (i).
Any denoiser with residual target prediction error has $\mathbb E[\hat\lambda_t^j]>0$.

\paragraph{Proof of Proposition~\ref{cor:terminal-blowup} (terminal blow-up of the transition rate).}
Fix a position $j$ with $a_t^j\neq a_1^j$. The summand in the total rate of Eq.~\ref{eq:total-rate}
vanishes at $z^j=a_t^j$ (its bracket is 0), so, factoring the schedule gradient out of
Eq.~\ref{eq:ret-kin},
\[
\lambda_t^j
=\sum_{z^j\neq a_t^j}u_t^j(z^j,a_t\mid a_1^j)
=\dot\beta_t
\underbrace{\sum_{z^j\in\mathcal S}p_t(z^j\mid a_1^j)\big[d(a_t^j,a_1^j)-d(z^j,a_1^j)\big]_+}_{=:~G_t^j}.
\]
Retaining only the destination $z^j=a_1^j$, whose bracket equals $d(a_t^j,a_1^j)-d(a_1^j,a_1^j)=d(a_t^j,a_1^j)$,
and discarding the remaining nonnegative terms gives
\[
G_t^j\ \ge\ p_t(a_1^j\mid a_1^j)\,d(a_t^j,a_1^j)\ >\ 0,
\]
where strict positivity holds because $d$ is a metric and $a_t^j\neq a_1^j$. By Eq.~\ref{eq:ret-kin},
\[
p_t(a_1^j\mid a_1^j)=\frac{1}{Z_t(a_1^j)}
=\Big(1+\textstyle\sum_{z\neq a_1^j}e^{-\beta_t d(z,a_1^j)}\Big)^{-1}\xrightarrow[t\to1]{}1,
\]
since $\beta_t$ grows without bound and $d(z,a_1^j)>0$ for every $z\neq a_1^j$. Combining with
the divergence of $\dot\beta_t$ from Eq.~\ref{eq:ret-betadot} gives that $\lambda_t^j=\dot\beta_t\,G_t^j$ grows without bound as
$t$ approaches 1. For a frozen rate, the jump probability is $1-e^{-h\lambda_t^j}$, which tends to 1 for fixed $h>0$ as $\lambda_t^j$ grows. This is a frozen-rate calculation, not an exact endpoint limit for the time-varying CTMC on $[0,1]$. Substituting the freshly sampled target
$\hat a_1^j$ for $a_1^j$ throughout, as the solver does at inference, gives the stated terminal
re-derivation. The corresponding expected replacement rate is bounded below in
Eq.~\ref{eq:replacement-rate}. \hfill$\square$

\paragraph{Scope of Proposition~\ref{cor:terminal-blowup}.}
The proposition describes transition exposure, which can produce either correction or incorrect revision. It does not establish a net accuracy loss, nor does it cover every uniform-flow path or solver.

\begin{corollary}[Continued exposure to transition]\label{cor:no-early-stop}
The realized rate vanishes exactly when the predicted target equals the token currently held:
\begin{equation}
\Pr\!\left[\hat\lambda_t^j=0\mid a_t,\mathcal C_t\right]
=
\mathbb Q^\theta_{1\mid t}(a_1^j=a_t^j\mid a_t,\mathcal C_t).
\label{eq:exposure-prob}
\end{equation}
If the position currently holds the correct token, its probability of remaining exposed to a
transition at grid point $t_k$ is $\varepsilon_{t_k}^j$, giving
$\sum_k\varepsilon_{t_k}^j$ expected exposures over the sampling grid.
\end{corollary}

\paragraph{Proof of Corollary~\ref{cor:no-early-stop}.}
The computation in part (i) used only that the value conditioning the velocity equals the token held
at the position. It did not use that this value was the true target. Applying it at the sampled
target gives, for any draw $\hat a_1^j=a_t^j$, that $u_t^j(z^j,a_t\mid\hat a_1^j)=0$ for every
$z^j\neq a_t^j$ and hence $\hat\lambda_t^j=0$. Conversely, if $\hat a_1^j\neq a_t^j$ then the
destination $z^j=\hat a_1^j$ contributes
$p_t(\hat a_1^j\mid\hat a_1^j)\dot\beta_t\,d(a_t^j,\hat a_1^j)>0$, so $\hat\lambda_t^j>0$. Therefore
$\{\hat\lambda_t^j=0\}=\{\hat a_1^j=a_t^j\}$ exactly, and taking probabilities under the draw
$\hat a_1^j\sim\mathbb Q^\theta_{1\mid t}(\cdot\mid a_t,\mathcal C_t)$ gives
Eq.~\ref{eq:exposure-prob}. When the position holds the correct token, $a_t^j=a_1^j$, the complement
has probability $\varepsilon_t^j$, and summing over the grid gives the stated expected exposure
count $\sum_k\varepsilon_{t_k}^j$. Since the source is uniform at $t=0$, the match probability begins
at $1/|\mathcal S|$ and rises toward one only as the denoiser concentrates, so the count is close to
the full number of sampling steps whenever residual uncertainty persists. \hfill$\square$

\paragraph{Fixed points of the marginalized sampler.}
Because the predicted target is redrawn each step, the object governing whether a position settles is
the rate marginalized over that draw. For an incumbent token $a_t^j$,
\[
\mathbb E\!\left[\hat\lambda_t^j\right]
=
\sum_{b\in\mathcal S}
\mathbb Q^\theta_{1\mid t}(a_1^j=b\mid a_t,\mathcal C_t)
\sum_{z^j\neq a_t^j}
p_t(z^j\mid b)\,\dot\beta_t
\left[d(a_t^j,b)-d(z^j,b)\right]_+ .
\]
Every summand is nonnegative, and for $t\in(0,1)$ both $\dot\beta_t>0$ and $p_t(z^j\mid b)>0$. The
expectation therefore vanishes if and only if $[d(a_t^j,b)-d(z^j,b)]_+=0$ for every $z^j\neq a_t^j$
and every $b$ in the support of the denoiser, that is, whenever $d(z^j,b)\geq d(a_t^j,b)$ for all
$z^j\neq a_t^j$. Choosing $z^j=b$ gives $0=d(b,b)\geq d(a_t^j,b)$, which forces $b=a_t^j$ since $d$ is
a metric. Hence
\begin{equation}
\mathbb E\!\left[\hat\lambda_t^j\right]=0
\quad\Longleftrightarrow\quad
\mathbb Q^\theta_{1\mid t}(\cdot\mid a_t,\mathcal C_t)=\delta_{a_t^j} ,
\label{eq:fixed-point}
\end{equation}
so the only fixed points of the marginalized sampler are states at which the denoiser is a point mass
on the token already held. This is a statement about certainty, not correctness: a confidently
incorrect denoiser produces a fixed point at a wrong token, while an uncertain denoiser produces
none, and in neither case is the time at which the sampler settles under the sampler's control.
The selective-absorption construction below creates a fixed point independently of the denoiser's confidence, at a time the
policy selects.

\paragraph{Absorption sets the rate to zero.}
Under Eq.~\ref{eq:absorbing-velocity}, $u_t^{\theta,\mathrm{abs},j}=(1-f_t^j)u_t^{\theta,j}$, so
a position absorbed at time $\tau_j$ has $f_s^j=1$ and hence zero outflow for all $s\geq\tau_j$.
Therefore $a_s^j=a_{\tau_j}^j$ and $\hat\lambda_s^j=0$ on $[\tau_j,1]$, and the cumulative exposure
$\int_{\tau_j}^1\mathbb E[\hat\lambda_s^j]\,\mathrm ds$ is truncated at $\tau_j$ rather than
accumulated to the end of sampling. \hfill$\square$

\paragraph{On late-time concentration.}
Eq.~\ref{eq:ret-betadot} shows that $\dot\beta_t$ grows without bound, and
$p_t(\hat a_1^j\mid\hat a_1^j)=1/Z_t(\hat a_1^j)$ tends to 1 as $\beta_t$ grows, so the rate in part (ii)
behaves as $\dot\beta_t\,d(a_1^j,\hat a_1^j)$ for large $t$: a late target prediction error acts
through an increasingly concentrated transport. We deliberately do not claim that the resulting error
probability diverges. The finite-step jump probability is bounded by one, and $\varepsilon_t^j$ is
expected to decrease as the denoiser sees a more resolved context, so the two factors act in opposite
directions. The defensible statement is that the consequence of residual prediction error is governed
by the product $\varepsilon_t^j\dot\beta_t$ rather than by $\varepsilon_t^j$ alone, which makes that product the natural diagnostic to measure.

\paragraph{Measured replacement rate.} Figure~\ref{fig:rate-diagnostic} measures this product on
the Sudoku denoiser under the native non-absorbing flow (Nikoli, $n{=}100$, seeds 0 to 4). Per grid
step we log the incumbent error $\varepsilon_t=1-\mathbb Q^\theta_{1\mid t}(a_t^j\mid a_t)$ averaged
over generated positions, the per-step hazard $c_k$ of the schedule, and their product, which is the
rate at which a currently held token is replaced. The hazard alone rises to one at the endpoint on
every grid, but the incumbent error does not vanish: at $t{=}1$ it is 0.24, 0.39, 0.51, 0.59,
and 0.64 for $K\in\{16,32,64,128,256\}$, so the terminal replacement rate equals those values.
The rate therefore grows as the grid is refined, from 0.24 at $K{=}16$ to 0.64 at $K{=}256$,
which is the opposite of what a discretization artifact would produce: refining the integration
brings the sampler closer to the continuous flow and makes the late churn worse, not better. This is
the empirical content of Proposition~\ref{cor:terminal-blowup}: the schedule gradient sets the
opportunity for replacement, and the residual denoiser error decides whether it is taken.

\begin{figure}[t]\centering
\includegraphics[width=0.9\columnwidth]{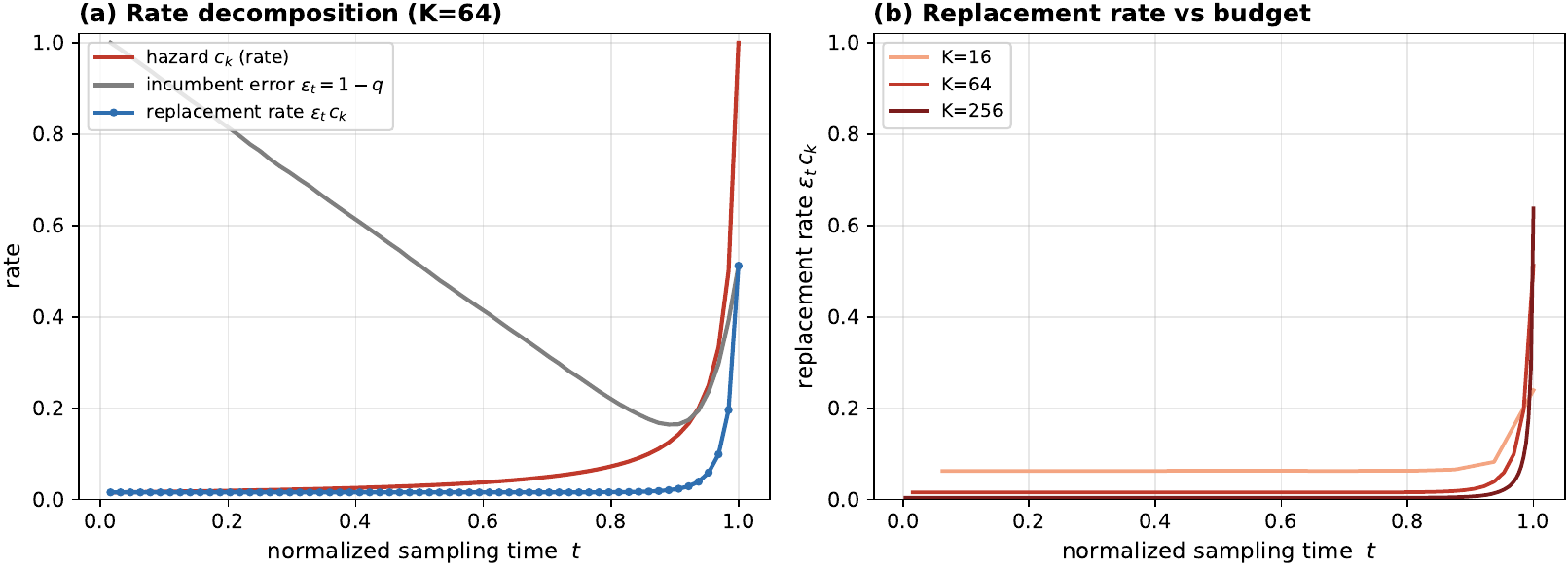}
\caption{\textbf{The replacement rate is the product of residual error and hazard.} (a) At $K{=}64$
the per-step hazard $c_k$ (red) rises to one as $t$ approaches 1 while the incumbent error $\varepsilon_t$
(grey) falls and then rises again in the last steps. Their product (blue), the correct-token
replacement rate, is small for most of sampling and spikes at the endpoint. (b) The product across
grids $K\in\{16,64,256\}$: the terminal spike grows with $K$, so refining the discretization
increases the late churn.}
\label{fig:rate-diagnostic}
\end{figure}

\paragraph{The absorption threshold.}
Fix a position $j$ and let $\rho=\Pr[a_t^j=a_1^j]$ be the probability that it currently holds the
correct token. Over the remaining steps, let $r$ be the probability that an incorrect token is
corrected and $w$ the probability that a correct token is replaced, the latter generated by the rate
of part (ii). Absorbing at the current state fixes the token, so its expected terminal error is
$\mathbb E_{\mathrm{abs}}=1-\rho$. Leaving the position active gives
\[
\mathbb E_{\mathrm{active}}
=
(1-\rho)(1-r)+\rho\,w ,
\]
the first term counting incorrect tokens that are never corrected and the second correct tokens that
are replaced. Continued refinement is preferable exactly when
$\mathbb E_{\mathrm{active}}<\mathbb E_{\mathrm{abs}}$, i.e.
\begin{equation}\label{eq:absorb-threshold}
(1-\rho)(1-r)+\rho w<1-\rho
\iff
\rho w<(1-\rho)r
\iff
  \rho<\frac{r}{r+w}.
\end{equation}
The threshold is increasing in $r$ and decreasing in $w$, as
expected: a sampler that corrects reliably should refine longer, and one whose replacement rate is
high should absorb sooner. Both $r$ and $w$ are properties of the denoiser and the schedule and are
not available at decode time, but the comparison depends on the current position only through $\rho$.
This is what makes a reliability estimate the operative quantity.

              % B: details for Section 2 (B.1 preliminaries, B.2 motivation)
\section{Details for Section~\ref{sec:sampler}: \name{}}\label{app:section3}
%==============================================================================
This appendix follows the three subsections of Section~\ref{sec:sampler}: Appendix~\ref{app:absorbing-background} details the absorbing-state augmentation of Section~\ref{sec:commit}, Appendix~\ref{app:ledflow-details} documents the sampler of Section~\ref{sec:ledflow-method}, and Appendix~\ref{app:order} collects the proofs, the calibration diagnostic, and the lookahead analysis behind Section~\ref{sec:ledflow-theory}.

\subsection{Selective absorption as an absorbing-state uniform flow}
\label{app:absorbing-background}
This appendix expands the construction in Section~\ref{sec:commit}. The uniform discrete flow has
state $a_t$ and learned coordinate velocity $u_t^{\theta,j}$. We augment it with
$f_t=(f_t^1,\ldots,f_t^M)\in\{0,1\}^M$, where $f_t^j=0$ marks an active position and $f_t^j=1$
marks a terminal position. Initially, $f_0^j=0$ for every generated position.

\paragraph{Absorbing velocity.}
For a fixed augmented state $(a_t,f_t)$, define
\[
u_t^{\theta,\mathrm{abs},j}(z^j,a_t,f_t,\mathcal C_t)
=(1-f_t^j)u_t^{\theta,j}(z^j,a_t,\mathcal C_t).
\]
When $f_t^j=0$, this is exactly the native uniform-flow velocity. When $f_t^j=1$, every outgoing
rate is zero, so the coordinate remains at its absorbed value. Because $1-f_t^j$ is nonnegative,
this modification preserves nonnegative off-diagonal rates and the zero row sum of the native CTMC
velocity. The augmented process is therefore a valid CTMC between absorption decisions.

\paragraph{Irreversible update.}
At grid point $t_k$, a policy chooses an absorption block $B_k$ from the active set
$\mathcal U_k=\{j:f_{t_k}^j=0\}$. For each $j\in B_k$, the update assigns the most likely target
value under $\mathbb Q_{1\mid t_k}^{\theta}$ and changes $f_{t_k}^j$ from zero to one. Indicators
never return to zero, hence
\[
\mathcal U_{k+1}\subseteq\mathcal U_k,
\qquad
f_{t_{k+1}}^j\geq f_{t_k}^j.
\]
The active positions subsequently follow the original uniform-flow update, while absorbed
positions are copied unchanged. The schedule in Section~\ref{sec:commit} fixes the number of active
positions remaining after each grid point. The priority policy determines which positions form
$B_k$.

\paragraph{Relation to the exact absorbing representation.}
Appendix~\ref{app:uniform-absorption-derivation} represents uniform transport as a mixture of
conditional processes indexed by auxiliary absorbing tokens. That construction is an exact
representation of the underlying uniform path. Our indicator $f_t^j$ instead records an inference
decision: after a model prediction is selected, Eq.~\eqref{eq:absorbing-velocity} deliberately removes its
future native-flow transitions. The selective augmentation therefore borrows the absorbing-state
mechanism but does not claim to preserve the path law of the unmodified uniform flow. This is the
decision whose local prediction error is analyzed in Section~\ref{sec:limit}.

\paragraph{Pretraining the task-specific puzzle denoiser.}
For the puzzle tasks, we pretrain a task-specific uniform-flow denoiser $\theta$ once, with the standard denoising cross-entropy objective and uniform-noise corruption. This checkpoint is fixed before any sampler is evaluated and is shared by every sampler in Table~\ref{tab:sudoku}. No sampler, including \name{}, is trained or tuned on it. Its input includes a per-position binary status indicator $f^j_t\in\{0,1\}$, embedded and added to the token embedding as an additional input feature, so the network can tell fixed positions from tentative ones. At each training step we draw a time $t\sim\mathcal U(0,1)$ and construct the corrupted input as follows. Conditioning positions and a fraction $t$ of the generated positions (drawn independently per position, matching the linear schedule $\kappa_t=t$) are marked \emph{absorbed} ($f^j_t=1$) and hold their ground-truth value. The remaining \emph{active} positions ($f^j_t=0$) carry a uniform-flow state that interpolates between the target and noise: each keeps its true value with probability $t$ and is otherwise replaced by a uniformly random token, so the active tokens are a partially informative estimate that sharpens as $t$ approaches 1 rather than pure noise or a single mask symbol. The denoiser predicts the clean value at every active position, trained with the denoising cross-entropy $\mathcal L_{\mathrm{DLM}}$ restricted to the active positions. Absorbed positions are conditioned on but not scored. At inference the model receives the active/absorbed configuration, fixes selected predictions, and updates the remaining positions. Training absorptions contain ground-truth values, whereas inference absorptions can be wrong and are selected adaptively. The indicator makes absorption status observable but does not eliminate this distribution shift or establish equivalence to the metric-induced path. On multimodal understanding, text-to-image generation, and mathematics we use the released FUDOKI weights: absorption in Eq.~\ref{eq:absorbing-velocity} is a sampler-side operation that only zeroes the outgoing rate of the selected positions, so the multimodal results of Section~\ref{sec:exp} are obtained by applying the same entropy-guided ordered absorption at inference to the pretrained FUDOKI uniform-flow model, leaving its weights unchanged. In both settings \name{} is a training-free sampler. The only difference is whether the underlying model is a released backbone or a task-specific puzzle denoiser pretrained once for all samplers.

\paragraph{Checkpoint control.}
As a control on the model rather than the sampler, we pretrain a second uniform-flow denoiser without the status indicator and evaluate it under the plain uniform-flow sampler with no absorption, on the same 100 Nikoli puzzles at $K{=}64$ as Fig.~\ref{fig:overjumpy}. It attains 36.3\% cell accuracy and solves $0/100$ puzzles, whereas the main checkpoint under the same non-absorbing sampler solves 0.410 (Fig.~\ref{fig:overjumpy}). The status indicator is therefore informative even when no position is ever absorbed, and the non-absorbing baselines in Table~\ref{tab:sudoku} run on the stronger of the two checkpoints. All reported sampler comparisons, including every result in Fig.~\ref{fig:overjumpy}, use the main checkpoint and differ only in the sampler.

\paragraph{Frozen backbones.}
On frozen FUDOKI the indicator $f_t$ is maintained by the sampler and is not a denoiser input. The model sees only the current token configuration. Fixing tokens can therefore change the states visited by a frozen backbone, without a guarantee that those states follow its training distribution. The puzzle denoiser, by contrast, takes the status indicator as an input feature, so fixed positions are visible to it under every sampler.

\subsubsection{Deriving the exact absorbing representation of uniform flow}
\label{app:uniform-absorption-derivation}
This appendix derives the representation used in Section~\ref{sec:uniform-absorption}. We first
follow the data-to-noise convention of
\citet{gourevitch2026uniformdiffusionmodelsrevisited}, then reverse time to recover the
source-to-target convention of discrete flow matching. Tokens are identified with their one-hot
vectors whenever they parameterize a categorical distribution.

\paragraph{Step 1: expand the conditional absorbing transition.}
Fix a coordinate $j$, a current token $a^j$, and an auxiliary absorbing token $u^j$. The
conditional transition over $[t,t+h]$ is
\begin{align}
q_{t+h\mid t}(z^j\mid a^j,u^j)
&=\operatorname{Cat}\!\left(
z^j;\alpha_{t+h\mid t}a^j+(1-\alpha_{t+h\mid t})u^j
\right)\nonumber\\
&=\alpha_{t+h\mid t}\delta(z^j,a^j)
+(1-\alpha_{t+h\mid t})\delta(z^j,u^j).
\label{eq:app-absorbing-kernel}
\end{align}
Writing $\alpha_{t+h\mid t}=\alpha_{t+h}/\alpha_t$ and assuming differentiability gives
\begin{equation}
\alpha_{t+h\mid t}
=1+\frac{\dot\alpha_t}{\alpha_t}h+o(h)
=1-\nu_t h+o(h),
\qquad
\nu_t:=-\frac{\dot\alpha_t}{\alpha_t}\geq0,
\label{eq:app-hazard}
\end{equation}
because the corruption schedule $\alpha_t$ is nonincreasing. Substituting
\eqref{eq:app-hazard} into~\eqref{eq:app-absorbing-kernel} yields
\begin{equation}
q_{t+h\mid t}(z^j\mid a^j,u^j)
=\delta(z^j,a^j)
+h\nu_t\!\left[\delta(z^j,u^j)-\delta(z^j,a^j)\right]+o(h).
\label{eq:app-absorbing-expansion}
\end{equation}

\paragraph{Step 2: identify the conditional CTMC velocity.}
Comparing~\eqref{eq:app-absorbing-expansion} with the infinitesimal CTMC update in
\eqref{eq:dfm-local-update} gives
\begin{equation}
u_t^{j,u}(z^j,a^j)
=\nu_t\!\left[\delta(z^j,u^j)-\delta(z^j,a^j)\right].
\label{eq:app-absorbing-velocity}
\end{equation}
For $z^j\neq a^j$, this rate equals $\nu_t\delta(z^j,u^j)\geq0$, and summing over $z^j$
gives zero. It is therefore a valid CTMC velocity. If $a^j=u^j$, every rate vanishes. Otherwise,
the only possible jump is from $a^j$ to $u^j$, after which the coordinate remains at $u^j$. This verifies
the absorbing property directly.

\paragraph{Step 3: marginalize the absorbing token.}
Let $U^j\sim\operatorname{Unif}(\mathcal S)$. Since
$\mathbb E[\delta(z^j,U^j)]=1/|\mathcal S|$, averaging either
\eqref{eq:app-absorbing-kernel} or~\eqref{eq:app-absorbing-velocity} gives
\begin{align}
\mathbb E_{U^j}\!\left[q_{t+h\mid t}(z^j\mid a^j,U^j)\right]
&=\alpha_{t+h\mid t}\delta(z^j,a^j)
+\frac{1-\alpha_{t+h\mid t}}{|\mathcal S|},
\label{eq:app-uniform-kernel}\\
\mathbb E_{U^j}\!\left[u_t^{j,U}(z^j,a^j)\right]
&=\nu_t\!\left[\frac{1}{|\mathcal S|}-\delta(z^j,a^j)\right].
\label{eq:app-uniform-velocity}
\end{align}
Equation~\eqref{eq:app-uniform-kernel} is exactly the uniform-corruption transition, and
\eqref{eq:app-uniform-velocity} is its infinitesimal velocity. Thus the uniform transition is a
mixture of conditional transitions, each with a single coordinate-specific absorbing token.

\paragraph{Step 4: reverse time to obtain the discrete-flow path.}
The diffusion convention above has $\alpha_0=1$ at data and $\alpha_t$ decreasing to 0 toward uniform noise.
Let $r:=1-t$ denote generation time and set $\kappa_r:=\alpha_{1-r}$. Then $\kappa_0=0$,
$\kappa_1=1$, and the endpoint-conditioned uniform flow path is
\begin{equation}
p_{r\mid1}^j(a^j\mid a_1^j)
=\kappa_r\delta(a^j,a_1^j)
+(1-\kappa_r)\frac{1}{|\mathcal S|}.
\label{eq:app-uniform-flow-path}
\end{equation}
Introducing $U^j\sim\operatorname{Unif}(\mathcal S)$ and using
$\mathbb E[\delta(a^j,U^j)]=1/|\mathcal S|$ gives the path identity
\begin{equation}
p_{r\mid1}^j(a^j\mid a_1^j)
=\mathbb E_{U^j}\!\left[
\kappa_r\delta(a^j,a_1^j)
+(1-\kappa_r)\delta(a^j,U^j)
\right].
\label{eq:app-uniform-flow-mixture}
\end{equation}
For a fixed realization $U^j=u^j\neq a_1^j$, the conditional path inside the expectation moves
once from $u^j$ to $a_1^j$. Its generation-time hazard is
\begin{equation}
h_\kappa(r)=\frac{\dot\kappa_r}{1-\kappa_r}.
\label{eq:app-reverse-hazard}
\end{equation}
Indeed, the source-token mass is $1-\kappa_r$, so its outflow is
$(1-\kappa_r)h_\kappa(r)=\dot\kappa_r$, exactly the derivative of the mass assigned to $a_1^j$.
If $u^j=a_1^j$, the conditional path is constant. Averaging these conditional paths proves the
uniform discrete-flow identity used in Section~\ref{sec:uniform-absorption}.

\paragraph{Step 5: relate the exact lifting to selective absorption.}
The derivation above concerns the exact auxiliary-state representation of uniform transport.
Appendix~\ref{app:absorbing-background} explains the distinct selective augmentation used by our
sampler: Section~\ref{sec:commit} uses $f_t^j$ to record a reliability-based terminal decision. Once
$f_t^j=1$, Eq.~\eqref{eq:absorbing-velocity} sets every outgoing velocity at coordinate $j$ to zero, making
that absorbed value absorbing. This borrows the augmented absorbing-state structure while
deliberately changing the native uniform-flow path according to the absorption policy.

\subsection{The \name{} sampler: selective-absorption methodology}
\label{app:ledflow-details}
%------------------------------------------------------------------
This subsection expands the implementation of Section~\ref{sec:sampler} and separates the two
operations performed at each grid point: an irreversible absorption decision and an otherwise
unchanged uniform-flow update. No token-specific time map or modified CTMC velocity is introduced.

\paragraph{Active positions and absorption count.}
Only generated positions participate in the sampler. Their indicators are initialized as
$f_{t_0}^j=0$, and the active set before step $k$ is
$\mathcal U_k=\{j:f_{t_k}^j=0\}$. The cosine schedule specifies the desired number
\[
m_k=\left\lfloor M\cos\!\left(\frac{\pi}{2}\frac{k+1}{K}\right)\right\rfloor
\]
of active positions after absorption. Consequently,
$b_k=\max\{0,|\mathcal U_k|-m_k\}$ positions are made terminal at that step. Rounding can produce
$b_k=0$, in which case the sampler performs only the native flow update. At the final grid point,
$m_{K-1}=0$, so every remaining generated position is absorbed. This cosine absorption schedule is used for the FUDOKI and main Sudoku absorption comparisons. Both absorbing and non-absorbing Sudoku baselines use the same coordinatewise kinetic-optimal update. The schedule controls absorptions rather than native flow transitions. Explicit singleton-absorption ablations use $b_k=1$ instead (Appendix~\ref{app:setup}).

\paragraph{Entropy selection and value absorption.}
One denoiser evaluation provides the target marginal for each active position. We compute
$H_\theta^j(s_k)$ and choose the $b_k$ smallest values, with a deterministic index order used only
to break exact ties. For every selected $j$, the sampler assigns the most likely target token and
sets $f_{t_k}^j=1$. Equation~\eqref{eq:absorbing-velocity} then makes the position terminal: its
outgoing velocity is zero at every later step and its value is copied unchanged. This is the only
irreversible operation in \name{}.

\paragraph{Native flow on the remaining positions.}
Positions outside $B_k^{\mathrm{LED}}$ retain $f_{t_k}^j=0$ and are advanced from $t_k$ to
$t_{k+1}$ by the backbone's original uniform-flow solver. Thus, \name{} changes which predictions
become final but does not change the learned posterior, the base probability path, or the numerical
integration rule on active positions. Absorption and numerical integration are therefore separate
design axes. For example, a time- or location-corrected update can be applied to the active
positions before the same entropy-based absorption decision at the next grid point.

\paragraph{Matched-policy comparisons.}
The probability-margin, Info-Gain, and arbitrary-order baselines use the same generated-position mask,
task-specific absorption schedule, most-likely-value update, and absorption rule as \name{}. They differ only in the
score used to rank $\mathcal U_k$. This controlled construction attributes differences among these
rows to the priority policy rather than to a different number of absorbed positions or a different
uniform-flow update. In contrast, the native Euler and corrected-flow baselines set
$f_t^j=0$ throughout and therefore remain non-absorbing. They test numerical integration without
the selective-absorption component.

Algorithm~\ref{alg:ledflow} gives the complete update order. The implementation first restricts
selection to generated positions, refreshes the posterior and priority scores, absorbs the selected
block, and finally applies the native velocity update only where the refreshed indicator remains
zero. These steps directly implement Eqs.~\eqref{eq:absorbing-velocity}, \eqref{eq:absorption-update}, and
\eqref{eq:led-block}.

\subsection{Theoretical analysis}\label{app:order}
%------------------------------------------------------------------
\subsubsection{KL motivation and local absorption-error decomposition}\label{app:tc}
\paragraph{Why measure decoding error with KL divergence?}
The starting point follows the sequential-decoding formulation of
\citet{xu2026scheduling}. Let $\pi(y)$ be a target distribution, $\sigma=(\sigma_1,\ldots,\sigma_M)$
an unmasking order, and $v_\phi(\sigma\mid y)$ the probability assigned to that order along the
teacher-forced trajectory of $y$. A denoiser assigns the path likelihood
\begin{equation}
p_\theta(y\mid\sigma)
=\prod_{i=1}^{M}p_\theta^{\sigma_i}
\!\left(y^{\sigma_i}\mid y^{\sigma_{<i}}\right),
\label{eq:app-path-likelihood}
\end{equation}
and the sequential decoder induces
$P_{\theta,\phi}^{\mathrm{seq}}(y)=\sum_\sigma
v_\phi(\sigma\mid y)p_\theta(y\mid\sigma)$. Define the data-policy and model-policy joint laws
\begin{equation}
Q_\phi(y,\sigma):=\pi(y)v_\phi(\sigma\mid y),
\qquad
P_{\theta,\phi}(y,\sigma):=p_\theta(y\mid\sigma)v_\phi(\sigma\mid y).
\label{eq:app-schedule-joints}
\end{equation}
Because both laws use the same order policy, its log-density cancels in their KL divergence:
\begin{align}
\KL\!\left(Q_\phi\,\middle\|\,P_{\theta,\phi}\right)
&=\mathbb E_{y\sim\pi,\,\sigma\sim v_\phi(\cdot\mid y)}
\left[\log\frac{\pi(y)}{p_\theta(y\mid\sigma)}\right]\nonumber\\
&=-H(\pi)
-\mathbb E_{y\sim\pi,\,\sigma\sim v_\phi(\cdot\mid y)}
\left[\log p_\theta(y\mid\sigma)\right].
\label{eq:app-joint-kl-path-loss}
\end{align}
Marginalization from $(y,\sigma)$ to $y$ and the data-processing inequality then give
\begin{equation}
\KL\!\left(\pi\,\middle\|\,P_{\theta,\phi}^{\mathrm{seq}}\right)
\leq
\KL\!\left(Q_\phi\,\middle\|\,P_{\theta,\phi}\right).
\label{eq:app-marginal-kl-bound}
\end{equation}
Equations~\eqref{eq:app-joint-kl-path-loss} to \eqref{eq:app-marginal-kl-bound} explain why KL is
useful for studying absorption decisions: it measures distributional mismatch at the decoder output,
while an upper bound can be evaluated as accumulated prediction loss along an ordered trajectory.
The cited work also considers an ordered block schedule $(B_1,\ldots,B_K)$. If its parallel law
samples positions in each block independently from the corresponding sequential-law marginals,
then its parallelization gap is
\begin{equation}
\KL\!\left(P_{\theta,\phi}^{\mathrm{seq}}\,\middle\|\,
P_{\theta,\phi}^{\mathrm{par}}\right)
=\sum_{k=1}^{K}\mathbb E_{P_{\theta,\phi}^{\mathrm{seq}}}
\!\left[\mathcal{TC}_{P_{\theta,\phi}^{\mathrm{seq}}}
(Y^{B_k}\mid Y^{B_{<k}})\right].
\label{eq:app-xu-parallel-tc}
\end{equation}
This identity measures the dependence lost by parallelizing a sequential model law. Our local
question uses a different reference: it compares the oracle joint posterior at an absorption event
directly with the product of learned denoiser marginals. Proposition~\ref{prop:kl-decomp} therefore
retains a total-correlation term for dependence loss and adds a separate term for marginal
prediction error.

\paragraph{Step 1: abbreviate the three distributions.}
Fix the reachable state $s_k$ and block $B_k$ from Section~\ref{sec:limit}. Write
\begin{equation}
P_B(b):=p^q_{1\mid s_k}(a_1^{B_k}=b),
\quad
P_j(b_j):=p^q_{1\mid s_k}(a_1^j=b_j),
\quad
Q_j(b_j):=\mathbb Q^\theta_{1\mid t_k}(a_1^j=b_j\mid s_k).
\label{eq:app-block-abbreviations}
\end{equation}
The oracle distribution $P_B$ retains dependence among the positions in $B_k$, whereas the
implemented absorption kernel is the product $\prod_{j\in B_k}Q_j$. We assume the latter is
positive on the support of $P_B$. Otherwise the relevant KL divergence is infinite and the identity
below holds in the extended-real sense.

\paragraph{Step 2: insert the product of oracle marginals.}
For every $b$ in the support of $P_B$,
\begin{equation}
\log\frac{P_B(b)}{\prod_{j\in B_k}Q_j(b_j)}
=
\log\frac{P_B(b)}{\prod_{j\in B_k}P_j(b_j)}
+\sum_{j\in B_k}\log\frac{P_j(b_j)}{Q_j(b_j)}.
\label{eq:app-log-ratio-split}
\end{equation}
Taking expectation under $P_B$ separates the local absorption loss into two terms:
\begin{align}
\mathcal L_{\mathrm{abs}}(B_k;s_k)
&=\sum_bP_B(b)\log\frac{P_B(b)}{\prod_{j\in B_k}P_j(b_j)}\nonumber\\
&\quad+\sum_{j\in B_k}\sum_bP_B(b)
\log\frac{P_j(b_j)}{Q_j(b_j)}.
\label{eq:app-expected-split}
\end{align}

\begin{lemma}[Conditional total correlation of an absorption block]\label{lem:tc}
For any reachable state $s_k$ and absorption block $B_k$,
\begin{equation}
\KL\!\left(
p^q_{1\mid s_k}(a_1^{B_k})
\,\middle\|\,
\prod_{j\in B_k}p^q_{1\mid s_k}(a_1^j)
\right)
=\mathcal{TC}_q(a_1^{B_k}\mid s_k)\geq0.
\label{eq:app-tc-lemma}
\end{equation}
The quantity vanishes exactly when the positions in $B_k$ are conditionally independent under the
oracle posterior, including every singleton block.
\end{lemma}

\paragraph{Step 3: identify the joint term.}
The first term in~\eqref{eq:app-expected-split} is
\begin{align}
\KL\!\left(P_B\,\middle\|\,\prod_{j\in B_k}P_j\right)
&=\sum_bP_B(b)\log P_B(b)
-\sum_{j\in B_k}\sum_bP_B(b)\log P_j(b_j)\nonumber\\
&=-H_q(a_1^{B_k}\mid s_k)
+\sum_{j\in B_k}H_q(a_1^j\mid s_k)\nonumber\\
&=\mathcal{TC}_q(a_1^{B_k}\mid s_k)
=\mathcal E_{\mathrm{joint}}(B_k;s_k).
\label{eq:app-joint-term}
\end{align}
The second equality marginalizes $P_B$ over all coordinates except $j$. Since it is itself a KL
divergence, the joint term is nonnegative and vanishes exactly when the coordinates in $B_k$ are
conditionally independent under the oracle posterior. In particular, it is zero for
$|B_k|=1$.

\paragraph{Step 4: identify the conditional term.}
For each $j\in B_k$, marginalizing the second term in~\eqref{eq:app-expected-split} over
$b_{-j}$ gives
\begin{align}
\sum_bP_B(b)\log\frac{P_j(b_j)}{Q_j(b_j)}
&=\sum_{b_j}P_j(b_j)\log\frac{P_j(b_j)}{Q_j(b_j)}\nonumber\\
&=\KL(P_j\|Q_j).
\label{eq:app-conditional-term}
\end{align}
Summing~\eqref{eq:app-conditional-term} over the block yields
$\mathcal E_{\mathrm{cond}}(B_k;s_k)$. Combining
\eqref{eq:app-joint-term} and~\eqref{eq:app-conditional-term} proves
Eq.~\eqref{eq:kl-decomp}.

\paragraph{Step 5: interpret the limiting cases.}
If the block is a singleton, independent sampling introduces no factorization error and only the
denoiser KL remains. If every denoiser marginal equals its oracle marginal, the conditional term
vanishes, but a multi-position block can still incur the joint term because the product kernel does
not reproduce conditional dependence. Both terms vanish when the block is conditionally
independent and every marginal is exact. For a random or adaptive absorption policy, the same
identity holds after conditioning on the realized pair $(s_k,B_k)$. Taking expectation over that
pair preserves the decomposition. Consequently, the policy affects the conditional term through
the states at which predictions become irreversible, while block size determines whether an
additional factorization cost can arise. When the oracle conditional at $s_k$ is a point mass, as
on unique-solution puzzles, every conditional entropy in the total correlation is zero and the joint
term vanishes for every block regardless of its size. A nondegenerate oracle, such as the
multi-solution Sudoku corpus or graph coloring, is required to observe it.

\subsubsection{Proof of Corollary~\ref{cor:risk-bridge-main}}\label{app:risk-bridge}
\begin{proposition}[Conditional error and absorption risk]\label{prop:risk-bridge}
For a reachable state $s_k$ and an absorption block $B_k$ with $|B_k|=b$, let
$\hat z^j=\arg\max_z\mathbb Q^\theta_{1\mid t_k}(a_1^j=z\mid s_k)$,
$c_j=p^q_{1\mid s_k}(a_1^j=\hat z^j)$, and
$c_j^\star=\max_zp^q_{1\mid s_k}(a_1^j=z)$. Then
\begin{equation}
\sum_{j\in B_k}(1-c_j)
\leq
\sum_{j\in B_k}(1-c_j^\star)
+
\sqrt{2b\,\mathcal E_{\mathrm{cond}}(B_k;s_k)}.
\label{eq:risk-bridge}
\end{equation}
\end{proposition}

\paragraph{Proof of Proposition~\ref{prop:risk-bridge}.}
Fix a reachable state $s_k$ and a position $j\in B_k$, and abbreviate the oracle and model marginals
at that position by $p:=p^q_{1\mid s_k}(a_1^j=\cdot)$ and
$\mathbb Q:=\mathbb Q^\theta_{1\mid t_k}(a_1^j=\cdot\mid s_k)$. Let
$\hat z^j=\arg\max_z\mathbb Q(z)$ be the token written by Eq.~\ref{eq:absorption-update} and
$z^{j\star}=\arg\max_zp(z)$ the oracle mode, so that $c_j=p(\hat z^j)$ and $c_j^\star=p(z^{j\star})$.
Adding and subtracting $\mathbb Q$ at both tokens,
\[
c_j^\star-c_j
=
p(z^{j\star})-p(\hat z^j)
=
\bigl[p(z^{j\star})-\mathbb Q(z^{j\star})\bigr]
+
\underbrace{\bigl[\mathbb Q(z^{j\star})-\mathbb Q(\hat z^j)\bigr]}_{\leq0}
+
\bigl[\mathbb Q(\hat z^j)-p(\hat z^j)\bigr],
\]
where the middle bracket is nonpositive because $\hat z^j$ maximizes $\mathbb Q$. Since
$\TV(p,\mathbb Q)=\max_{A\subseteq\mathcal S}|p(A)-\mathbb Q(A)|$ dominates the discrepancy on any
singleton, each outer bracket is at most $\TV(p,\mathbb Q)$, giving the plug-in bound
$c_j^\star-c_j\leq2\,\TV(p,\mathbb Q)$. Pinsker's inequality
$\TV(p,\mathbb Q)\leq\sqrt{\tfrac12\KL(p\,\|\,\mathbb Q)}$ then yields
\[
\left(1-c_j\right)-\left(1-c_j^\star\right)
\;\leq\;
\sqrt{2\,\KL\!\left(p\,\|\,\mathbb Q\right)} ,
\]
whose divergence is exactly the $j$-th summand of $\mathcal E_{\mathrm{cond}}(B_k;s_k)$ in
Eq.~\ref{eq:kl-decomp}. Denote this divergence by $D_j$. Summing over $j\in B_k$ and applying
Cauchy-Schwarz to the $b$ square roots,
\[
\sum_{j\in B_k}\sqrt{2D_j}
\;\leq\;
\sqrt{2b\sum_{j\in B_k}D_j}
\;=\;
\sqrt{2b\,\mathcal E_{\mathrm{cond}}(B_k;s_k)} ,
\]
which gives Eq.~\ref{eq:risk-bridge}. \hfill$\square$

Two remarks delimit the result. First, the bound separates the block's oracle ambiguity
$\sum_{j\in B_k}(1-c_j^\star)$, which is a property of $q$ at $s_k$ for the chosen block, from the
model-dependent excess, which is bounded by a function of $\mathcal E_{\mathrm{cond}}$. The
ambiguity term does depend on which block is selected, since a different block sums different
positions, so a position that is genuinely ambiguous under the oracle is not rendered reliable by a
small conditional KL. The absorption policy of Section~\ref{sec:sampler} targets only the second
term. Second, the $\sqrt{\cdot}$ and the Cauchy-Schwarz step are both loose, and the statement
is a control of an upper bound rather than a ranking of realized risk: decreasing
$\mathcal E_{\mathrm{cond}}$ tightens the bound on the excess absorption risk, but a smaller bound
does not by itself guarantee a smaller realized excess. This is what motivates using the conditional
term as the objective in Section~\ref{sec:sampler}, and it is the reason the policy is evaluated
empirically rather than claimed to minimize realized risk. Combined with
Definition~\ref{def:monotone}, the two bounds are
$(1-c_j)-(1-c_j^\star)\leq\sqrt{2D_j}$ and
$D_j\leq\phi(H_\theta^j)$, connecting the absorption error to the entropy criterion used at inference.

\subsubsection{Entropy-guided absorption: proof of Theorem~\ref{thm:order}}
%------------------------------------------------------------------
This appendix proves the local absorption result in
Theorem~\ref{thm:order}, derives Lemma~\ref{lem:gain-amplification}, and records the scope of both
claims. Throughout, the sampler state $s_k$ is reachable under the oracle process and
$\mathcal U_k$ is the set of positions that remain active immediately before absorption event $k$.

\paragraph{Scope of the main-text results.}
A wrong absorption can leave the oracle-reachable set, in which case the oracle conditional at the resulting state may be undefined. The statements below therefore control error only while the trajectory remains oracle-consistent.
Theorem~\ref{thm:order} is a conditional exchange argument for a surrogate bound, not a derivation of entropy-error regularity, and it does not establish optimal terminal accuracy or control the trajectory after an oracle-inconsistent absorption.
Low entropy alone does not imply correctness: Definition~\ref{def:monotone} is an envelope assumption whose empirical coverage on the Sudoku denoisers, including the relaxed offset $\phi_0$ and the confidently-wrong tail, is reported in Appendix~\ref{app:calib}. The observed coverage is descriptive calibration evidence and does not establish the pointwise assumption or a held-out guarantee.
Corollary~\ref{cor:risk-bridge-main} shows that entropy selection minimizes the excess-risk bound for a fixed absorption count, but it need not minimize total risk, because the oracle term also varies across blocks. For $b>1$, even exact marginals need not produce a jointly valid block of modes.
Lemma~\ref{lem:gain-amplification} is worst-case additive, so the realized lookahead error can grow more slowly than $\eta_m$. It compares score sensitivity rather than terminal losses under the two policies, and does not prove that local entropy is better whenever the bound is large.

\paragraph{Relaxed entropy-error regularity.}\label{app:relaxed_regularity}
The relaxed inequality in Eq.~\ref{eq:relaxed-regularity} adds a position-independent constant $\phi_0$ to accommodate the confidently-wrong tail at entropy near zero. It holds on 99.4\% of sampled states for the 6.4M denoiser in the diagnostic of Appendix~\ref{app:calib}.
As $\phi_0$ is the same for every active position, it shifts $\Phi(B_k;s_k)$ by the constant $b\phi_0$ and leaves the minimizer of $\sum_{j\in B_k}\phi\!\left(H_\theta^j(s_k)\right)$ unchanged. The bound remains conditional on the relaxed inequality holding pointwise. An offset chosen as a mean provides empirical coverage, not a universal upper bound.
\begin{equation}
\KL\!\left(
p^q_{1\mid s_k}(a_1^j)
\,\middle\|\,
\mathbb Q^\theta_{1\mid t_k}(a_1^j\mid s_k)
\right)
\leq
\phi_0+\phi\!\left(H_\theta^j(s_k)\right),
\qquad \phi_0\geq0,
\label{eq:relaxed-regularity}
\end{equation}

\paragraph{Alternative uncertainty scores.}
For maximum predicted probability and probability margin, the corresponding uncertainty scores are $1-p_{(1)}$ and $1-(p_{(1)}-p_{(2)})$, respectively, where $p_{(1)}\geq p_{(2)}$ are the two largest posterior probabilities. Minimizing these scores selects the highest-confidence positions, and entropy in Eq.~\ref{eq:led-block} can be replaced by the corresponding score to obtain a different ordering policy under the same absorption count.
Definition~\ref{def:monotone} is stated for predictive entropy. Entropy regularity does not imply the analogous conditions for maximum probability or probability margin, and the scores can rank positions differently, so the resulting policies are compared empirically in Appendix~\ref{app:results-reasoning}.

\paragraph{Oracle objective and tractable envelope.}
Fix an absorption count $b\geq1$ and recall
$\mathfrak B_k(b)=\{B\subseteq\mathcal U_k:|B|=b\}$. An oracle that knew the realized target
sequence could fix every absorbed position correctly, eliminating errors propagated from
earlier irreversible decisions. For the inference setting, where absorbed values are still predicted
by the denoiser, we use the true conditional posterior only as an analytical reference. By
Proposition~\ref{prop:kl-decomp}, this oracle reference would select
\[
B_k^{\mathrm{oracle}}
\in
\arg\min_{B\in\mathfrak B_k(b)}
\underbrace{
\sum_{j\in B}
\KL\!\left(
p^q_{1\mid s_k}(a_1^j)
\,\middle\|\,
\mathbb Q^\theta_{1\mid t_k}(a_1^j\mid s_k)
\right)
}_{\mathcal E_{\mathrm{cond}}(B;s_k)}.
\]
This decision cannot be evaluated at inference because $p^q_{1\mid s_k}$ is unavailable. Under
Definition~\ref{def:monotone}, however, every summand admits the denoiser-computable envelope
\[
\KL\!\left(
p^q_{1\mid s_k}(a_1^j)
\,\middle\|\,
\mathbb Q^\theta_{1\mid t_k}(a_1^j\mid s_k)
\right)
\leq \phi\!\left(H_\theta^j(s_k)\right).
\]
Summing over a candidate block gives
\[
\mathcal E_{\mathrm{cond}}(B;s_k)
\leq
\Phi(B;s_k)
:=
\sum_{j\in B}\phi\!\left(H_\theta^j(s_k)\right).
\]
The entropy policy need not coincide with $B_k^{\mathrm{oracle}}$. The claim is that it exactly
minimizes this tractable upper bound.

\paragraph{Proof of Theorem~\ref{thm:order}.}
\begin{proof}
Fix an absorption count $b$ and let $B_{\mathrm{low}}$ contain $b$ positions having the smallest predictive
entropies. Consider any other block $B$ of size $b$. If $B\neq B_{\mathrm{low}}$, there are
$i\in B_{\mathrm{low}}\setminus B$ and $j\in B\setminus B_{\mathrm{low}}$ with
$H_\theta^i(s_k)\leq H_\theta^j(s_k)$. Since $\phi$ is nondecreasing, exchanging $j$ for $i$
cannot increase $\Phi$ because
\[
\Phi\!\left((B\setminus\{j\})\cup\{i\};s_k\right)-\Phi(B;s_k)
=
\phi\!\left(H_\theta^i(s_k)\right)-\phi\!\left(H_\theta^j(s_k)\right)
\leq0.
\]
Repeating the exchange produces $B_{\mathrm{low}}$, proving that it minimizes $\Phi$ over
$\mathfrak B_k(b)$.

When $b=1$, Eq.~\eqref{eq:entropy-policy} follows immediately. Moreover, conditional total
correlation vanishes for a singleton:
\[
\mathcal E_{\mathrm{joint}}(\{j\};s_k)
=
H_q(a_1^j\mid s_k)-H_q(a_1^j\mid s_k)
=0.
\]
Hence Proposition~\ref{prop:kl-decomp} reduces to
$\mathcal L_{\mathrm{abs}}(\{j\};s_k)=\mathcal E_{\mathrm{cond}}(\{j\};s_k)$, and the lowest
entropy position minimizes the upper bound on the complete local absorption error.

For part (ii), convergence to the oracle posterior on reachable states means that, for every
admissible $j$ and $z^j\in\mathcal S$,
\[
\mathbb Q^\theta_{1\mid t_k}(a_1^j=z^j\mid s_k)
\longrightarrow
p^q_{1\mid s_k}(a_1^j=z^j).
\]
Because $\mathcal S$ is finite, each term with positive oracle mass converges continuously and each
zero-mass term contributes zero. Therefore,
\[
\KL\!\left(p^q_{1\mid s_k}(a_1^j)\,\middle\|\,
\mathbb Q^\theta_{1\mid t_k}(a_1^j\mid s_k)\right)\longrightarrow0.
\]
Every admissible block contains finitely many positions, so
\[
\mathcal E_{\mathrm{cond}}(B;s_k)
=
\sum_{j\in B}
\KL\!\left(p^q_{1\mid s_k}(a_1^j)\,\middle\|\,
\mathbb Q^\theta_{1\mid t_k}(a_1^j\mid s_k)\right)
\longrightarrow0.
\]
This holds for every admissible block and thus for the block selected by any absorption policy,
which proves part (ii).
\end{proof}

\paragraph{Why the result is local.}
The exchange argument holds at a fixed state and fixed absorption count. Absorbing a prediction
changes both the active set and the conditioning state used by the denoiser at later events.
Theorem~\ref{thm:order} therefore controls the next block contribution to
$\mathcal E_{\mathrm{cond}}$. It does not claim that greedy entropy selection minimizes the final
loss over every complete absorption trajectory. When $b>1$, it does not optimize the joint term
$\mathcal E_{\mathrm{joint}}(B;s_k)$, which can vary with the selected block.

\paragraph{Reachable states.}\label{app:order-remarks}\label{rem:reachable}
The regularity condition is stated on oracle-reachable states, where every absorbed partial
configuration has positive probability under $q$ and its oracle conditional is defined. A learned
sampler can leave this set after an incorrect irreversible prediction. At such a state,
$p^q_{1\mid s_k}$ may be undefined if no valid completion remains, so the KL condition cannot be
evaluated. The theorem should consequently be read as controlling the error incurred while the
sampler remains on an oracle-consistent trajectory. Its experiments test whether this local rule
also reduces departures from that set. They do not turn the pointwise premise into a global
guarantee.

\paragraph{Sampling versus most-likely absorption.}
The KL decomposition in Proposition~\ref{prop:kl-decomp} uses the categorical denoiser marginal as
the prediction kernel. Algorithm~\ref{alg:ledflow} instead absorbs its most likely value. The two
kernels are generally different, including on unique-solution tasks unless the learned marginal is
itself degenerate. We therefore use the divergence analysis to justify the position-selection rule
and evaluate the deployed most-likely-value rule empirically rather than claiming equality of the
two output laws.

\begin{remark}[Entropy at inference and training]\label{rem:twouses}
The inference policy absorbs low-entropy positions because they have the smallest next-block
conditional-error bound. This differs from weighting a training loss, which changes how the
denoiser is learned. Theorem~\ref{thm:order} concerns only the inference decision and does not
assume that predictive entropy equals downstream information gain.
\end{remark}

\subsubsection{Derivation of the global lookahead bound}
Recall that $m=|\mathcal U(s)|$ and
$R_r(s)=\sum_{j\in\mathcal U(s)}h_j^r(s)$ for $r\in\{q,\theta\}$. Write
\[
G_j^r(s)=R_r(s)-\mathbb E_{z\sim P_r^j(\cdot\mid s)}R_r(s^{j,z}),
\qquad
D_{jj'}^r(s):=G_j^r(s)-G_{j'}^r(s),
\]
so that $R_r(s)$, which does not depend on the candidate, cancels exactly in $D_{jj'}^r$. The
selection made by Info-Gain depends on the scores only through the differences $D_{jj'}^\theta$, so
these are the quantities the bound must control.

For a single successor aggregate the assumed pointwise entropy error gives
$|R_\theta(s^{j,z})-R_q(s^{j,z})|\leq(m-1)\delta$, since $s^{j,z}$ leaves $m-1$ positions active.
Moreover $0\leq R_q(s^{j,z})\leq(m-1)\log|\mathcal S|$, so the total-variation assumption implies
\[
\left|
\mathbb E_{z\sim P_\theta^j}R_q(s^{j,z})-
\mathbb E_{z\sim P_q^j}R_q(s^{j,z})
\right|
\leq(m-1)\,\epsilon\log|\mathcal S|,
\]
where we used $|\mathbb E_{P}f-\mathbb E_{P'}f|\leq\TV(P,P')\,(\sup f-\inf f)$ with
$\TV$ in the $\tfrac12\ell_1$ convention of Section~\ref{sec:relorder}. Combining the two
displays, the error of the expected successor aggregate for a single candidate is at most
$\eta_m:=(m-1)(\delta+\epsilon\log|\mathcal S|)$.

\paragraph{Successors outside the oracle support.}
The expectation over $z\sim P_\theta^j(\cdot\mid s)$ can place mass on a value $z$ with
$P_q^j(z\mid s)=0$. The successor $s^{j,z}$ then has no oracle conditional, so $h_{j'}^q(s^{j,z})$
is undefined. We adopt the convention $h_{j'}^q(s^{j,z}):=\log|\mathcal S|$ for every active
$j'$ at such a successor, treating an oracle-inconsistent absorption as leaving the remaining
positions maximally uncertain, and we require the pointwise entropy condition $|h^\theta-h^q|\leq\delta$
only at successors of positive model probability with this extension in force. Any other bounded
extension leaves the total-variation step unchanged because it uses only the range
$[0,(m-1)\log|\mathcal S|]$ of $R_q$. If one prefers not to impose the entropy condition at
off-support successors at all, their total model mass is at most $\TV(P_\theta^j,P_q^j)\leq\epsilon$,
which adds at most $\epsilon(m-1)\log|\mathcal S|$ to the first display and changes $\eta_m$ to
$(m-1)(\delta+2\epsilon\log|\mathcal S|)$ without altering the form or reading of the lemma.

Since the current aggregate cancels,
\[
\left|D_{jj'}^\theta(s)-D_{jj'}^q(s)\right|
\leq
\left|\mathbb E_{P_\theta^j}R_\theta(s^{j,z})-\mathbb E_{P_q^j}R_q(s^{j,z})\right|
+
\left|\mathbb E_{P_\theta^{j'}}R_\theta(s^{j',z})-\mathbb E_{P_q^{j'}}R_q(s^{j',z})\right|
\leq2\eta_m.
\]
For the model and oracle maximizers, insert and subtract the model scores and note that the
differences again eliminate $R_r(s)$:
\[
G_{j^\star}^q-G_{\hat j}^q
=
D_{j^\star\hat j}^q
\leq
\underbrace{D_{j^\star\hat j}^\theta}_{\leq0}
+\left|D_{j^\star\hat j}^\theta-D_{j^\star\hat j}^q\right|
\leq2\eta_m.
\]
The analogous two-error argument for the local entropy minimizers gives
\[
h_{\hat j_{\mathrm{ent}}}^q(s)-\min_jh_j^q(s)\leq2\delta,
\]
which completes the proof of Lemma~\ref{lem:gain-amplification}. \hfill$\square$

\paragraph{Interpretation of the lookahead result.}
Three qualifications delimit what the lemma does and does not say. First, it bounds score
estimation, not the realized terminal loss: Info-Gain remains effective whenever $2\eta_m$ is small
relative to the oracle margin $D_{j^\star j'}^q$ between its best and second-best candidates, and the
decision-relevant diagnostic is therefore the ratio of the two rather than $\eta_m$ alone. Second,
because the bound is stated on score \emph{differences}, it is invariant to the normalization of the
aggregate: rescaling every candidate score by a common positive factor rescales $\eta_m$ and the
margin identically and changes neither the selected candidate nor the comparison with the local rule.
The window dependence that remains comes from the $m-1$ positions genuinely carried inside each
successor aggregate, not from the choice of units. Third, the additivity over those positions is
worst case. If the per-position entropy errors are not adversarially aligned, a concentration
argument would give $O(\sqrt{m-1})$ in place of $O(m-1)$, so the lemma should be read as an upper
envelope on how the lookahead window can amplify error rather than as the typical rate. The result isolates the
source of accumulation: global lookahead combines errors over the unresolved positions, whereas local
entropy selection estimates one position at the current state.

\subsubsection{Additional comparison with Info-Gain}\label{app:capability-crossover}
At a reachable state $s$, couple the entropy and Info-Gain policies so that they differ only in the
next absorbed position and share the continuation. Let $L_{\mathrm{local}}$ and $L_{\mathrm{IG}}$
be their final losses. Let $E_s$ be the event that Info-Gain selects an oracle-consistent position
and set $\kappa_\theta(s)=\mathbb P(E_s\mid s)$. Define
\[
B_s=\mathbb E[L_{\mathrm{local}}-L_{\mathrm{IG}}\mid E_s,s],\qquad
C_s=\mathbb E[L_{\mathrm{IG}}-L_{\mathrm{local}}\mid E_s^c,s].
\]
When correct lookahead helps and incorrect lookahead hurts, $B_s,C_s\geq0$ measure the corresponding
benefit and cost.

\begin{proposition}[When the local entropy policy beats Info-Gain]\label{thm:capability-crossover}
At any state satisfying the definitions above,
\[
\mathbb E[L_{\mathrm{IG}}-L_{\mathrm{local}}\mid s]
=(1-\kappa_\theta(s))C_s-\kappa_\theta(s)B_s.
\]
If $B_s+C_s>0$, the local entropy policy has no larger expected final loss if and only if
\begin{equation}\label{eq:crossover-threshold}
\kappa_\theta(s)\leq\frac{C_s}{B_s+C_s}.
\end{equation}
\end{proposition}

\begin{proof}
Conditioning on $E_s$ and $E_s^c$ gives the displayed identity. Requiring its right-hand side to
be nonnegative and rearranging gives Eq.~\eqref{eq:crossover-threshold}.
\end{proof}

This identity does not assert that local entropy always outperforms lookahead. It makes the
comparison depend on how often the lookahead choice is oracle-consistent and on the relative costs
of its correct and incorrect choices. These quantities can be estimated by paired continuations on
tasks with reference solutions, as illustrated by the paired continuation estimates below.

\paragraph{Measured comparison.}
On Sudoku-Extreme, paired continuations estimate the three quantities in
Proposition~\ref{thm:capability-crossover} and recover the observed winner at both denoiser sizes
(Table~\ref{tab:crossover}). This is a diagnostic of the stated identity, not evidence that model size itself defines the threshold. It is also a per-state comparison with a shared continuation: end-to-end on the same 21.3M denoiser the entropy policy still solves more puzzles than Info-Gain (0.964 against 0.947 at $n{=}1000$, $p{=}0.027$, Table~\ref{tab:scaleup}), so the per-state advantage of lookahead at states where it is oracle-consistent has not yet translated into a trajectory-level win at this capacity.

\begin{table}[t]
\centering
\small
\caption{Measured comparison on Sudoku-Extreme (600 paired reachable states per denoiser).
$\kappa_\theta$ is the probability that Info-Gain selects an oracle-consistent position and
$\kappa^\star=C_s/(B_s+C_s)$. The local entropy policy has no larger expected loss when
$\kappa_\theta\leq\kappa^\star$.}
\label{tab:crossover}
\begin{tabular}{lccccc}
\toprule
Denoiser & $\kappa_\theta$ & $\kappa^\star$ & $B_s$ & $C_s$ & better policy\\
\midrule
$6.4$M & $0.527$ & $0.698$ & $0.0043$ & $0.0100$ & entropy\\
$21.3$M  & $0.695$ & $0.397$ & $0.0038$ & $0.0025$ & Info-Gain\\
\bottomrule
\end{tabular}
\end{table}

\subsubsection{Calibration and miscalibration}\label{app:calib}
The following analysis provides an empirical diagnostic of the entropy-error relationship in
Definition~\ref{def:monotone} and measures the surrogate slack used by Theorem~\ref{thm:order}.
Per-cell error rises with the denoiser's predictive entropy, from 0.015 in the lowest-entropy bin
to 0.526 in the highest (Table~\ref{tab:calib}), and the empirical slope decreases from
the 6.4M to the 21.3M denoiser. The two reported constants move in opposite directions with
capacity, the slope falling from 0.45 to 0.37 while the surrogate gap $\varepsilon$ rises from
0.21 to 0.30, and the two measure different things. The slope describes how sharply error grows
with the model's own uncertainty, and a better-calibrated denoiser improves it. The surrogate gap
describes how well a \emph{static} score tracks the \emph{path-conditional} predictive entropy, and
a stronger denoiser exploits constraint propagation more aggressively, so its entropies move more
along the path and a fixed ranking tracks them less well. A capable denoiser therefore makes the
entropy-error relationship tighter and the fixed-score policy's job harder at the same time. The
fixed score nevertheless recovers only part of the adaptive policy's gain on the weaker model.

\begin{figure}[t]\centering
\includegraphics[width=\textwidth]{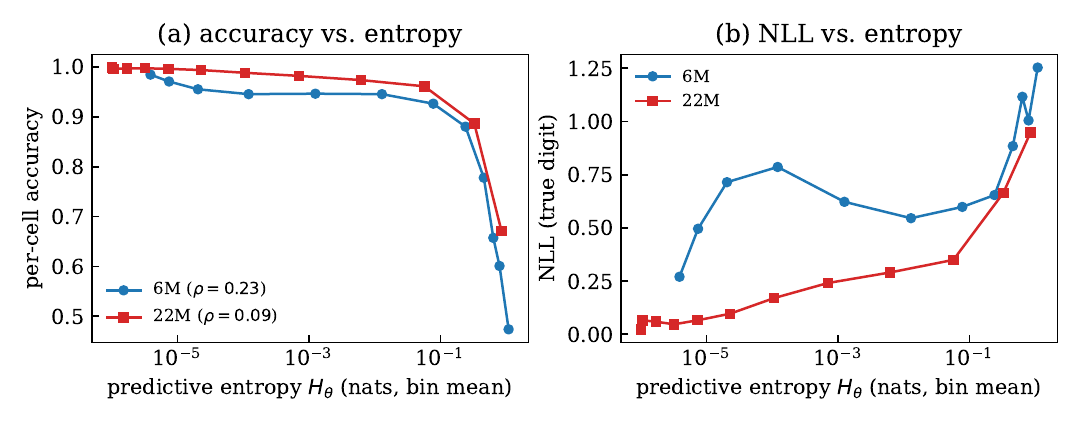}
\caption{\textbf{Entropy-accuracy correlation, both denoisers.} Same binned trajectory data as
Table~\ref{tab:calib} (all 12 equal-count bins, 6.4M and 21.3M), plotted against the model's
predictive entropy $H_{\theta}$ on a log axis. (a) Per-cell accuracy ($1{-}\mathrm{err}$) decreases
with entropy at both capacities, with Spearman $\rho(H,\mathrm{err})=0.23$ (6.4M) and 0.09 (21.3M)
(the low $\rho$ reflects that most mass sits in the flat, near-zero-entropy, near-ceiling regime
common to reachable Sudoku states, and the drop is concentrated in the top bins, consistent with the
nonlinear slope $L_{\phi}$ reported below). (b) NLL of the true digit rises with entropy at both
capacities. The 21.3M denoiser's NLL curve sits uniformly below the 6.4M curve, the calibration
improvement with capacity that motivates comparing the two models throughout this appendix.}
\label{fig:entropy-calib}
\end{figure}

\paragraph{Pointwise envelope.} Definition~\ref{def:monotone} requires a \emph{pointwise} bound
$\mathrm{KL}\le\phi(H_{\theta})$, not merely the bin-mean trend above. Fitting a nondecreasing
95th-percentile upper envelope to the full per-cell cloud (158,191 reachable partial configurations) gives an
empirical envelope slope $L_{\mathrm{env},0.95}\approx1.70$ (6.4M) and 1.34 (21.3M), about
3.6 to 3.8 times the bin-mean OLS slope. The robust calibration summary reports a small
\emph{confidently-wrong} fraction (0.99\% at 6.4M and 0.35\% at 21.3M),
with near-zero-entropy NLL reaching approximately 20.9 and 22.3 nats, respectively. The fitted inequality
with
$\phi(0)=0$ at the mean slope covers 95.5\% (6.4M) and 97.6\% (21.3M)
of sampled reachable states. The relaxed envelope uses the mean offset $\phi_0=0.57$ nats for 6.4M and 0.12 nats for 21.3M (Table~\ref{tab:phi-offsets}), under which $\mathrm{KL}\leq\phi_0+\phi(H_\theta)$ holds on 99.4\% and 99.7\% of sampled states. For any fixed position-independent offset, the additive contribution $b\phi_0$ preserves the entropy argmin at a fixed absorption count, but a mean offset does not establish a pointwise upper bound. Substituting the updated envelope slopes
into the diagnostic expression $2L\varepsilon M$ gives 39.8 and 45.2 nats, respectively,
compared with 10.5 and 12.4 using the bin-mean slopes. These fitted slopes and empirical
coverage rates do not establish a uniform bound on every reachable state.

\begin{table}[t]
\centering\small
\caption{Offset $\phi_0$ summaries for the relaxed envelope (nats). The mean column gives the offsets used to compute the reported coverage. The other columns describe the distribution's upper tail.}
\label{tab:phi-offsets}
\begin{tabular}{lrrrr}
\toprule
Denoiser & Mean & 95th percentile & 99th percentile & Maximum\\
\midrule
$6.4$M & 0.57 & 0.002 & 16.70 & 20.85\\
$21.3$M & 0.12 & 0.001 & 0.003 & 22.33\\
\bottomrule
\end{tabular}
\end{table}

\paragraph{Oracle versus model entropy.} To separate ``entropy is the ordering signal'' from ``\emph{the
model's} entropy is the signal,'' we compare absorbing by the exact oracle entropy $H_q$ against absorbing by the model's $H_{\theta}$. Here $H_q$ is the entropy
of the completion-count distribution over viable digits, computed on a multi-solution corpus where
$H_q>0$ is non-degenerate. Over 120 multi-solution
puzzles (median 23 completions), $H_{\theta}$ correlates strongly with $H_q$ ($r=0.88$ at 6.4M,
0.90 at 21.3M), and ordering by the model's offline $H_{\theta}$ matches or exceeds ordering by the
exact $H_q$ (valid-completion rate 0.958 vs.\ 0.900 at 6.4M, and 0.983 vs.\ 0.983 at 21.3M), while
adaptively recomputing $H_{\theta}$ reaches 1.000. In this experiment, model entropy matches or exceeds the
tested oracle-entropy ordering. The improvement from adaptive recomputation suggests that updating the
score as the active state changes matters beyond the initial ranking.

\begin{table}[t]\centering\small
\caption{Entropy-error diagnostic (Definition~\ref{def:monotone}) on Nikoli ($n{=}100$): per-cell error and
NLL of the true digit, binned by the model's predictive entropy $H_{\theta}$ over the whole sampling
trajectory (6.4M denoiser, representative equal-count bins). Per-cell error rises monotonically
with entropy. The bottom row reports the empirical error-entropy slope and surrogate error
$\varepsilon$ for both denoisers.}
\label{tab:calib}
\begin{tabular}{lcccccc}
\toprule
$H_{\theta}$ (bin mean) & 0.00 & 0.24 & 0.46 & 0.64 & 0.79 & 1.08\\
\midrule
per-cell error & 0.015 & 0.120 & 0.222 & 0.343 & 0.399 & 0.526\\
NLL (true digit) & 0.27 & 0.65 & 0.88 & 1.12 & 1.01 & 1.25\\
\midrule
\bottomrule
\end{tabular}\par\smallskip
\parbox{\linewidth}{\footnotesize 6.4M: empirical slope $=0.45$, $\varepsilon{=}0.21$ \quad\textbar\quad 21.3M: empirical slope $=0.37$, $\varepsilon{=}0.30$\par 95th-pct envelope slope: $L_{\mathrm{env},0.95}{\approx}1.70$/1.34 (6.4M/21.3M). Reported confidently-wrong fraction 0.99\%/0.35\%. Strict-$\phi$ compliance 95.5\%/97.6\%, relaxed-$\phi$ compliance 99.4\%/99.7\%.}

\end{table}

\begin{table}[t]\centering\small
\caption{Two attempts to make a fixed external priority beat adaptive entropy-guided absorption by
inducing miscalibration. \emph{internal} is adaptive entropy and \emph{external} is the best of the fixed
external score and logical forcedness. Top: density
shift (train dense, test sparse). Bottom: wrong constraint (train Latin, test Sudoku) over a
difficulty sweep. The external order wins only in the floored hard band (a non-significant
${\sim}2$-puzzle margin). Wherever the denoiser is capable, confidence wins.}
\label{tab:miscal}
\begin{tabular}{llccc}
\toprule
probe & regime (denoiser) & internal & external & gap\\
\midrule
\multicolumn{5}{l}{\textit{(a) density shift, test givens $26$--$31$, $n{=}185$}}\\
& calibrated (ECE $0.008$)     & 0.973 & 0.746 & $+0.227$\\
& overconfident (ECE $0.074$)  & 0.827 & 0.546 & $+0.281$\\
\midrule
\multicolumn{5}{l}{\textit{(b) wrong constraint: Latin-trained (box-blind), test Sudoku}}\\
& hard (givens $32$--$40$, floored)  & 0.109 & \textbf{0.124} & $-0.015^{*}$\\
& medium (givens $40$--$48$)         & 0.637 & 0.630 & $+0.007$\\
& easy (givens $44$--$52$)           & 0.827 & 0.810 & $+0.017$\\
\bottomrule
\end{tabular}\\[2pt]
{\footnotesize $^{*}$floored regime (${\sim}0.11$ solve). ${\sim}2$-puzzle margin at $n{=}137$, not significant.}
\end{table}

              % C: details for Section 3 (C.1 formulation, C.2 sampler, C.3 theory)
\section{Experimental settings and hyperparameters}\label{app:setup}
%==============================================================================
This appendix collects the experimental protocols and hyperparameters for \S\ref{sec:exp}, following the order of the main experiments.

\begin{table}[t]
\centering\footnotesize
\caption{Protocol map for Sudoku results. All samplers at a given capacity share the same checkpoint. Each ID identifies a separate reported experiment, not an interchangeable estimate. $K$ is the number of sampling steps. It need not equal the number of singleton absorptions. Solve accuracies and policy labels are those of Table~\ref{tab:sudoku}.}
\label{tab:protocol-map}
\begin{tabular}{@{}>{\raggedright\arraybackslash}p{.07\linewidth}>{\raggedright\arraybackslash}p{.24\linewidth}>{\raggedright\arraybackslash}p{.32\linewidth}>{\raggedright\arraybackslash}p{.25\linewidth}@{}}
\toprule
ID & Checkpoint / location & Decoder or absorption protocol & Cohort / sampling steps\\
\midrule
M & Shared task-specific $6.4$M; Table~\ref{tab:sudoku}, all rows; Fig.~\ref{fig:overjumpy} & Kinetic-optimal updates; absorbing rows add a cosine absorption schedule & Matched across all rows: same puzzles, Generated/Nikoli/Extreme $n=500/100/700$, five seeds, $K=64$; Fig.~\ref{fig:overjumpy} traces the Nikoli subset\\
A & Task-specific $6.4$M; Table~\ref{tab:temp-entropy} & Ordering-score ablation; cosine absorption schedule & Temperature sweep, seeds $0,1$, $K=64$\\
L & Shared model at each capacity; Table~\ref{tab:order-ladder} & Singleton absorptions & Each set $n=100$, five seeds\\
G & $6.4$M and $21.3$M; Table~\ref{tab:scaleup} & Separate singleton-ordering study & Generated $n=1000$, $22$--$34$ givens\\
\bottomrule
\end{tabular}
\end{table}

\paragraph{Sampler protocol.} At each grid point, \textbf{\name{}} recomputes predictive entropy
from the current sampler state and absorbs the lowest-entropy active positions under the
absorption count. Active positions then follow the backbone's unchanged update, while
absorbed positions remain fixed under Eq.~\eqref{eq:absorbing-velocity}. Top K-margin,
Info-Gain, and arbitrary-order policies use the same absorption schedule, value update, and absorption rule and differ only
in the score used to select the block. The absorption schedule is task specific: on the FUDOKI understanding,
text-to-image, and mathematics harnesses the cosine schedule of Appendix~\ref{app:ledflow-details}
sets $b_k$ at each of the $K$ grid points. The main Sudoku comparisons also use a cosine absorption schedule over $K=64$ steps, with coordinatewise kinetic-optimal updates in both absorbing and non-absorbing arms (Appendix~\ref{app:flow-active-gate}). Explicit singleton-absorption diagnostics evaluate once per generated position (Table~\ref{tab:efficiency}). Their NFE and vanishing joint term do not describe the runs with cosine absorption schedules. Appendix~\ref{app:ledflow-details} gives the complete update
order and matched-policy controls.

\paragraph{Absorption-policy baselines.}\label{app:ordering_baselines}
Top K-margin selects the $b_k$ active positions with largest $p_{(1)}-p_{(2)}$, where $p_{(1)}\geq p_{(2)}$ are the two largest posterior probabilities, and uses the same argmax absorption rule as \name{}. Info-Gain uses a lookahead window of eight, the best-performing setting in Figure~\ref{fig:theory-aligned-analysis}(c), with its additional evaluation cost reported in Appendix~\ref{app:results-efficiency}.

\paragraph{Understanding.}\label{app:understanding_protocol} We evaluate with VLMEvalKit~\citep{duan2024vlmevalkit} on the full
standard splits, following the established protocol~\citep{wang2025fudoki,wan2026corrected}. The
base model is frozen FUDOKI throughout and only the decode-time sampler changes. The non-absorbing
samplers, guidance, Top K-margin, and
\name{} all use one denoiser evaluation per grid point.
Info-Gain additionally spends lookahead evaluations (Appendix~\ref{app:results-efficiency}). Split sizes, metrics, and the 95\% intervals appear in
Table~\ref{tab:understanding}. The reported Euler MMBench score uses dev-EN, whereas our evaluation uses the full standard split.

\paragraph{Generation.} GenEval~\citep{ghosh2023geneval} is evaluated on the full suite of 553
prompts with 4 images per prompt, scored with the standard Mask2Former and CLIP pipeline, with the
standard $K{=}64$ sampling steps (denoiser evaluations). FUDOKI's native default is 128 and scores saturate
by $K{\sim}32$, consistent with prior reports~\citep{wan2026corrected}. We use one fixed \name{}
configuration for the full suite and all reported category scores.

\paragraph{Reasoning.}\label{app:sudoku_protocol} All Sudoku results at a given model capacity use the same checkpoint. In Table~\ref{tab:sudoku}, both non-absorbing and absorbing rows use the shared task-specific 6.4M Sudoku denoiser (Appendix~\ref{app:absorbing-background}) at $K=64$ with coordinatewise kinetic-optimal updates. Non-absorbing rows keep every generated position active throughout. Absorbing rows apply the same update at active positions and use a cosine absorption schedule $m_k$ to determine how many positions to absorb. These absorbing policies differ only in their ordering score. All rows are evaluated on the same puzzles and seeds (protocol M in Table~\ref{tab:protocol-map}).
The capacity diagnostics use a 6.4M denoiser of width 256 with 8 layers and 8 heads, trained for 60k steps. Their capacity-sensitivity comparison against a 21.3M denoiser, width 384 with 12 layers and
12 heads trained for 120k steps on a mixed easy-to-hard corpus, is reported separately in
\S\ref{app:sudoku-capacity}. This is a robustness check rather than a controlled scaling study,
because capacity, number of training steps, and training distribution change together.  The three Sudoku sets
are held-out synthetic puzzles, Nikoli~\citep{seely2025sudoku}, and
Sudoku-Extreme~\citep{wang2025hrm}. The additional constraint structures are $9{\times}9$ Latin
squares, 3-coloring of random graphs, and molecular infilling on MOSES
SMILES~\citep{polykovskiy2020moses}. The Sudoku denoisers are trained on the Kaggle million-puzzle
set after deduplication by puzzle string and solution grid against the evaluation sets.

\paragraph{Reasoning evaluation details.} All Sudoku rows of Table~\ref{tab:sudoku}, absorbing and non-absorbing, follow the matched protocol of Section~\ref{sec:exp}: the same checkpoint, the same 500 generated, 100 Nikoli, and 700 Sudoku-Extreme puzzles, five seeds, and $K{=}64$. Figure~\ref{fig:overjumpy} traces this cohort's Nikoli subset and reports its solve accuracies. The capacity and ordering-ladder diagnostics retain their own $n{=}100$, five-seed protocol. Zero seed variance under deterministic decoding does not imply zero puzzle-sampling uncertainty. Table~\ref{tab:scaleup} reports puzzle-level confidence intervals and paired tests on a separate $n{=}1000$ cohort. Its generated subset is not the generated set in Table~\ref{tab:sudoku}, so their absolute accuracies should not be used interchangeably.

\paragraph{Generality.} For the additional constraint structures, we train a fresh 4.8M denoiser
for $9{\times}9$ Latin squares and a 1.2M graph-attention discrete flow denoiser for random
3-coloring. Molecular infilling uses a character-level model on MOSES SMILES and $n{=}2000$
examples. The pretrained-model experiments change only the native sampler while leaving model
weights fixed. On GSM8K, LLaDA-8B-Instruct~\citep{nie2025llada} uses $n{=}200$ and Dream-7B
~\citep{ye2025dream} uses $n{=}1000$, both with 128 sampling steps, generation length 256, and
argmax token choice. LLaDA additionally uses block length 32. MMaDA-8B~\citep{yang2025mmada}
uses $n{=}500$, 256 steps, and chain-of-thought prompting. BBH logical deduction pools $n{=}750$
three-, five-, and seven-object puzzles. For open multimodal mathematics, LLaDA-V uses seed-averaged
default-order baselines on MathVerse ($n{=}900\times4$) and MathVista ($n{=}1000\times2$), whereas MMaDA
uses a single default-order run on $n{=}900$ and $n{=}1000$, respectively.

%==============================================================================
        % D: experimental settings and hyperparameters
\section{Additional experimental analysis}\label{app:extra-exp}\label{app:extra}
%==============================================================================
This appendix provides the additional analyses and full result tables for \S\ref{sec:exp}, in the same order as the main experiments.

\subsection{Understanding}\label{app:results-understanding}
\paragraph{Uncertainty of the main-text gains.}\label{app:understanding_intervals}
The gain that clears its reported approximate 95\% interval is +1.6 on POPE ($\pm1.3$). The remaining margins, including +2.3 on MM-Vet and +1.2 on MMMU, are smaller than their intervals and we do not read them as individually significant. Table~\ref{tab:understanding} reports the intervals and their approximation limits.

\begin{table}[t]\centering\small
\caption{Understanding suite on the full standard splits via FUDOKI's VLMEvalKit. The table includes
the six benchmarks of Table~\ref{tab:mm-understanding}. The last column is an approximate 95\% interval
on the \emph{native-sampler-vs-\name{} difference} ($1.96\sqrt{2p(1-p)/N}$, evaluated at the native-sampler rate).
Entropy-guided absorption improves every benchmark in the main-text suite, and the POPE margin clears the
reported interval. The remaining margins are smaller than their intervals. This places short-answer
understanding at the low-sensitivity end of the observed task spectrum, consistent with a smaller
conditional term than on the constrained tasks in \S\ref{sec:exp-sudoku}.}
\label{tab:understanding}
\resizebox{\columnwidth}{!}{%
\begin{tabular}{lrccccc}
\toprule
benchmark & $N$ & native sampler & \name{} (ours) & metric & spread & $95\%$ CI (diff)\\
\midrule
POPE      & 5127  & 87.4   & 89.0   & acc & $+1.6$  & $\pm1.3$\\
MME       & 2374  & 1494.5 & 1504.6 & P+R & $+10.1$ & ---\\
MMBench   & 4377  & 73.7   & 74.6   & Overall & $+0.9$ & $\pm1.8$\\
GQA     & 12578 & 57.0   & 58.0   & acc & $+1.0$ & $\pm1.2$\\
MM-Vet    & 218   & 38.2   & 40.5   & GPT-judge & $+2.3$ & $\pm9.1$\\
MMMU      & 1050  & 36.2   & 37.4   & acc & $+1.2$ & $\pm4.1$\\
\bottomrule
\end{tabular}}
\\[2pt]
{\footnotesize The MME row is a summed perception-plus-reasoning score rather
than an accuracy, so the binomial interval does not apply. These intervals do not account for paired predictions. The binomial approximation is also not a calibrated interval for graded GPT-judge scores such as MM-Vet. Paired example-level resampling is needed for inferential comparisons.}
\end{table}

\begin{table}[t]\centering\small
\caption{Understanding suite in context: reported unified understanding-and-generation models
alongside our frozen-FUDOKI decode-time samplers (condensed in Table~\ref{tab:mm-understanding} of
the main text). Numbers are compiled from prior work~\citep{discreteguidance2026}. Understanding is a
measurement-validity control, not a SOTA claim. Higher is better, ``---'' not reported.}
\label{tab:mm-understanding-full}
\resizebox{\columnwidth}{!}{%
\begin{tabular}{lcccccc}
\toprule
model & POPE & MME$^{\dagger}$ & MMBench$^{\ddagger}$ & GQA & MMMU & MM-Vet$^{\S}$\\
\midrule
\multicolumn{7}{l}{\textit{Reported unified understanding-and-generation models (context)}}\\
LWM            & 75.2 & ---    & ---  & 44.8 & ---  & 9.6\\
Chameleon      & ---  & ---    & ---  & ---  & 22.4 & 8.3\\
Show-o-256     & 73.8 & 948.4  & ---  & 48.7 & 25.1 & ---\\
Show-o-512     & 80.0 & 1097.2 & ---  & 58.0 & 26.7 & ---\\
D-DiT          & 84.0 & 1124.7 & ---  & 59.2 & ---  & ---\\
VILA-U         & 85.8 & 1401.8 & ---  & 60.8 & ---  & 33.5\\
ILLUME         & 88.5 & 1445.3 & 65.1 & ---  & 38.2 & 37.0\\
TokenFlow-XL   & 86.8 & 1545.9 & 68.9 & 62.7 & 38.7 & 40.7\\
Janus          & 87.0 & 1338.0 & 69.4 & 59.1 & 30.5 & 34.3\\
Janus-Pro-1B   & 86.2 & 1444.0 & \textbf{75.5} & 59.3 & 36.3 & 39.8\\
FUDOKI (reported)~\citep{wang2025fudoki} & 86.1 & 1485.4 & 73.9 & 57.6 & 34.3 & 38.0\\
FUDOKI $+$ guidance~\citep{discreteguidance2026} & 86.8 & 1492.7 & 74.2 & 58.2 & 35.4 & 38.6\\
\midrule
\multicolumn{7}{l}{\textit{Frozen FUDOKI, decode-time samplers (this setting)}}\\
native sampler (FUDOKI) & 87.38 & 1494.5 & 73.7 & 57.0 & 36.2 & 38.2\\
\name{} (ours)          & 89.04 & 1504.6 & 74.6 & 58.0 & 37.4 & 40.5\\
\bottomrule
\end{tabular}}
\\[2pt]
{\footnotesize $^{\dagger}$our MME is MME-P (perception only), the metric the reported models use. $^{\ddagger}$our
MMBench is the full standard split, harder than the dev-EN split behind the published FUDOKI 73.9.
$^{\S}$MM-Vet uses a GPT judge, and absolute values shift with the judge version, so we re-grade with
the standard 0 to 1 partial-credit protocol (VLMEvalKit), which gives scores in the reported range.}
\end{table}

\subsection{Generation}\label{app:results-generation}
\subsubsection{Text-to-image generation}\label{sec:exp-t2i}
We use this setting as a control on soft structure, where the absorption policy should matter less
than on tightly constrained outputs. We evaluate the full GenEval suite (553 prompts, four images
per prompt, Mask2Former and CLIP scoring) on frozen FUDOKI with the standard $K{=}64$ sampling steps.
\name{} reaches 0.781 Overall against the native sampler's 0.753. Its improvements span the
category scores, including color attribute (0.723 vs.\ 0.663) and position (0.715 vs.\
0.680).

\paragraph{Paired bootstrap intervals on GenEval.} We resample the 553 prompts with replacement,
stratified within each tag and keeping a prompt's four images together, with $10{,}000$ replicates.
The Overall gain of \name{} over the native sampler is +0.0282 with a 95\% interval of
$[+0.009,+0.048]$ and $p=0.003$, so the interval excludes zero. The per-tag differences are
color attribute +0.060 $[+0.005,+0.117]$, position +0.035 $[-0.018,+0.090]$, colors +0.029
$[-0.003,+0.061]$, single object +0.019 $[-0.016,+0.056]$, counting +0.019 $[-0.037,+0.078]$,
and two objects +0.008 $[-0.028,+0.043]$. Only color attribute is individually significant, so the
Overall gain reflects a consistent direction across tags rather than one decisive category. This gain
is small relative to those on constrained reasoning tasks, consistent with text-to-image prompts
imposing weaker structure. The comparison with exact guidance (0.7800) cannot be tested this way,
because that value is a published number~\citep{discreteguidance2026} rather than a per-prompt run of
ours, so we report it as a numerical comparison only.

\paragraph{GenEval in the masked discrete flow literature.} Placing \name{} against the
recent masked discrete flow literature on the same FUDOKI base (Table~\ref{tab:t2i-landscape}) separates two
orthogonal axes. Among \emph{frozen-model samplers} (our setting) the whole family sits within a
band of about 0.05: \name{} (0.781) is the strongest entry, ahead of the native sampler
(0.753) and of the corrected-sampler family (Euler 0.754, time-/location-corrected
0.76 to 0.77~\cite{wan2026corrected}), though the margin over the corrected samplers is comparable
to the benchmark's own resolution. The higher
scores in the literature come from the orthogonal \emph{training} axis (guidance/RLHF
0.78~\cite{discreteguidance2026}, rate-aware policy optimization 0.93~\cite{wan2026dflowgrpo}),
which modify the model rather than the sampler and are complementary to our decode-time contribution.
Table~\ref{tab:mm-geneval-full} places the frozen-sampler comparison against generation-only and
unified models.

\begin{table}[t]\centering\footnotesize
\setlength{\tabcolsep}{4.2pt}
\caption{GenEval in context: generation-only and unified understanding-and-generation models
alongside the frozen-FUDOKI decode-time samplers of Table~\ref{tab:mm-geneval} (main text). The
external rows modify or replace the model and are orthogonal to the decode-time comparison. External
rows as compiled by their cited source. Higher is better, ``---'' not reported.}
\label{tab:mm-geneval-full}
\begin{tabular}{lccccccc}
\toprule
method & Single & Two & Count & Colors & Pos. & Attr. & Overall$\uparrow$\\
\midrule
\multicolumn{8}{l}{\textit{Generation-only models (context)}}\\
SDXL~\citep{wan2026dflowgrpo}          & 0.98 & 0.74 & 0.39 & 0.85 & 0.15 & 0.23 & 0.55\\
DALL-E\,3~\citep{wan2026dflowgrpo}     & 0.96 & 0.87 & 0.47 & 0.83 & 0.43 & 0.45 & 0.67\\
SD3-Medium~\citep{discreteguidance2026}    & 0.99 & 0.94 & 0.72 & 0.89 & 0.33 & 0.60 & 0.74\\
FLUX.1-dev~\citep{wan2026dflowgrpo}    & 0.98 & 0.93 & 0.75 & 0.93 & 0.68 & 0.65 & 0.82\\
\midrule
\multicolumn{8}{l}{\textit{Unified understanding-and-generation models (context)}}\\
Chameleon~\citep{discreteguidance2026}     & ---  & ---  & ---  & ---  & ---  & ---  & 0.39\\
UniDisc~\citep{discreteguidance2026}       & 0.92 & 0.47 & 0.15 & 0.67 & 0.13 & 0.19 & 0.42\\
LWM~\citep{discreteguidance2026}           & 0.93 & 0.41 & 0.46 & 0.79 & 0.09 & 0.15 & 0.47\\
SEED-X~\citep{discreteguidance2026}        & 0.97 & 0.58 & 0.26 & 0.80 & 0.19 & 0.14 & 0.49\\
Emu3-Gen~\citep{discreteguidance2026}      & 0.98 & 0.71 & 0.34 & 0.81 & 0.17 & 0.21 & 0.54\\
Show-o~\citep{discreteguidance2026}        & 0.95 & 0.52 & 0.49 & 0.82 & 0.11 & 0.28 & 0.53\\
\quad $+$ Mask-GRPO~\citep{wan2026dflowgrpo} & 0.99 & 0.90 & 0.69 & 0.85 & 0.35 & 0.59 & 0.73\\
TokenFlow-XL~\citep{discreteguidance2026}  & 0.95 & 0.60 & 0.41 & 0.81 & 0.16 & 0.24 & 0.55\\
ILLUME~\citep{discreteguidance2026}        & 0.99 & 0.86 & 0.45 & 0.71 & 0.39 & 0.28 & 0.61\\
Janus~\citep{discreteguidance2026}         & 0.97 & 0.68 & 0.30 & 0.84 & 0.46 & 0.42 & 0.61\\
D-DiT~\citep{discreteguidance2026}         & 0.97 & 0.80 & 0.54 & 0.76 & 0.32 & 0.50 & 0.65\\
Janus-Pro-1B~\citep{discreteguidance2026}  & 0.98 & 0.82 & 0.51 & 0.89 & 0.65 & 0.56 & 0.73\\
Janus-Pro-7B~\citep{wan2026dflowgrpo}  & 0.99 & 0.89 & 0.59 & 0.90 & 0.79 & 0.66 & 0.80\\
MMaDA~\citep{wan2026dflowgrpo}         & 0.96 & 0.60 & 0.45 & 0.81 & 0.14 & 0.25 & 0.56\\
\midrule
\multicolumn{8}{l}{\textit{Frozen FUDOKI with decode-time samplers (Table~\ref{tab:mm-geneval})}}\\
Euler@64 (base FUDOKI)~\citep{wan2026corrected} & 0.9625 & 0.8384 & 0.4875 & 0.8833 & 0.7200 & 0.6300 & 0.7536\\
\textbf{\name{}} (ours) & 0.9563 & 0.8712 & 0.5219 & 0.9016 & 0.7150 & 0.7225 & \textbf{0.7814}\\
\bottomrule
\end{tabular}
\end{table}

\subsubsection{Comparison with related generation methods}
\begin{table}[t]\centering\small
\caption{GenEval Overall on FUDOKI in the masked discrete flow literature, separating decode-time
\emph{samplers} (frozen model, our setting) from \emph{training}-based methods.
Among frozen-model samplers the spread is ${\approx}0.05$, and an external FUDOKI Euler@64 result
reproduces our native sampler to 0.001. \name{} is the strongest decode-time entry
in this comparison, although its margin over the corrected-sampler family is within the benchmark's
resolution. Training-based methods modify the model and address a separate axis.}
\label{tab:t2i-landscape}
\begin{tabular}{lc}
\toprule
method & GenEval Overall\\
\midrule
\multicolumn{2}{l}{\textit{Frozen-model samplers (this setting)}}\\
native sampler (Euler@64, ours) & 0.753\\
Euler@64~\citep{wan2026corrected} & 0.754\\
time-/location-corrected~\citep{wan2026corrected} & $0.765$\\
\textbf{\name{}} (ours) & \textbf{0.781}\\
\midrule
\multicolumn{2}{l}{\textit{Training-based (context)}}\\
FUDOKI base ($+$CFG)~\citep{discreteguidance2026} & 0.77\\
$+$ guidance / RLHF~\citep{discreteguidance2026} & 0.78\\
$+$ dFlowGRPO~\citep{wan2026dflowgrpo} & 0.93\\
\bottomrule
\end{tabular}
\end{table}

\subsection{Trajectory analysis under matched settings}\label{app:trajectory}
Figure~\ref{fig:overjumpy} compares native uniform flow with \name{} on 100 Nikoli puzzles at $K{=}64$: puzzle solve accuracy rises from 0.410 to 0.845, while correct-to-wrong reversions fall from 0.094 to 0.026 of generated cells. These are the motivation measurements used in the main text.
On text-to-image generation the two directions must be read on the counterfactual prompts. On
in-distribution GenEval prompts the native flow recovers more than it destroys late in sampling
(0.156 wrong-to-correct against 0.052 correct-to-wrong), which is what memorized configurations
predict when the prompt distribution is the one the backbone was trained on. On the counterfactual
prompts of Table~\ref{tab:ood-prompts}, where that route is unavailable, the ordering reverses and
correct-to-wrong dominates (0.125 against 0.037). Table~\ref{tab:trajectory} and Figure~\ref{fig:traj-t2i} report the text-to-image trajectory analysis.

\begin{table}[t]
\centering
\small
\caption{\textbf{Text-to-image trajectory measurements.} FUDOKI text-to-image (GenEval prompts,
$K{=}64$). Token statistics are per image token, with 95\% bootstrap intervals over images for the
difference, and correctness flips are the fraction of samples whose GenEval verdict changes between
$t{=}0.75$ and the endpoint, on in-distribution GenEval prompts and on the 16 counterfactual
out-of-distribution prompts of Table~\ref{tab:ood-prompts}. The Sudoku counterpart of these
measurements is Figure~\ref{fig:overjumpy}(c).
$^{\star}$GenEval prompts are in-distribution for the pretrained backbone, so late revisions there can
recover memorized configurations. The counterfactual prompts isolate the same measurement without
that confound, and under them correct-to-wrong exceeds wrong-to-correct.}
\label{tab:trajectory}

\setlength{\tabcolsep}{4pt}
\resizebox{\textwidth}{!}{%
\begin{tabular}{@{}lccc@{}}
\toprule
Measurement & Euler & \name{} & Diff.\ [$95\%$ CI] \\
\midrule
Token changes per position
    & $4.14$ & $2.33$ & $+1.81$ [$1.72,1.88$] \\
\quad excluding the final step
    & $3.44$ & $2.30$ & $+1.13$ [$1.04,1.21$] \\
Positions changed more than once
    & $0.97$ & $0.73$ & $+0.24$ [$0.22,0.25$] \\
Changes occurring after $t{=}0.75$
    & $0.45$ & $0.11$ & $+0.34$ [$0.34,0.35$] \\
Mean time of last change
    & $0.96$ & $0.54$ & $+0.42$ [$0.41,0.43$] \\
\midrule
GenEval accuracy
    & $0.753$ & $0.781$ & $-0.028$ [$-0.037,-0.020$] \\
\midrule
In-dist.\ correct to wrong$^{\star}$
    & $0.052$ & $0.010$ & $+0.042$ \\
In-dist.\ wrong to correct$^{\star}$
    & $0.156$ & $0.000$ & $+0.156$ \\
OOD correct to wrong
    & $0.125$ & $0.025$ & $+0.100$ \\
OOD wrong to correct
    & $0.037$ & $0.013$ & $+0.024$ \\
\bottomrule
\end{tabular}%
}
\end{table}

\begin{figure}[t]\centering
\includegraphics[width=0.8\columnwidth]{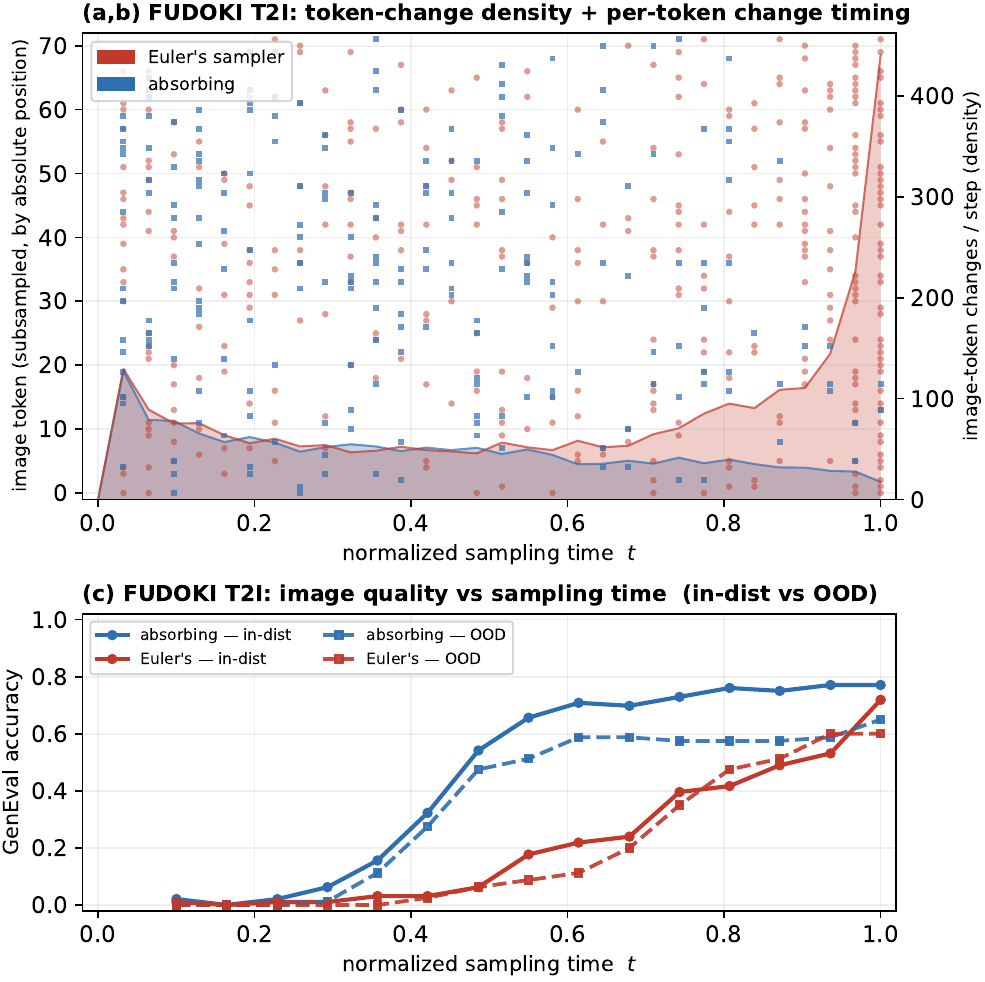}
\caption{\textbf{FUDOKI text-to-image trajectories, native uniform flow versus \name{}.} (a,b)
Image-token change density over sampling time: the native flow accumulates most of its changes
after $t{=}0.75$, while \name{} front-loads absorption and its density decays. (c) GenEval accuracy
of the image decoded from the intermediate state, on in-distribution GenEval prompts (solid) and on
the 16 counterfactual prompts of Table~\ref{tab:ood-prompts} (dashed). \name{} reaches its final quality by $t{\approx}0.6$ and holds
it. The native flow reaches comparable quality only at the endpoint, after the late burst of
changes in (a,b).}
\label{fig:traj-t2i}
\end{figure}

\begin{table}[t]\centering\small
\caption{\textbf{Out-of-distribution prompts for the text-to-image trajectory analysis.} Sixteen
counterfactual prompts, four per GenEval tag, written so that the requested attribute, count, pairing,
or spatial relation is unlikely under the training distribution. They are scored with the standard
GenEval detector pipeline and are the dashed curves in Figure~\ref{fig:traj-t2i} and the OOD rows of
Table~\ref{tab:trajectory}.}
\label{tab:ood-prompts}
\setlength{\tabcolsep}{6pt}
\begin{tabular}{@{}lp{0.78\linewidth}@{}}
\toprule
Tag & Prompts\\
\midrule
Colors & a photo of a green cow; a photo of a purple giraffe; a photo of a blue banana; a photo of a pink elephant\\
Counting & a photo of eight apples; a photo of six cats; a photo of seven birds; a photo of five clocks\\
Two objects & a photo of a toothbrush and a giraffe; a photo of a pizza and a fire hydrant; a photo of a laptop and a cow; a photo of an elephant and a spoon\\
Position & a photo of a car above a bird; a photo of a boat above a bed; a photo of a bicycle above an elephant; a photo of a cup above a train\\
\bottomrule
\end{tabular}
\end{table}
\paragraph{Ordering diagnostic.} Table~\ref{tab:matched-gate} compares absorption policies under the uniform discrete flow. All four arms use the same
denoiser, puzzles, grid ($K{=}64$), and argmax value rule, and the three absorbing arms share one
code path. The first two use argmax values and differ in selection rule, while the third samples the absorbed value. Arbitrary-order absorption raises cell accuracy from 0.413 to 0.723 (paired
puzzle-level 95\% interval $[0.285,0.336]$) and removes every correct-to-wrong flip. Ordering
raises it further to 0.935 (interval $[0.166,0.256]$) and raises solve accuracy from 0.410 under
the arbitrary order to 0.845. The mechanism is visible in the bad-absorption column, the fraction
of absorption events that write an incorrect value, which falls from 28.0\% under the arbitrary
order to 6.6\% under the entropy order. Absorbing a sampled value instead of the argmax changes
little (0.923 against 0.935 cell accuracy, 0.821 against 0.845 solve), so the gain comes from
freezing plus ordering rather than from the value rule.

\begin{table}[t]\centering\small
\caption{\textbf{Absorption-policy diagnostic on Nikoli} ($n{=}100$, 6.4M denoiser, $K{=}64$, seeds
0 to 4). Bad-absorption is the fraction of absorption events whose absorbed value is wrong. The arbitrary-order and
entropy-order solve accuracies are the values of Table~\ref{tab:sudoku}. The no-absorption arm never
fixes a position, so no puzzle completes. Paired
puzzle-level bootstrap intervals for the cell-accuracy differences: no absorption to arbitrary
+0.310 $[0.285,0.336]$. Arbitrary to entropy +0.212 $[0.166,0.256]$.}
\label{tab:matched-gate}
\setlength{\tabcolsep}{5pt}
\resizebox{\linewidth}{!}{%
\begin{tabular}{@{}lcccccc@{}}
\toprule
Arm & Cell acc. & Solve & Correct$\to$wrong & Changes/cell & Late-change frac. & Bad-absorption\\
\midrule
Uniform flow (no absorption) & $0.413$ & $0.000$ & $0.281$ & $1.87$ & $0.60$ & ---\\
\cdashline{1-7}[2pt/1pt]
\quad + arbitrary absorption order & $0.723$ & $0.410$ & $0.000$ & $1.00$ & $0.37$ & $0.280$\\
\quad + entropy order (\name{}) & $\mathbf{0.935}$ & $\mathbf{0.845}$ & $0.000$ & $1.00$ & $0.37$ & $\mathbf{0.066}$\\
\quad + entropy order, sampled absorption & $0.923$ & $0.821$ & $0.000$ & $1.00$ & $0.37$ & $0.078$\\
\bottomrule
\end{tabular}}
\end{table}

\subsection{Absorption-order comparison and selected-position analysis}\label{app:flow-active-gate}
Table~\ref{tab:flow-active-gate} reports the current Sudoku solve accuracies and policy labels from Table~\ref{tab:sudoku}. All methods at this capacity use the shared task-specific 6.4M Sudoku checkpoint. The non-absorbing methods use kinetic-optimal updates. Absorbing methods additionally use a cosine absorption schedule. All rows are evaluated on the same puzzles and seeds (protocol M in Appendix~\ref{app:setup}). The Euler and Time-corrected rows coincide on every puzzle column. Appendix~\ref{app:absorbing-collapse} explains why the two integrators agree to first order on these tasks and quantifies the residual. The ordering diagnostic in Table~\ref{tab:matched-gate} and the earlier trajectory diagnostics below are separate experiments.

\begin{table}[t]\centering\small
\caption{\textbf{Sudoku solve accuracy by sampler.} These values and policy labels reproduce Table~\ref{tab:sudoku}. All rows share the same checkpoint, puzzles, seeds, and $K{=}64$ (protocol M). The identical Euler and Time-corrected rows are a discretization effect, explained in Appendix~\ref{app:absorbing-collapse}.}
\label{tab:flow-active-gate}
\begin{tabular}{lccc}
\toprule
Sampler & Generated & Nikoli & Sudoku-Extreme\\
\midrule
Euler & 0.610 & 0.410 & 0.078\\
FreeCorrection & 0.630 & 0.400 & 0.126\\
Time-corrected & 0.610 & 0.410 & 0.078\\
Location-corrected & 0.588 & 0.360 & 0.096\\
Top K-margin & 0.820 & 0.796 & 0.271\\
Info-Gain & 0.804 & 0.780 & 0.231\\
\name{} & 0.865 & 0.845 & 0.269\\
\bottomrule
\end{tabular}
\end{table}

\paragraph{What the current comparison shows.}
On generated and Nikoli puzzles, entropy absorption attains higher solve accuracy than the other displayed policies. On Sudoku-Extreme, margin is numerically higher (0.271 versus 0.269). These are descriptive comparisons on a matched cohort of 100 puzzles per set. The paired significance tests for the ordering policies are reported at $n{=}1000$ in Table~\ref{tab:scaleup}.

\paragraph{Statistical evidence.}
On seed-0 per-puzzle arrays, exact McNemar tests favor entropy over the native sampler and over arbitrary absorption on all three sets ($p\leq3\times10^{-8}$). This diagnostic does not separate entropy from the other confidence scores. The paired comparison against probability margin and maximum probability is the $n{=}1000$ test in Table~\ref{tab:scaleup}, where the three are statistically tied and entropy separates from Info-Gain (-0.066, $p{=}3\times10^{-9}$).

\paragraph{Decoder and schedule.}
Both non-absorbing and absorbing baselines use coordinatewise kinetic-optimal updates. Here a singleton flow transition changes one coordinate. It does not specify the number of positions absorbed at a grid point. The non-absorbing baseline keeps all generated positions active. Absorbing policies apply the same kinetic-optimal update to positions that remain active, while a cosine absorption schedule over $K=64$ grid steps determines the size of each irreversible absorption block. A step can absorb zero, one, or several positions, so this protocol is not the $b_k=1$ absorption ablation and need not have zero joint dependence error. For policies without additional lookahead, the grid protocol uses one denoiser evaluation per step, giving $\mathrm{NFE}=K$. Figure~\ref{fig:overjumpy} reports separate trajectory and endpoint runs under kinetic-optimal updates.

\paragraph{Selected-position reliability.}
As a reliability diagnostic rather than a validation of a specific theoretical envelope, Table~\ref{tab:d9-calib}
bins the entropy gate's absorbed positions by entropy at absorption and reports the empirical absorption-error
rate per bin. We do not fit or test a KL envelope here. Under the flow-active decode almost all absorptions are
made at very low entropy (99.6\% / 99.4\% of absorptions on generated/Nikoli), where the error rate is
5.9\% / 8.6\%. Error rises with entropy over the sparse higher-entropy bins. On Sudoku-Extreme the
denoiser is over-confident: even the lowest-entropy absorptions are wrong 44.8\% of the time. The first
erroneous absorption is correspondingly rare and late on easy puzzles (generated: 14\% of puzzles have
any wrong absorption, median first-wrong time $t{=}0.55$, and Nikoli: 16\%, $t{=}0.31$) but common and early on
hard puzzles (Sudoku-Extreme: 75\%, $t{=}0.23$). Low absorption error where the policy operates is
consistent with, but does not by itself establish, the conditional guarantee's envelope assumption. The
hard-puzzle regime shows confidently wrong absorption that any such assumption must exclude.

\begin{table}[t]\centering\small
\caption{\textbf{Absorption error at absorbed positions} (\name{} gate, seed 0, flow-active decode): empirical
error rate binned by entropy at absorption, with the number of absorbed positions per bin. Low entropy
predicts low error on generated/Nikoli. The relation fails under hard-puzzle over-confidence
(Sudoku-Extreme). ``---'' marks empty bins.}
\label{tab:d9-calib}
\setlength{\tabcolsep}{6pt}
\begin{tabular}{@{}lcccccc@{}}
\toprule
& \multicolumn{2}{c}{Generated} & \multicolumn{2}{c}{Nikoli} & \multicolumn{2}{c}{Sudoku-Extreme}\\
\cmidrule(lr){2-3}\cmidrule(lr){4-5}\cmidrule(lr){6-7}
Entropy at absorption & err & $n$ & err & $n$ & err & $n$\\
\midrule
$[0.00,0.05)$ & $0.059$ & $26496$ & $0.086$ & $5540$ & $0.448$ & $36240$\\
$[0.05,0.10)$ & $0.700$ & $20$ & $0.167$ & $6$ & $0.550$ & $866$\\
$[0.10,0.20)$ & $0.750$ & $16$ & $0.333$ & $6$ & $0.570$ & $732$\\
$[0.20,0.40)$ & $0.750$ & $16$ & $0.778$ & $9$ & $0.537$ & $555$\\
$[0.40,0.80)$ & $0.600$ & $35$ & $0.800$ & $10$ & $0.622$ & $585$\\
$[0.80,1.60)$ & $0.778$ & $9$ & $0.000$ & $1$ & $0.714$ & $70$\\
$[1.60,10.0)$ & $1.000$ & $1$ & --- & $0$ & $1.000$ & $6$\\
\midrule
Overall & $0.061$ & $26593$ & $0.088$ & $5572$ & $0.457$ & $39054$\\
\bottomrule
\end{tabular}
\end{table}

\subsection{Reasoning}\label{app:results-reasoning}
Two qualifications apply to the main-text reasoning table. Info-Gain leads on molecular infilling, so local entropy is not uniformly preferable, and the current results do not isolate why that ordering reverses. Average ranks are descriptive: small mathematics differences, including near-floor GSM8K accuracy, do not establish meaningful wins.
\subsubsection{Sudoku as a transparent oracle}\label{sec:sudoku}
Sudoku provides a transparent oracle because quantities that are usually inaccessible can be
computed directly, in particular the terms in the absorption-policy analysis
(Theorem~\ref{thm:order}). The joint term is the exception: on a unique-solution puzzle the oracle
conditional is a point mass, so the conditional total correlation of Lemma~\ref{lem:tc} is
identically zero for every block, and unique-solution Sudoku cannot be used to diagnose the joint
term. For that term we use the multi-solution variant described below and the graph-coloring task
in Appendix~\ref{app:exp-latin}. Its vocabulary and sequence length are bounded by
$|\mathcal S|=9$ digits and $M\leq81$ cells, much smaller than the subword vocabulary and open-ended
length of free-text generation~\citep{tao2024scaling}. More importantly, a standard Sudoku puzzle has
a unique solution~\citep{tjusila2023cluesgivebilevelformulation}, so
$q(a_{1}\mid\mathcal C)=\mathbf 1[a_{1}=a_{1}^{\star}]$ and $H(q)=0$. The oracle denoiser is
therefore invariant to the absorption priority by \Cref{sec:limit-order}. Any priority effect exhibited by a learned
denoiser arises from approximation error and is captured by the conditional term
$\mathcal E_{\mathrm{cond}}$ in \Cref{prop:kl-decomp}. For every reachable partial fill $a_S$ with
at least one valid completion, the true conditional distribution is computable from completion counts:
\[
q(a_{1}^{j}=v\mid a_{S})=\frac{N(a_{S},j,v)}{\sum_{v'}N(a_{S},j,v')},
\qquad
H_{q}(a_{1}^{j}\mid a_{S})=-\sum_{v=1}^{9}q(a_{1}^{j}=v\mid a_{S})\log q(a_{1}^{j}=v\mid a_{S}),
\]
where $N(a_S,j,v)$ counts valid completions of $a_S$ with cell $j$ fixed to $v$. Constraint
propagation and backtracking compute these counts in milliseconds for standard $9\times9$ puzzles.
When no valid completion exists, the conditional is undefined, so oracle diagnostics use only
consistent prefixes. On the unique-solution RRN~\citep{palm2018rrn},
Sudoku-Extreme~\citep{wang2025hrm}, and million-puzzle Kaggle corpora, the counts satisfy
$N\in\{0,1\}$ and $H_q$ vanishes along every consistent path. The multi-solution
variant~\citep{nandwani2021neural} provides nondegenerate oracle entropy when needed.

\paragraph{Structured-task-specific baselines.}
Table~\ref{tab:structured-baselines} retains two controls omitted from the cross-task main table.
Uniformization is evaluated here as an exact CTMC control on the puzzle tasks. We did
not adapt it to the mathematics harness. The fixed-score absorption policy
similarly ranks positions by a static task score before selective absorption, so it does not directly
apply when absorption is not used and positions can be rewritten. We also did not define a corresponding
static relevance score for character-level SMILES infilling.

\begin{table}[t]\centering\small
\caption{Structured-task results for uniformization and fixed-score absorption. Higher is better.
``---'' denotes the SMILES setting without a static relevance score.}
\label{tab:structured-baselines}
\setlength{\tabcolsep}{5pt}
\begin{tabular}{lcccccc}
\toprule
Sampler & Generated & Nikoli & Extreme & Latin & Graph & Molec.\\
\midrule
Uniformization & 0.562 & 0.240 & 0.074 & 0.694 & 1.000 & 0.831\\
Fixed-score absorption & 0.702 & 0.420 & 0.158 & 0.946 & 1.000 & ---\\
\bottomrule
\end{tabular}
\end{table}

\subsubsection{Reasoning controls}
Two controls separate the absorption policy from the other choices a sampler makes.
Table~\ref{tab:scaleup} repeats the Sudoku comparison at $n{=}1000$ with puzzle-level bootstrap
intervals and paired McNemar tests, so that the ordering gap is read against sampling uncertainty
rather than seed variance, and stratifies it by difficulty. Its lower block adds the competing
ordering policies on the same puzzles: the probability-margin and maximum-probability policies are
not significantly different from entropy in these paired tests ($p{=}0.23$ and 0.80). This does not establish statistical equivalence, and the scores need not induce the same ordering of positions. Info-Gain trails by 0.066 ($p{=}3{\times}10^{-9}$) at roughly six times the wall-clock on this task. On the stronger 21.3M denoiser the gap narrows to 0.017 but remains significant ($p{=}0.027$), consistent with the crossover analysis of Appendix~\ref{app:capability-crossover}: lookahead gains on the entropy policy as the denoiser improves, but has not overtaken it end-to-end at this capacity. The separation
that matters is therefore between confidence-ordered single-position absorption and either
lookahead ordering or no ordering at all, which is the comparison Lemma~\ref{lem:gain-amplification}
and Figure~\ref{fig:theory-aligned-analysis}(c) speak to. Table~\ref{tab:control} then crosses the
absorption priority with the value-selection rule: holding value selection fixed and varying only
the priority moves accuracy far more than holding the priority fixed and switching between argmax
and sampled values. This attributes the gain to position selection rather than to the rule used to
assign an absorbed value. Table~\ref{tab:sudoku-capacity} extends the same manipulation to arbitrary
and high-entropy-first policies.

\begin{table}[t]\centering\small
\caption{Sudoku at scale ($n{=}1000$ generated puzzles, 6.4M, \#givens $\in[22,34]$): solve accuracy
with 95\% puzzle-level bootstrap CIs and paired McNemar tests. Upper block: paired against
Euler. The adaptive and fixed-score absorption policies are both significant, and logical forcedness
underperforms the model's entropy policy (an imperfect denoiser tracks calibration, not
forcedness). Lower block: the competing single-position ordering policies on the same 1000
puzzles, paired against the entropy policy. The three confidence-based scores (entropy, probability
margin, maximum probability) are statistically tied, whereas Info-Gain (window 8) is
significantly below them and unordered Euler far below. Wall-clock per puzzle on one H100 is
6.9\,ms (entropy), 6.2 (margin), 2.4 (max-probability), 54.5 (Info-Gain), and 8.6 (Euler, re-run at 0.427, within the interval of the upper-block value). On the 21.3M denoiser the entropy policy remains ahead of Info-Gain (0.964 against 0.947, $p{=}0.027$) at 9.5 against 149.3\,ms. Right block: accuracy stratified
by difficulty, where the ordering gap grows as puzzles harden. We note that this is not Theorem~\ref{thm:order}(ii) directly, which concerns the limit in
which the denoiser approaches the oracle. The two agree only under the additional premise that
harder puzzles leave the denoiser further from the oracle, which is plausible here but is not
established by the theorem.}
\label{tab:scaleup}
\resizebox{\linewidth}{!}{%
\begin{tabular}{lccc|ccc}
\toprule
 & & & vs Euler & \multicolumn{3}{c}{by \#givens (determined)}\\
absorption policy & solve-acc & $95\%$ CI & (McNemar) & $22$--$25$ & $26$--$29$ & $30$--$34$\\
\midrule
entropy guided (adaptive) & \textbf{0.922} & $[.905,.938]$ & $+.486$, $p{<}10^{-12}$ & 0.837 & 0.949 & 0.967\\
logical-forcedness (MRV) & 0.621 & $[.591,.651]$ & $+.185$, $p{<}10^{-12}$ & --- & --- & ---\\
fixed external score & 0.527 & $[.496,.558]$ & $+.091$, $p{=}9{\times}10^{-7}$ & 0.288 & 0.493 & 0.775\\
arbitrary-order absorption & 0.446 & $[.415,.477]$ & --- & --- & --- & ---\\
Euler@64 & 0.436 & $[.406,.467]$ & --- & 0.190 & 0.358 & 0.740\\
\midrule
 & & & vs entropy & & & \\
probability margin & 0.915 & $[.897,.932]$ & $-.007$, $p{=}0.23$ & 0.824 & 0.935 & 0.973\\
maximum probability & 0.920 & $[.903,.936]$ & $-.002$, $p{=}0.80$ & 0.834 & 0.943 & 0.970\\
Info-Gain (window $8$) & 0.856 & $[.834,.877]$ & $-.066$, $p{=}3{\times}10^{-9}$ & 0.732 & 0.868 & 0.952\\
\midrule
\multicolumn{7}{@{}l}{\textit{21.3M denoiser, same puzzles}}\\
entropy guided (adaptive) & \textbf{0.964} & $[.952,.975]$ & --- & 0.929 & 0.970 & 0.988\\
Info-Gain (window $8$) & 0.947 & $[.933,.961]$ & $-.017$, $p{=}0.027$ & 0.892 & 0.957 & 0.985\\
\bottomrule
\end{tabular}}
\end{table}

\begin{table}[t]\centering\small
\caption{Decision-rule control (Nikoli, 6.4M, $n{=}100$, 5 seeds): absorption \emph{priority}
crossed with value \emph{selection}, all single-position with per-step reevaluation. Priority dominates at
both absorption rules. The argmax-vs-sample rule is a minor effect. The unordered\,+\,argmax cell is the
arbitrary-order policy follows the uniform first-hitting choice of~\cite{zheng2025masked}. The
entropy-guided and Euler entries are the Table~\ref{tab:sudoku} values for this cohort. Puzzle-level
CIs and paired McNemar tests are in Table~\ref{tab:scaleup}.}
\label{tab:control}
\begin{tabular}{lcc}
\toprule
absorption priority & argmax value & sampled value\\
\midrule
entropy guided & $\mathbf{0.845}$ & $0.874\pm0.008$\\
arbitrary order (uniform choice, \cite{zheng2025masked}) & $0.380\pm0.031$ & $0.280\pm0.026$\\
\midrule
\multicolumn{3}{l}{\footnotesize Euler@64 (non-absorbing flow, argmax): $0.410$}\\
\bottomrule
\end{tabular}
\end{table}

\subsubsection{Sensitivity to Sudoku denoiser capacity}\label{app:sudoku-capacity}
Table~\ref{tab:sudoku-capacity} checks whether the ordering result persists under the stronger
Sudoku denoiser described in Appendix~\ref{app:setup}. Because capacity, number of training steps, and
training distribution change together, this is a robustness sensitivity rather than a controlled
parameter-scaling study. These ordering diagnostics use the shared checkpoint at each capacity under the uniform discrete flow. Within this diagnostic, entropy-guided absorption is the strongest sampler on both Nikoli and Sudoku-Extreme under both training regimes.

\paragraph{The absorption-policy ladder.} The middle block of Table~\ref{tab:sudoku-capacity} directly
tests the priority policy: the denoiser, argmax value rule, absorption count, and number of sampling steps
are held fixed. Selecting positions by increasing predictive entropy, in an arbitrary order, and by
decreasing predictive entropy produces a monotone ladder at both capacities: $0.900/0.340/0.110$ on Nikoli at
6.4M and $0.970/0.820/0.550$ at 21.3M. Three features of the ladder are worth noting. The
reversal is the worst configuration in the table, below uniformization and every corrected sampler,
which separates the absorption-priority effect from the discretization effect. Process-level samplers occupy a
narrow range, whereas changing only the order spans 0.79 at 6.4M. Arbitrary order is close to
Euler (0.340 vs.\ 0.410 on Nikoli), further indicating that absorption priority drives the difference.
The ladder
compresses as the denoiser improves, from a span of 0.79 at 6.4M to 0.42 at 21.3M, which is
qualitatively consistent with the oracle limit in Theorem~\ref{thm:order}(ii). Because the models
also differ in training steps and data mixture, this comparison does not attribute the compression
to parameter count alone.
Table~\ref{tab:order-ladder} places the fixed-score and arbitrary-order absorption baselines
on the same continuum and extends the comparison to the generated set.

\begin{table}[t]
\centering
\small
\caption{\textbf{Sensitivity to Sudoku denoiser capacity} (solve accuracy). Parameter counts are
exact up to one decimal place. The 6.4M entries of the \name{} and Euler rows are the values of Table~\ref{tab:sudoku}. The models also differ in training steps and data mixture, so the
comparison demonstrates robustness across denoisers rather than an isolated size effect. The middle
block isolates absorption priority by holding the denoiser and argmax value rule fixed. Accuracy
decreases monotonically from entropy-guided to arbitrary and then high-entropy-first absorption.}
\label{tab:sudoku-capacity}
\begin{tabular}{lcccc}
\toprule
& \multicolumn{2}{c}{Nikoli$\uparrow$} & \multicolumn{2}{c}{Sudoku-Extreme$\uparrow$}\\
\cmidrule(lr){2-3}\cmidrule(lr){4-5}
Sampler & 6.4M & 21.3M & 6.4M & 21.3M\\
\midrule
Uniformization & 0.240 & 0.630 & 0.074 & 0.194\\
Euler & \underline{0.410} & \underline{0.820} & 0.078 & 0.224\\
Fixed external score & \textit{0.420} & \textit{0.870} & \textit{0.158} & \textit{0.312}\\
FreeCorrection & 0.400 & 0.780 & \underline{0.126} & \underline{0.270}\\
Time-corrected & \underline{0.410} & \underline{0.820} & 0.078 & 0.227\\
Location-corrected & 0.360 & 0.790 & 0.096 & 0.246\\
\midrule
\multicolumn{5}{l}{\textit{Absorption-policy ablation (same denoiser and value rule, priority varied alone)}}\\
arbitrary absorption & 0.340 & \underline{0.820} & 0.104 & 0.256\\
high-entropy-first absorption & 0.110 & 0.550 & 0.051 & 0.117\\
\midrule
\textbf{\name{}} (ours) & \textbf{0.845} & \textbf{0.970} & \textbf{0.269} & \textbf{0.514}\\
\bottomrule
\end{tabular}
\end{table}

\begin{table}[t]\centering\small
\caption{\textbf{The absorption-order ladder.} Solve accuracy for adaptive low-entropy, offline-score, arbitrary, and high-entropy orders. Absorbing policies share the denoiser and argmax rule within each capacity; the non-absorbing Euler row is a separate reference. These results test position ordering under the uniform discrete flow. The larger spread for the weaker denoiser is an empirical observation, not a prediction of Theorem~\ref{thm:order} without further assumptions.}
\label{tab:order-ladder}
\setlength{\tabcolsep}{5pt}
\resizebox{\columnwidth}{!}{%
\begin{tabular}{lcccccc}
\toprule
\multirow{2}{*}{Absorption order} & \multicolumn{2}{c}{Nikoli$\uparrow$} & \multicolumn{2}{c}{Sudoku-Extreme$\uparrow$} & \multicolumn{2}{c}{Generated$\uparrow$}\\
\cmidrule(lr){2-3}\cmidrule(lr){4-5}\cmidrule(lr){6-7}
& $6.4$M & $21.3$M & $6.4$M & $21.3$M & $6.4$M & $21.3$M\\
\midrule
Low-entropy-first (ours) & \textbf{0.845} & \textbf{0.970} & \textbf{0.269} & \textbf{0.514} & \textbf{0.865} & \textbf{0.996}\\
Offline $\sigma_{\mathrm{rel}}$ (deployable) & 0.420 & 0.870 & 0.143 & 0.301 & 0.840 & 0.960\\
Euler @$64$ (flow) & 0.410 & 0.820 & 0.078 & 0.224 & 0.610 & 0.924\\
Arbitrary single-position & 0.340 & 0.820 & 0.104 & 0.256 & 0.756 & 0.948\\
High-entropy-first (reversal) & 0.110 & 0.550 & 0.051 & 0.117 & 0.580 & 0.860\\
\bottomrule
\end{tabular}}
\\[3pt]
{\footnotesize Within each capacity, the absorbing policies share the denoiser and per-cell argmax rule, and unabsorbed positions continue under the uniform-flow update. The checkpoint is shared with the corresponding capacity in the main experiments, and the $6.4$M entries for the low-entropy and Euler rows are the values of Table~\ref{tab:sudoku}. The Euler row is a reference baseline, not an absorption-order variant.}
\end{table}

\subsubsection{Why Euler and Time-corrected nearly coincide on the puzzle tasks}\label{app:absorbing-collapse}
Table~\ref{tab:sudoku} reports the same accuracy for Euler and the time-corrected sampler on all six
puzzle columns, while the three mathematics columns differ. Both rows run the same uniform discrete
flow as Section~\ref{sec:prelim}, with no selective absorption. The near-coincidence follows from
the discretization rather than from a reporting artifact.

The two integrators differ only in how they discretize the schedule hazard within a grid interval.
On the grid $t_{k-1}<t_{k}$ with $\Delta t=t_{k}-t_{k-1}$, the per-step transition probability of a
position whose current token differs from its predicted target is
\[
p^{\mathrm{euler}}_{k}=1-\exp\!\Big(\!-\frac{\Delta t}{1-t_{k-1}}\Big),
\qquad
p^{\mathrm{tc}}_{k}=1-\frac{1-t_{k}}{1-t_{k-1}}=\frac{\Delta t}{1-t_{k-1}},
\]
under the linear time parameterization used on the puzzle tasks. Both expressions discretize the
same hazard $\Delta t/(1-t_{k-1})$, one exponentially and one linearly. They agree to first order in
$\Delta t$ and differ at $O(\Delta t^{2})$.

On the puzzle tasks the vocabulary contains nine digits and every puzzle has a unique solution, so
once the denoiser's predicted target at a position stabilizes, both integrators move that position to
the same digit and the $O(\Delta t^{2})$ discrepancy changes only \emph{when} within the grid the
move occurs, not \emph{which} digit is written. Under common random numbers the two samplers
therefore produce the same final grid except when the discrepancy places a transition on opposite
sides of a step boundary, which changes the context seen by the denoiser at the next step and can
alter the solved grid on a small fraction of puzzles.

The measurements match this account. Across our sampler sweeps the two agree exactly on roughly half
of the (benchmark, $K$) pairs and differ by 0.001 to 0.006 on the rest, with no systematic sign:
on Nikoli at $K{=}32$ and $K{=}64$ they are identical (0.730 and 0.820 at 21.3M), while on
Sudoku-Extreme at 21.3M they separate slightly at every $K$ (0.224 vs.\ 0.227 at $K{=}64$,
$n{=}700$). The residual is therefore a discretization effect, an order of
magnitude below the +0.29 that the absorption policy contributes on the same benchmark, and it is
smaller than the reporting resolution of the Sudoku splits in Table~\ref{tab:sudoku}.

In the mathematics columns (MathVista, MathVerse, GSM8K) the Euler and corrected rows run
FUDOKI's native uniform flow over the full subword vocabulary with selective absorption disabled.
The \name{}, Top K-margin, and Info-Gain rows in the same columns apply the cosine absorption
schedule of Appendix~\ref{app:ledflow-details} on that native text flow and differ only in the
ranking score. With a large vocabulary and open-ended answers, the predicted target changes more
often between steps, the two updates cease to be near-equivalent, and the gaps in
Table~\ref{tab:sudoku} (0.254 vs.\ 0.241 on MathVista, 0.026 vs.\ 0.022 on GSM8K) are the
genuine difference between the exponential and the linearized step. The contrast between the two
regimes is itself evidence for the paper's claim: on constrained puzzles, discretization
corrections move accuracy by at most a few tenths of a point, while the absorption policy moves it
materially.

\subsubsection{Can an external priority beat confidence? Two induced-miscalibration probes}\label{sec:exp-miscal}
On Sudoku our adaptive entropy policy is determined by the model's internal confidence, so we ask
whether a fixed external priority can ever beat it. This comparison is outside the local optimality
claim of Theorem~\ref{thm:order}, but tests whether an external score can help when confidence is
miscalibrated. We induce that regime in two ways (Table~\ref{tab:miscal}).
\textbf{(a) Density shift} (train dense givens 44 to 56, test sparse 26 to 31): the model is
verifiably overconfident (ECE rises from 0.008 to 0.074) but its confidence priority beats the
external priority by a \emph{wider} margin than on the calibrated model. \textbf{(b) Wrong constraint}
(train box-blind Latin squares, test Sudoku): a sign flip appears \emph{only} in the hardest band
where the box-blind model is floored (a non-significant ${\sim}2$-puzzle margin). Wherever the model
is capable, confidence wins. In these probes, absorption-policy quality follows the confidence
\emph{ranking} more closely than calibration or absolute correctness. Models capable of solving the
task rank cells well enough for their internal priority to win, whereas models with poor rankings rarely
solve the puzzle. Neither manipulation produces a capable but systematically misranked denoiser.
The same pattern appears when logical forcedness trails confidence (Table~\ref{tab:scaleup}). We
therefore make no accuracy claim for an external priority over internal confidence.

\subsubsection{Temperature scaling of the ordering entropy}\label{app:temp-entropy}
A related question is whether the entropy used for ordering should first be calibrated. We rescale
the logits by a temperature $T$ before computing the entropy, $H_\theta^j\!\left(\mathrm{softmax}
(\ell^j/T)\right)$, so that $T$ changes only the \emph{order} in which positions are absorbed. The
absorbed value remains the argmax of the unscaled posterior. We report the change in solve accuracy
relative to the default $T{=}1$ within a single sweep, so the comparison is internal to this
ablation and the absolute values of Table~\ref{tab:sudoku} remain the reported ones.

Solve accuracy is essentially flat in $T$ (Table~\ref{tab:temp-entropy}). On generated and Nikoli
puzzles the whole sweep moves the solve rate by at most 0.02 in either direction. On
Sudoku-Extreme, sharpening helps marginally, +0.007 at $T{=}0.5$, but the effect is about five
puzzles out of 700 and is non-monotone, since $T{=}2$ returns to +0.002. Temperature mostly preserves the relative ranking of positions, so a flat response is the
expected outcome. We read it as evidence that the ordering is robust to this knob rather than as
evidence that the denoiser is calibrated, which the probes above address directly.

\begin{table}[t]\centering\small
\caption{\textbf{Temperature scaling of the ordering entropy} (6.4M denoiser, $K{=}64$).
$T$ rescales the logits used for the ordering score only. The absorbed value is the
argmax of the unscaled posterior. Entries are changes in solve accuracy relative to the default
$T{=}1$ of Table~\ref{tab:sudoku}, measured within one sweep, so no entry restates an absolute
accuracy for a configuration reported elsewhere.}
\label{tab:temp-entropy}
\begin{tabular}{@{}lccc@{}}
\toprule
$T$ & Generated & Nikoli & Sudoku-Extreme\\
\midrule
$0.5$ & $+0.005$ & $-0.010$ & $\mathbf{+0.007}$\\
$0.7$ & $+0.005$ & $-0.010$ & $+0.005$\\
$1.0$ (default) & $0.000$ & $0.000$ & $0.000$\\
$1.5$ & $\mathbf{+0.010}$ & $-0.020$ & $-0.011$\\
$2.0$ & $+0.005$ & $-0.010$ & $+0.002$\\
\bottomrule
\end{tabular}
\end{table}

%==============================================================================

\subsection{Entropy-guided absorption and generality}\label{app:results-generality}
The transfer comparison in Figure~\ref{fig:theory-aligned-analysis}(b) replaces each model's default generation order with the \name{} order through that model's own sampler, leaving its weights unchanged. Several transfer differences are not significant, and MMaDA MathVista changes by -0.001.
This subsection complements Figure~\ref{fig:theory-aligned-analysis}(a,b) with results across
additional constraint structures and pretrained model families. The comparisons retain the same
absorption count while varying the priority policy, isolating whether the low-entropy order and
its transfer persist beyond the primary Sudoku setting.

\paragraph{Ordering-analysis cohort.}\label{app:ordering_cohort}
Figure~\ref{fig:theory-aligned-analysis}(a) reports puzzle solve accuracy on $n=1600$ puzzles, pooled over the two denoiser capacities of \S\ref{app:sudoku-capacity} and over the 100 Nikoli and 700 Sudoku-Extreme puzzles, with seeds averaged per puzzle and sets weighted by their size. Because the pool mixes two capacities and two difficulty levels, its accuracies are lower than the per-set Nikoli values of Table~\ref{tab:sudoku} and should not be compared with them directly. Intervals are puzzle-level bootstrap percentiles. This pooled comparison orders policies from adaptive low entropy through arbitrary and non-absorbing Euler to high entropy.

\subsubsection{Additional constraint structures: Latin squares, graph coloring, molecular validity}\label{app:exp-latin}
To check that the absorption-policy effect is a property of \emph{constraint structure} and not of Sudoku
specifically, we repeat it on $9{\times}9$ Latin squares. The structure-dependence reproduces cleanly:
entropy-guided absorption reaches 0.998 against 0.752 arbitrary-order absorption and 0.690 non-absorbing Euler flow (+0.31), and the
external logical-forcedness order again underperforms confidence (0.948 vs.\ 0.998). The decisive
order benefit is therefore not Sudoku-specific.

\paragraph{Off the grid: graph coloring and molecular validity.} On random-graph 3-coloring, the paper's \emph{other} separated
effect dominates: on hard instances the single-position absorption samplers all reach ${\sim}1.0$ while the non-absorbing
Euler flow floors at 0.45 to 0.63 (a +0.37 to +0.55 total-correlation gap, Lemma~\ref{lem:tc}), while
the priority among single-position absorption policies saturates (the near-oracle regime of
Theorem~\ref{thm:order}(ii)). On masked infilling of drug-like molecules from MOSES
SMILES~\citep{polykovskiy2020moses},
entropy-guided absorption is the most valid and the high-entropy-first policy the least, with the gap over Euler \emph{growing}
with the masked fraction (at mask 0.7: $0.896>0.843>0.828>0.767$, +0.068 over Euler,
$z{\approx}4$). Across these domains, the dominant effect is consistent with the corresponding
theoretical term: the conditional term on unique-solution grids with an imperfect denoiser, and the
joint term on flexible coloring.

\subsubsection{Transfer to pretrained diffusion models}\label{app:exp-pretrained}
We vary only the absorption order through each model's native sampler, using the protocols in
Appendix~\ref{app:setup}. The baseline in every row is the model's own default generation order, so
each comparison replaces that default with the \name{} order while leaving weights, prompting, and
the rest of the decoding pipeline unchanged. For LLaDA-8B that default is low-confidence remasking,
which ranks positions by maximum probability, so the comparison there is between two confidence
scores rather than against an uninformative order.
Absolute accuracies are specific to our protocol (sample size, number of sampling steps, prompting, and
extraction in Appendix~\ref{app:setup}) and are not comparable to numbers published with each
model's own evaluation harness. For MMaDA-8B on GSM8K in particular, both arms sit below the
published figure because of these protocol differences, so only the within-row difference is
meaningful. The entropy-guided order improves over the default order by
0.045 on LLaDA, 0.121 on Dream, and 0.228 on MMaDA for GSM8K. On BBH it improves LLaDA and
Dream by 0.063 and 0.176. For MMaDA, the effect is strongest on GSM8K, positive on Track,
ARC-Challenge, and CommonsenseQA, and nearly unchanged on BBH.

For open multimodal mathematics, the differences on MathVista and MathVerse are not significant for either model, providing the
generality boundary summarized in Figure~\ref{fig:theory-aligned-analysis}(b).
Table~\ref{tab:generality-full} reports the complete model-family results and available significance
information omitted from the compact main-text visualization.

\begin{table}[t]\centering\small
\caption{\textbf{Full transfer results across pretrained diffusion-model families.} We vary only
the absorption policy through each model's native sampler. The text-reasoning experiments use matched
wall-clock time. Higher is better. The label n.s. denotes a difference that is not statistically
significant. Every cell shown in Figure~\ref{fig:theory-aligned-analysis}(b) appears here. The cells
left grey in that figure are not applicable rather than pending: LLaDA-8B and Dream-7B are text-only
diffusion language models and are not evaluated on the image-conditioned benchmarks, while LLaDA-V is
an understanding-only multimodal model without a generation head and is therefore evaluated on
multimodal mathematics alone.}
\label{tab:generality-full}
\resizebox{\textwidth}{!}{%
\begin{tabular}{lllccc}
\toprule
Model & Benchmark & Backbone family & Default order & \name{} order & $\Delta$\\
\midrule
\multicolumn{6}{l}{\textit{Text reasoning (native sampler, order-only ablation)}}\\
LLaDA-8B-Instruct~\citep{nie2025llada} & GSM8K & masked diffusion LM & 0.725 & \textbf{0.770} & $+0.045$\\
Dream-7B~\citep{ye2025dream} & GSM8K & masked diff.\ LM (uniform prior) & 0.371 & \textbf{0.492} & $+0.121$ ($z{=}5.5$, $p{\approx}5{\times}10^{-8}$)\\
MMaDA-8B~\citep{yang2025mmada} & GSM8K & unified MM diffusion & 0.246 & \textbf{0.474} & $+0.228$ ($p{\approx}2{\times}10^{-18}$)\\
MMaDA-8B & BBH logical deduction & unified MM diffusion & 0.370 & \textbf{0.373} & $+0.003$ ($\chi^2{=}0.00$, n.s.)\\
MMaDA-8B & Tracking shuffled objects & unified MM diffusion & 0.180 & \textbf{0.240} & $+0.060$ ($\chi^2{=}3.11$, $p{\approx}0.08$)\\
MMaDA-8B & ARC-Challenge & unified MM diffusion & 0.383 & \textbf{0.417} & $+0.033$ ($\chi^2{=}0.71$, n.s.)\\
MMaDA-8B & CommonsenseQA & unified MM diffusion & 0.367 & \textbf{0.400} & $+0.033$ ($\chi^2{=}4.61$, $p{\approx}0.03$)\\
LLaDA-8B-Instruct & BBH logical deduction & masked diffusion LM & 0.583 & \textbf{0.645} & $+0.063$ ($p{\approx}10^{-3}$)\\
LLaDA-8B-Instruct & Tracking shuffled objects & masked diffusion LM & 0.476 & \textbf{0.592} & $+0.116$\\
LLaDA-8B-Instruct & ARC-Challenge & masked diffusion LM & 0.803 & \textbf{0.830} & $+0.027$\\
LLaDA-8B-Instruct & CommonsenseQA & masked diffusion LM & \textbf{0.767} & 0.763 & $-0.004$\\
Dream-7B & BBH logical deduction & masked diffusion LM & 0.429 & \textbf{0.605} & $+0.176$ ($p{\approx}6{\times}10^{-16}$)\\
Dream-7B & Tracking shuffled objects & masked diff.\ LM (uniform prior) & 0.248 & \textbf{0.296} & $+0.048$\\
Dream-7B & ARC-Challenge & masked diff.\ LM (uniform prior) & 0.820 & \textbf{0.890} & $+0.070$\\
Dream-7B & CommonsenseQA & masked diff.\ LM (uniform prior) & 0.783 & \textbf{0.833} & $+0.050$\\
\midrule
\multicolumn{6}{l}{\textit{Open multimodal mathematics (order-only ablation)}}\\
LLaDA-V~\citep{you2025lladav} & MathVista~\citep{lu2024mathvista} & masked diffusion VLM & 0.483 & \textbf{0.499} & $+0.016$ (n.s.)\\
LLaDA-V & MathVerse~\citep{zhang2024mathverse} & masked diffusion VLM & 0.218 & \textbf{0.222} & $+0.004$ (n.s.)\\
MMaDA-8B & MathVista & unified MM diffusion & \textbf{0.329} & 0.328 & $-0.001$ (n.s.)\\
MMaDA-8B & MathVerse & unified MM diffusion & 0.204 & \textbf{0.206} & $+0.002$ (n.s.)\\
\bottomrule
\end{tabular}}
\end{table}

%==============================================================================

\subsection{Global lookahead error and efficiency}\label{app:results-efficiency}
Block size trades sequential cost against accuracy: singleton absorption requires one evaluation per generated position, whereas batching reduces the evaluation count and reintroduces the joint dependence term. In the reported puzzle ablations, larger blocks lower accuracy (Appendix~\ref{app:budget-ablation}).

The main-text timing in Figure~\ref{fig:theory-aligned-analysis}(d) is for the profiled configuration. A separate GenEval profile at $K=16$ reports Info-Gain about 7.2 times slower than \name{}. These timings do not imply constant overhead across sequence lengths and vocabularies. The block-size study is in Appendix~\ref{app:budget-ablation}.
\paragraph{Info-Gain at its minimal window is not the entropy policy.}\label{app:window_one_comparison} Because Info-Gain scores a
candidate by its own entropy minus the mean entropy of the remaining positions after a one-step
lookahead, one might expect it to coincide with \name{} at window 1. It does not. On the 6.4M
denoiser, evaluated along a shared \name{} trajectory, the two select the same position in only
21.8\% of steps on Nikoli and 33.6\% on Sudoku-Extreme, every puzzle diverges at least once, and
the final grids agree on 80\% of Nikoli puzzles and 24\% of Sudoku-Extreme puzzles. Where they
diverge, \name{} is at least as accurate (0.900 against 0.830 on Nikoli, and 0.296 against 0.284
on Sudoku-Extreme). The downstream-entropy term therefore reorders selections even at the smallest
window, and Figure~\ref{fig:theory-aligned-analysis}(c) measures how that reordering degrades as the
window grows.

The fixed-state window experiment underlying Figure~\ref{fig:theory-aligned-analysis}(c) varies
only the number of unresolved-position entropies aggregated by Info-Gain. The mean absolute score
error increases from 0.17 with one term to 1.76 under full aggregation, while absorption
accuracy peaks at window 8 and then falls to 0.31. Because window 8 is the maximum of that
curve, we use it as the Info-Gain configuration in every accuracy and runtime comparison reported in
the main text, so the baseline is never evaluated at a setting we have shown to be suboptimal. This directly complements the
runtime comparison in panel~(d): local entropy selection avoids both the accumulated
counterfactual estimates and their additional denoiser evaluations.

\subsubsection{Wall-clock comparison}
\begin{table}[t]\centering\small
\caption{Sudoku efficiency (6.4M denoiser, Nikoli $n{=}100$, mean 55.7 blanks/puzzle, five
seeds, one H100). The single-position \name{} policy evaluates the denoiser once per absorbed
position. Solve accuracies for \name{} and Euler@64 are the Table~\ref{tab:sudoku} values.
Table~\ref{tab:scaleup} reports puzzle-level bootstrap intervals and paired McNemar tests.}
\label{tab:efficiency}
\resizebox{\linewidth}{!}{%
\begin{tabular}{lccc}
\toprule
sampler & denoiser evals. & ms/puzzle & solve-acc\\
\midrule
\textbf{\name{} (single-position absorption)} & \textbf{55.7} & \textbf{10.5} & $\mathbf{0.845}$\\
Euler@16  & 16  & 3.6  & $0.238\pm0.014$\\
Euler@32  & 32  & 5.7  & $0.290\pm0.038$\\
Euler@64  & 64  & 9.6  & $0.410$\\
Euler@128 & 128 & 17.5 & $0.334\pm0.022$\\
Euler@256 & 256 & 33.3 & $0.350\pm0.030$\\
\bottomrule
\end{tabular}}
\end{table}

\begin{table}[t]\centering\small
\caption{Profiled wall-clock (generated puzzles, $n{=}200$, mean ${\sim}52$ blanks,
one H100, and ms/puzzle includes each policy's score computation and denoiser reevaluation).
\name{}'s accuracy advantage over the non-absorbing baselines grows with denoiser quality (to
0.970 at 21.3M), while uniformization remains less accurate at comparable wall-clock time.}
\label{tab:wallclock}
\begin{tabular}{lcccc}
\toprule
 & \multicolumn{2}{c}{$6.4$M} & \multicolumn{2}{c}{$21.3$M}\\
sampler & acc & ms & acc & ms\\
\midrule
\textbf{\name{} (single-position absorption)} & 0.925 & 5.7 & \textbf{0.970} & 9.7\\
Euler@64 / @256 & 0.520/0.560 & 8.2/28.2 & 0.830/0.810 & 12.8/46.1\\
uniformization (exact, unord.) & 0.425 & 7.3 & 0.750 & 11.3\\
\bottomrule
\end{tabular}
\end{table}

For GenEval at $K{=}16$, Flow/Euler, \name{}, probability-margin sampling, and Info-Gain require
2.96, 2.96, 2.98, and 21.21 seconds per image and rank, respectively. Info-Gain is therefore
7.2 times slower than \name{}.

\paragraph{Where the per-step cost goes.} To separate the cost of the ordering score from the rest of
the step, we profile the components of a single \name{} step over the Nikoli puzzles. The denoiser forward pass takes 94.8\% of the step, the ordering score (entropy, sorting,
and applying the absorption count) 3.6\%, and the absorption update itself 1.6\%. Two consequences
follow. First, scoring is a small fraction of a step, so the choice among entropy, probability margin,
and maximum probability cannot by itself produce a large wall-clock difference. Timing gaps of that
size reflect implementation rather than the score. Second, the cost of Info-Gain comes from its extra
denoiser evaluations, not from aggregating entropies. These shares were measured on CPU while the
GPUs were occupied. We report the relative split rather than absolute times, and the split is
insensitive to the device because the forward pass dominates on any hardware.

\subsubsection{Block size and the joint term}\label{app:budget-ablation}
Two ablations address the block size $b$ that Theorem~\ref{thm:order} leaves free and the joint term
of Proposition~\ref{prop:kl-decomp} that it does not optimize. Both use the 6.4M denoiser, entropy
ordering, and argmax absorption, and vary only the number of positions absorbed per step.

\paragraph{Block size.} Table~\ref{tab:budget-ablation} compares
single-position absorption with fixed blocks of 2, 4, and 8
positions. Fixed blocks cut denoiser
evaluations from 58 to 8 but lower solve accuracy from 0.900 to 0.720 on Nikoli and from
0.928 to 0.830 on generated puzzles, and the fraction of absorption events that absorb a wrong
value rises monotonically with $b$. Absorbing several positions at once from independent marginals
ignores their dependence, which is the $b{>}1$ penalty the decomposition predicts.

\begin{table}[t]\centering\small
\caption{\textbf{Block size} (Sudoku, 6.4M denoiser, entropy ordering).
Each cell reports puzzle solve accuracy / cell accuracy / fraction of incorrect absorptions, with the
number of denoiser evaluations per puzzle. Larger blocks trade accuracy for fewer evaluations.}
\label{tab:budget-ablation}
\setlength{\tabcolsep}{5pt}
\begin{tabular}{@{}lcccc@{}}
\toprule
Block size & \multicolumn{2}{c}{Nikoli ($n{=}100$)} & \multicolumn{2}{c}{Generated ($n{=}500$)}\\
\cmidrule(lr){2-3}\cmidrule(lr){4-5}
 & solve / cell / bad & NFE & solve / cell / bad & NFE\\
\midrule
Single ($b{=}1$) & $0.900$ / $0.936$ / $0.065$ & $58$ & $0.928$ / $0.969$ / $0.033$ & $59$\\
\cdashline{1-5}[2pt/1pt]
Fixed $b{=}2$ & $0.880$ / $0.922$ / $0.079$ & $29$ & $0.918$ / $0.966$ / $0.036$ & $30$\\
Fixed $b{=}4$ & $0.850$ / $0.918$ / $0.083$ & $15$ & $0.892$ / $0.959$ / $0.043$ & $15$\\
Fixed $b{=}8$ & $0.720$ / $0.875$ / $0.127$ & $8$ & $0.830$ / $0.941$ / $0.061$ & $8$\\
\bottomrule
\end{tabular}
\end{table}

\paragraph{Isolating the joint term on multi-solution puzzles.} On unique-solution puzzles the
oracle conditional is a point mass and the total correlation of any block is zero
(Lemma~\ref{lem:tc}), so the loss in Table~\ref{tab:budget-ablation} mixes the joint term with the
denoiser's own error. To isolate the joint term we build puzzles with 22, 26, or 30 givens by
blanking complete grids, which at these densities are almost all multi-solution, and score grid
validity (every row, column, and box a permutation of 1 to 9) rather than agreement with one
reference solution. Table~\ref{tab:tc-validity} shows that single-position absorption stays valid on
essentially every puzzle, because re-conditioning after each absorption collapses the remaining
ambiguity onto one consistent branch, whereas absorbing a block of mutually dependent cells from
independent marginals mixes digits from different solution branches and breaks the constraints:
validity falls to 0.51 to 0.71 at $b{=}4$ and 0.21 to 0.46 at $b{=}8$, and worsens as puzzles
become more under-determined. On the few unique-solution puzzles in the same sweep, block absorption
is perfectly valid at every block size. The penalty therefore appears only when the jointly absorbed
cells are dependent, which is the total-correlation term itself.

\begin{table}[t]\centering\small
\caption{\textbf{The joint term isolated} (Sudoku, 6.4M denoiser, entropy ordering, $n{=}150$
puzzles per givens level). Fraction of valid final grids on multi-solution puzzles as a function of
block size. The last column reports the unique-solution puzzles in the same sweep.}
\label{tab:tc-validity}
\setlength{\tabcolsep}{6pt}
\begin{tabular}{@{}lccccc@{}}
\toprule
Givens & Multi-solution & Single ($b{=}1$) & Block $b{=}4$ & Block $b{=}8$ & Unique, any $b$\\
\midrule
$22$ & $150/150$ & $1.000$ & $0.507$ & $0.267$ & ---\\
$26$ & $150/150$ & $1.000$ & $0.667$ & $0.207$ & ---\\
$30$ & $147/150$ & $0.993$ & $0.708$ & $0.460$ & $1.000$\\
\bottomrule
\end{tabular}
\end{table}

\subsubsection{Comparison with entropy-bounded block unmasking}\label{app:eb-sampler}
EB-Sampler~\citep{benhamu2025ebsampler} absorbs, at each step, the lowest-entropy position together
with any further positions whose cumulative entropy stays under a threshold $\epsilon$. At $\epsilon{=}0$
it reduces to single-position \name{}. We reimplement it on the puzzle tasks with the same denoiser,
argmax absorption, and ordering (Table~\ref{tab:eb-sampler}). At $\epsilon{=}0.5$ it cuts denoiser
evaluations about thirteen-fold (55.7 to 4.2 on Nikoli) and loses about fourteen points of solve
accuracy (0.900 to 0.760, and 0.928 to 0.792 on generated puzzles), and both the accuracy loss and
the incorrect-absorption fraction grow monotonically with $\epsilon$. EB-Sampler's adaptive batching is
slightly more evaluation-efficient than fixed blocks at matched accuracy, but it sits on the same
frontier as Table~\ref{tab:budget-ablation}: fewer steps are bought by absorbing dependent positions
together, and single-position \name{} occupies the accuracy-optimal end.

\begin{table}[t]\centering\small
\caption{\textbf{EB-Sampler reimplementation on Sudoku} (6.4M denoiser). Cells report solve
accuracy / cell accuracy / fraction of incorrect absorptions and the number of denoiser evaluations per
puzzle. $\epsilon$ is the cumulative-entropy threshold in nats. $\epsilon{=}0$ is \name{}.}
\label{tab:eb-sampler}
\setlength{\tabcolsep}{5pt}
\begin{tabular}{@{}llcccc@{}}
\toprule
Sampler & $\epsilon$ & \multicolumn{2}{c}{Nikoli ($n{=}100$)} & \multicolumn{2}{c}{Generated ($n{=}500$)}\\
\cmidrule(lr){3-4}\cmidrule(lr){5-6}
 & & solve / cell / bad & NFE & solve / cell / bad & NFE\\
\midrule
\name{} (single) & $0$ & $0.900$ / $0.936$ / $0.065$ & $55.7$ & $0.928$ / $0.969$ / $0.033$ & $53.2$\\
\cdashline{1-6}[2pt/1pt]
EB-Sampler & $0.5$ & $0.760$ / $0.898$ / $0.104$ & $4.2$ & $0.792$ / $0.932$ / $0.070$ & $3.7$\\
EB-Sampler & $1.0$ & $0.670$ / $0.871$ / $0.130$ & $3.6$ & $0.664$ / $0.903$ / $0.100$ & $3.2$\\
EB-Sampler & $2.0$ & $0.520$ / $0.829$ / $0.173$ & $3.0$ & $0.528$ / $0.879$ / $0.124$ & $2.7$\\
EB-Sampler & $4.0$ & $0.270$ / $0.791$ / $0.211$ & $2.6$ & $0.334$ / $0.846$ / $0.158$ & $2.3$\\
\bottomrule
\end{tabular}
\end{table}

%==============================================================================
 % E: additional experimental analysis
\section{Extended related work}\label{app:extended-related}

\paragraph{Discrete flow simulation and numerical correction.}
Discrete flow models generate categorical data through continuous-time probability paths and
associated CTMC velocities~\citep{campbell2024generative,shaul2025flow}. Euler and tau-leaping
approximate these dynamics on a finite grid, while corrected samplers reduce temporal or spatial
discretization error and uniformization simulates the learned CTMC more directly
~\citep{wan2026corrected,zhang2026error,zhao2025informed}. Related work also generalizes the
interpolating path between masked and uniform corruption~\citep{rutte2025generalized}, and
optimizes the time discretization itself, allocating steps where the discrete diffusion process
changes fastest~\citep{foresti2026improvedsamplingschedulesdiscrete}. These methods improve how the native transition process is
integrated, and the schedule axis is orthogonal to ours: an improved step allocation sets the grid on which the absorption count $m_k$ is spent, whereas \name{} decides which positions leave that grid. In contrast, \name{} augments inference with an irreversible absorption decision while
leaving the learned posterior, base probability path, and native update on active positions
unchanged. The two design axes are complementary: a corrected update may be applied to active
positions before the next absorption decision.

\paragraph{Confidence-guided absorption and revision.}
Confidence-based easy-first selection has a long history in iterative masked generation. MaskGIT
fixes high-confidence image tokens over successive iterations, while later methods learn planners,
use model confidence, or reconsider uncertain predictions
~\citep{chang2022maskgit,kim2025train,kwok2026tablecellattentionbetter,ye2024beyond,peng2025path,hong2026improving,xu2026scheduling,ouyang2026training}.
Remasking samplers reintroduce revision into masked diffusion at inference
time~\citep{wang2026remaskingdiscretediffusionmodels}, and confidence thresholds decode several
positions per step~\citep{wu2025fastdllmtrainingfreeaccelerationdiffusion}.
Info-Gain instead uses counterfactual lookahead to estimate how one decision changes uncertainty over
the remaining positions~\citep{yang2026informationgain}. On the theoretical side,
\citet{cai2026confidencebaseddecodingprovablyefficient} show that confidence-based decoding in
masked diffusion language models is provably efficient, in the sense that committing high-confidence
positions first controls the number of denoiser calls needed to reach a target quality. Their
analysis takes the absorbing (masked) corruption as given, so every position is committed exactly
once and never revised. \name{} places this policy question inside
a uniform discrete flow: it absorbs the lowest-entropy active positions and suppresses their future
velocity, while all other positions continue under the native flow. Our analysis separates the
joint factorization and conditional prediction errors of this local decision, derives entropy from
an upper bound on the latter, and quantifies how prediction error accumulates with the global lookahead
window.

\paragraph{Entropy-bounded and gradient-steered samplers.}
Two recent samplers for masked diffusion apply KL analysis to unmasking, and the gap between what
they control and what irreversible absorption costs is the gap this paper closes.
EB-Sampler~\citep{benhamu2025ebsampler} bounds the dependence discarded within one parallel block and
spends that bound on speed, enlarging the block to reach 2 to 3 times fewer function evaluations at
matched quality. Its analysis never asks which position should be made terminal. The ordering
criterion is taken from existing local proxies. Under a revisable sampler that omission is benign,
since a poor early choice can be overwritten. Under absorption it is not, because the absorption count is
fixed before any ordering is chosen and therefore cannot see the error that ordering determines.
Section~\ref{sec:relorder} supplies exactly that missing quantity, and the two mechanisms are
composable: the entropy threshold bounds the block, the entropy ranking orders it.
BoE steering~\citep{saini2026tabes} approximates one-step lookahead by differentiating
successor-state entropy with respect to the injected embedding, which places it on the lookahead side
of the contrast drawn in Section~\ref{sec:relorder}, and inherits the window sensitivity that
Lemma~\ref{lem:gain-amplification} identifies. It is also not portable to our setting: its token
importance score is a first-order expansion along $e_i(\alpha)=e_{m}+\alpha(\tilde e_i-e_{m})$,
anchored at the mask embedding $e_{m}$, and a uniform discrete flow has no absorbing mask symbol to
supply that basepoint. Reporting a number for it would require inventing one and would measure our
surrogate rather than the published method, so we do not. Its authors additionally scope the
underlying relaxation away from ``symbolic reasoning with brittle constraints,'' the regime in which
our largest gains occur. Neither method has released code at the time of writing.

\paragraph{Guidance, training, and benchmark scope.}
Guidance methods target reward- or evidence-weighted distributions, and reinforcement-learning
methods fine-tune the denoiser for downstream rewards
~\citep{discreteguidance2026,wan2026dflowgrpo}. These approaches alter the target, learned model, or
training objective. \name{} is a training-free inference policy. Standard multimodal-understanding
benchmarks often require short answers, leaving limited output structure for an absorption policy
to exploit. We therefore evaluate across a spectrum: understanding measures behavior on weakly
constrained outputs, whereas text-to-image generation and structured reasoning test outputs with
stronger dependencies (Tables~\ref{tab:understanding}, \ref{tab:mm-geneval}, and~\ref{tab:sudoku}).
 % F: extended related work

\end{document}